\documentclass[11pt]{article}

\usepackage{fullpage}
\usepackage{multirow}
\usepackage{natbib}
\usepackage{graphicx}
\usepackage{amsthm}
\usepackage{bbding}
\usepackage{subfigure}
\usepackage{makecell}
\usepackage{booktabs}
\usepackage{array}
\usepackage{url}
\usepackage{algorithm}
\usepackage{algorithmic}
\usepackage{dsfont}
\usepackage{amsfonts}
\usepackage{amsmath}

\usepackage{mathrsfs}
\usepackage{hyperref}
\hypersetup{hidelinks}
\hypersetup{
colorlinks=true,
linkcolor=blue,
citecolor=blue
}

\usepackage{color,xcolor}
\usepackage{colortbl}
\usepackage{bbm}
\usepackage{tcolorbox}
\usepackage{graphicx}

\newcommand{\cA}{\mathcal{A}}

\newcommand{\cN}{\mathcal{N}}
\newcommand{\cO}{\mathcal{O}}

\ifx\theorem\undefined
\newtheorem{theorem}{Theorem}
\fi
\ifx\example\undefined
\newtheorem{example}[theorem]{Example}
\fi
\ifx\property\undefined
\newtheorem{property}{Property}
\fi
\ifx\lemma\undefined
\newtheorem{lemma}[theorem]{Lemma}
\fi
\ifx\proposition\undefined
\newtheorem{proposition}[theorem]{Proposition}
\fi
\ifx\remark\undefined
\newtheorem{remark}[theorem]{Remark}
\fi
\ifx\corollary\undefined
\newtheorem{corollary}[theorem]{Corollary}
\fi
\ifx\definition\undefined
\newtheorem{definition}[theorem]{Definition}
\fi
\ifx\conjecture\undefined
\newtheorem{conjecture}[theorem]{Conjecture}
\fi
\ifx\fact\undefined
\newtheorem{fact}[theorem]{Fact}
\fi
\ifx\claim\undefined
\newtheorem{claim}[theorem]{Claim}
\fi
\ifx\assumption\undefined
\newtheorem{assumption}[theorem]{Assumption}
\fi
\ifx\cond\undefined
\newtheorem{cond}[theorem]{Condition}
\fi

\newcommand{\PP}{\mathbb{P}}

\newcommand{\RR}{\mathbb{R}}

\usepackage{graphicx} 

\title{How Over-Parameterization Slows Down Gradient Descent in Matrix Sensing: The Curses of Symmetry and Initialization}
\author{Nuoya Xiong\thanks{IIIS, Tsinghua University. Email: \texttt{nuoyaxiong@gmail.com}. Part of the work was done while Nuoya Xiong was visiting  the University of Washington.} \qquad Lijun Ding\thanks{Wisconsin Institute for Discovery, University of Wisconsin-Madison, Madison.  Email: \texttt{lding47@wisc.edu}}
\qquad Simon S. Du\thanks{Paul G. Allen School of Computer Science and Engineering, University of Washington.  Email: \texttt{ssdu@cs.washington.edu}} }
\date{\today}

\begin{document}

\maketitle

\begin{abstract}
    This paper rigorously shows how over-parameterization dramatically changes the convergence behaviors of gradient descent (GD) for the matrix sensing problem, where the goal is to recover an unknown low-rank ground-truth matrix from near-isotropic linear measurements.
First, we consider the symmetric setting with the symmetric parameterization where $M^\star \in \mathbb{R}^{n \times n}$ is a positive semi-definite unknown matrix of rank $r \ll n$, and one uses a symmetric parameterization $XX^\top$ to learn $M^\star$. Here, $X \in \mathbb{R}^{n \times k}$ with $k > r$ is the factor matrix.
We give a novel $\Omega\left(1/T^2\right)$ \emph{lower bound} of randomly initialized GD for the over-parameterized case ($k >r$) where $T$ is the number of iterations.
This is in stark contrast to the exact-parameterization scenario ($k=r$) where the convergence rate is $\exp\left(-\Omega\left(T\right)\right)$.
Next, we study asymmetric setting where $M^\star \in \mathbb{R}^{n_1 \times n_2}$ is the unknown matrix of rank $r \ll \min\{n_1,n_2\}$, and one uses an asymmetric parameterization $FG^\top$ to learn $M^\star$ where $F \in \mathbb{R}^{n_1 \times k}$ and $G \in \mathbb{R}^{n_2 \times k}$. Building on prior work, 
we give a global exact convergence result of randomly initialized GD for the exact-parameterization case ($k=r$) with an $\exp\left(-\Omega\left(T\right)\right)$ rate.
Furthermore, we give the first global exact convergence result for the over-parameterization case ($k>r$) with an $\exp\left(-\Omega\left(\alpha^2 T\right)\right)$ rate where $\alpha$ is the initialization scale.
This linear convergence result in the over-parameterization case is especially significant because one can apply the asymmetric parameterization to the symmetric setting to speed up from $\Omega\left(1/T^2\right)$ to linear convergence. Therefore, we identify a surprising phenomenon: \emph{asymmetric parameterization can exponentially speed up convergence.}
Equally surprising is our analysis that highlights the importance of \emph{imbalance} between $F$ and $G$. This is in sharp contrast to prior works which emphasize balance. 
We further give an example showing the dependency on $\alpha$ in the convergence rate is unavoidable in the worst case.
On the other hand, we propose a novel method that only modifies one step of GD and obtains a convergence rate independent of $\alpha$, recovering the rate in the exact-parameterization case.
We provide empirical studies to verify our theoretical findings.
\end{abstract}

\section{Introduction}
\label{sec:intro}

A line of recent work showed over-parameterization plays a key role in optimization, especially for neural networks~\citep{allen2019learning, du2018gradient, jacot2018neural, safran2018spurious, chizat2019lazy, wei2019regularization, nguyen2020rigorous, fang2021modeling, lu2020mean,zou2020gradient}.
However, our understanding of the impact of over-parameterization on optimization is far from complete.
In this paper, we focus on the canonical matrix sensing problem and show that over-parameterization qualitatively changes the convergence behaviors of gradient descent (GD).

Matrix sensing aims to recover a low-rank unknown matrix $M^\star$ from $m$ linear measurements, 
\begin{equation}
\label{eq: observation_model}
y_i = \mathcal{A}_i(M^\star)= \langle A_i,M^\star \rangle = \mathrm{tr}(A_i^\top  M^\star), \text{ for } i=1,\ldots,m,
\end{equation} where $\mathcal{A}_i$ is a linear measurement operator and $A_i$ is the measurement matrix of the same size as $M^\star$.
This is a classical problem with numerous real-world applications, including signal processing \citep{weng2012low} and face recognition \citep{chen2012low}, image reconstruction~\citep{zhao2010low,peng2014reweighted}.
Moreover, this problem can serve as a test-bed of convergence behaviors in deep learning theory since it is non-convex and retains many key phenomena \citep{soltanolkotabi2023implicit, jin2023understanding,li2018algorithmic,li2020towards,arora2019implicit}.
We primarily focus on the \emph{over-parameterized} case where we use a model with rank larger than that of $M^\star$ in the learning process. This case is particularly relevant because $\text{rank}(M^\star)$ is usually unknown in practice.

\begin{figure}[t]\label{figure:symmetric}
    \centering
    \subfigure[Loss curve]{\label{figure:loss_symmetric}\includegraphics[scale = 0.28]{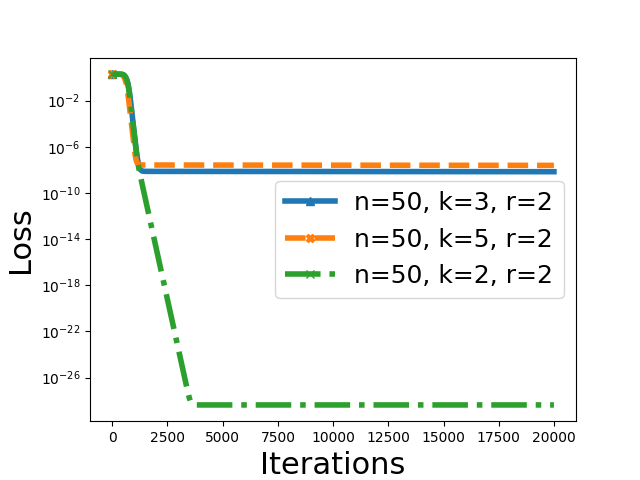}
    }
    \subfigure[Curve of logarithm of loss: $\log_T\|X_tX_t^\top - M^\star\|_F^2$]{\label{figure:logloss_symmetric}\includegraphics[scale=0.28]{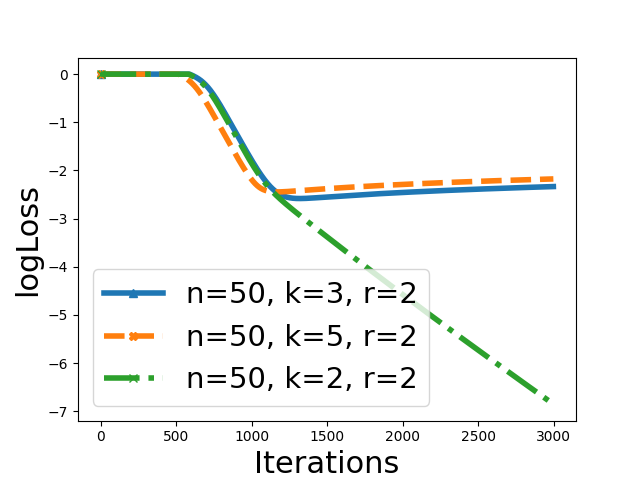}}
    \subfigure[Comparison between symmetric and asymmetric parameterization]{\label{figure:comparison}\includegraphics[scale=0.28]{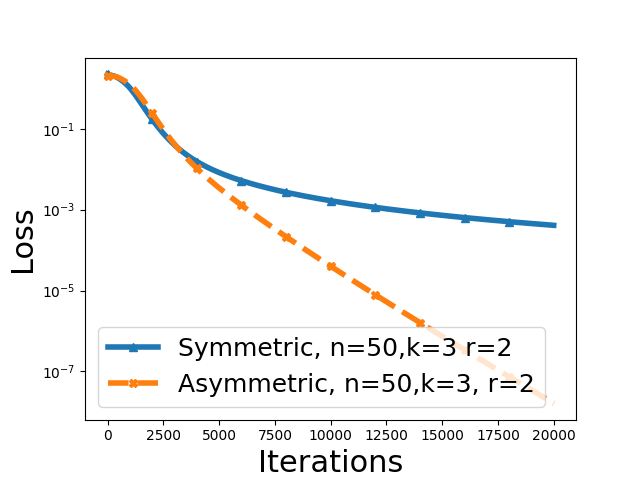}}
            \vspace{-0.4cm}
    \caption{   
    Experiments on symmetric setting. 
    $n$: ambient dimension.
    $k$: rank in our model.
    $r$: true rank.
    The first two figures show that the convergence rate of symmetric matrix factorization in the over-parameterized setting is about $\Theta(1/T^2),$ while the rate of the exact-parameterized setting is linear.  \ref{figure:comparison} shows that using asymmetric parameterization is exponentially faster than symmetric parameterization. 
    See \S \ref{sec:exp} for experimental details. 
    } 
    \vspace{-0.6cm}
\end{figure}

\subsection{Setting 1: Symmetric Matrix Sensing with Symmetric Parameterization}
We first consider the symmetric matrix sensing setting, where $M^\star \in \mathbb{R}^{n \times n}$ is a positive semi-definite matrix of rank $r \ll n$.
A standard approach is to use a factored form $XX^\top$ to learn $M^\star$ where $X \in \mathbb{R}^{n \times k}$. We call this \emph{symmetric parameterization} because $XX^\top$ is always symmetric and positive semi-definite. 
We will also introduce the \emph{asymmetric} parameterization soon. 
We call the case when $k=r$ the \emph{exact-parameterization} because the rank of $XX^\top$ matches that of $M^\star$.
However, in practice, $r$ is often unknown, so one may choose some large enough $k > r$ to ensure the expressiveness of $XX^\top$, and we call this case \emph{over-parameterization}.

We consider using gradient descent to minimize the standard $L_2$ loss for training:
$
    L_{\mathrm{tr}}(X) = \frac{1}{2m}\sum_{i=1}^m \left(y_i- \langle A_i, XX^\top \rangle\right)^2.
$
We use the Frobneius norm of the reconstruction error as the performance metric:
\begin{align}
    L(X) = \frac{1}{2}\Vert XX^\top -M^\star\Vert^2_F. \label{eq: loss_symmetric}
\end{align}
We note that $L(X)$ is also the \emph{matrix factorization} loss and can be viewed as a special case of $L_{\mathrm{tr}}$ when $\{A_i\}_{i=1}^m$ are random Gaussian matrices and the number of linear measurements goes to infinity.

For the exact-parameterization case, one can combine results in \citet{stoger2021small} and \citet{tu2016low} to show that randomly initialized gradient descent enjoys an $\exp\left(-\Omega\left(T\right)\right)$ convergence rate where $T$ is the number of iterations.
For the over-parameterization case, one can combine the results by \citet{stoger2021small} and \citet{zhuo2021computational} to show an $O\left(1/T^2\right)$ convergence rate \emph{upper bound}\footnote{More specifically, one can combine \citep[Theorem 3.3]{stoger2021small} and \citep[Lemma 3]{zhuo2021computational} to achieve the $O(1/T^2)$ rate.  Theorem 3.3 in \citep{stoger2021small} is used to achieve the initial condition in \citep[Lemma 3]{zhuo2021computational}. One can set the noise parameter in \citep[Lemma 3]{zhuo2021computational} as $0$ and replace the subgaussian assumption on $A_i$ there by the restricted isometry property, Definition \ref{def: RIP}).}, which is exponentially worse. This behavior has been empirically observed \citep{zhang2021preconditioned,zhang2023preconditioned,zhuo2021computational} without a rigorous proof. See Figure~\ref{figure:symmetric}.

\noindent 
\textbf{Contribution 1: $\Omega(1/T^2)$ Lower Bound for Symmetric Over-Parameterization}. Our first contribution is a rigorous exponential separation between the exact-parameterization and over-parameterization by proving an $\Omega(1/T^2)$ convergence rate \emph{lower bound} for the symmetric setting with the symmetric over-parameterization.

\begin{theorem}[Informal]\label{informal thm:symmetric lower bound}
Suppose we initialize $X$ with a Gaussian distribution with small enough variance that scales with $\alpha^2$, and use gradient descent with a small enough constant step size to optimize the matrix factorization loss \eqref{eq: loss_symmetric}.
Let $X_t$ denote the factor matrix at the $t$-th iteration. Then, with high probability over the initialization, there exists $T^{(0)} > 0$ such that we have\footnote{For clarity, in our informal theorems in Section~\ref{sec:intro}, we only display the dependency on $\alpha$ and $T$, and ignore parameters such as dimension, condition number, and step size.} 
\footnote{$T^{(0)}$ here and $T^{(1)}$, $T^{(2)}$, $T^{(3)}$ in theorems below represent the burn-in time to get to a neighborhood of an optimum, which can depend on initialization scale $\alpha$, condition number, dimension, and step size.}\begin{equation}\label{eq: symMF.rate}
        \|X_tX_t^\top -M^\star\|_F^2 \ge \left(\frac{\alpha^2}{ t}\right)^2, \forall t \ge T^{(0)}. 
    \end{equation}

\end{theorem}

\paragraph{Technical Insight:} %
We find the root cause of the slow convergence is from the \emph{redundant space} in $XX^\top$, which converges to $0$ at a much slower rate compared to the \emph{signal space} which converges to $M^\star$ with a linear rate.
To derive the lower bound, we construct a potential function and use some novel analyses of the updating rule to show that the potential function decreases slowly after a few rounds.
See the precise theorem and more technical discussions in Section~\ref{sec: lbsymf}.
    



\subsection{Setting 2: Symmetric and Asymmetric Matrix Sensing with Asymmetric Parameterization}

Next, we consider the more general asymmetric matrix sensing problem where the ground-truth $M^\star \in \mathbb{R}^{n_1 \times n_2}$ is an asymmetric matrix of rank $r$.
For this setting, we must use the \emph{asymmetric parameterization}.
Specifically, we use $FG^\top$ to learn $M^\star$ where $F \in \mathbb{R}^{n_1 \times k}$ and $G \in \mathbb{R}^{n_2 \times k}$.
Same as in the symmetric case, exact-parameterization means $k=r$ and over-parameterization means $k > r$.
We still use gradient descent to optimize the $L_2$ loss for training:
\begin{align}L_{\mathrm{tr}}(F,G) &= \frac{1}{2m}\sum_{i=1}^m \left(y_i-\langle A_i, FG^\top \rangle \right)^2, \label{eq: asym_ms_loss}
\end{align}
and the performance metric is still:
$
L(F,G) = \frac{1}{2}\Vert FG^\top -M^\star\|_F^2.
$ To enable the analysis, we assume throughout the paper that the matrices $\{A_i\}_{i=1}^m$ satisfies the Restricted Isometry Property (RIP) of order $2k+1$ with parameter $\delta \leq \tilde{\mathcal{O}}(\frac{1}{\sqrt{kr}})$. (See Definition \ref{def: RIP} for the detailed definition).

Also note that even for the \emph{symmetric matrix sensing problem} where $M^\star$ is positive semi-definite, one can still use \emph{asymmetric parameterization}.
Although doing so seems unnecessary at the first glance, we will soon see using asymmetric parameterization enjoys an \emph{exponential gain}.

\paragraph{Contribution 2: Global Exact Convergence of Gradient Descent for Asymmetric Exact-Parameterization with a Linear Convergence Rate}. Our second contribution is a global exact convergence result for randomly initialized gradient descent, and we show it enjoys a linear convergence rate.\footnote{By exact convergence we mean the error goes to $0$ as $t$ goes to infinity in contrast to prior works, which only guarantee to reach a point with the error proportional to the initialization scale $\alpha$ within a finite number of iterations.}

\begin{theorem}[Informal]\label{informal theorem:main theorem asymmetric exact}
In the exact-parameterization setting ($k=r$), suppose we initialize $F$ and $G$ using a Gaussian distribution with small enough variance $\alpha^2$ and use gradient descent with a small enough constant step size to optimize the asymmetric matrix sensing loss~\eqref{eq: asym_ms_loss}.
Let $F_t$ and $G_t$ denote the factor matrices at the $t$-the iteration. Then, with high probability over the random initialization, there exists $T^{(1)} > 0$ such that we have
\begin{align}\|F_tG_t^\top -M^\star\|_F^2 = \exp\left(-\Omega\left(t\right)\right),\forall t \ge T^{(1)}.
\end{align}

\end{theorem}

Compared to our results, prior results either require initialization to be close to optimal~\citep{ma2021beyond}, or can only guarantee to find a point with an error of similar scale as the initialization~\citep{soltanolkotabi2023implicit}.
In contrast, our result only relies on random initialization and guarantees the error goes to $0$ as $t$ goes to infinity.
Notably, this convergence rate is independent of $\alpha$. See Figure~\ref{figure:loss exact scale different}.

Naturally, such a result is expected by the works \citep[Theorem 1]{ma2021beyond} and \citep{soltanolkotabi2023implicit}. Indeed, one should be able to achieve the initial condition of \citep[Theorem 1]{ma2021beyond} by carefully inspecting the proof of \citep{soltanolkotabi2023implicit} and additional work in translating different measures of balancing and closeness. Our proof is very different from \citep{ma2021beyond} as we further decompose the factors $F$ and $G$, and we only need \cite[Theroem 3.3]{soltanolkotabi2023implicit} to deal with the initial phase.
 


\paragraph{Contribution 3: Global Exact Convergence of Gradient Descent for Asymmetric Over-Parameterization with an Initialization-Dependent Linear Convergence Rate.}
Our next contribution is analogue theorem for the over-parameterization case with the caveat that the initialization scale $\alpha$ also appears in the convergence rate.

\begin{theorem}[Informal]\label{informal theorem:main theorem assymmetric}
In the over-parameterization setting ($k>r$), suppose we initialize $F$ and $G$ using a Gaussian distribution with small enough variance $\alpha^2$ and use gradient descent with a small enough constant step size to optimize the asymmetric matrix sensing loss~\eqref{eq: asym_ms_loss}.
Let $F_t$ and $G_t$ denote the factor matrices at the $t$-the iteration. Then, with high probability over the random initialization, there exists $T^{(2)} > 0$ such that we have
\begin{align}\|F_tG_t^\top -M^\star\|_F^2 = \exp\left(-\Omega\left(\alpha^2t\right)\right), \forall t \ge T^{(2)}.
\end{align}

\end{theorem}
This is also the first global exact convergence result of randomly initialized gradient descent in the over-parameterized case.
Recall that for the symmetric matrix sensing problem, even if $M^\star$ is positive semi-definite, one can still use an asymmetric parameterization $FG^\top$ to learn $M^\star$, and Theorem~\ref{informal theorem:main theorem assymmetric} still holds.
Comparing Theorem~\ref{informal theorem:main theorem assymmetric} and Theorem~\ref{thm:symmetric lower bound}, we obtain a surprising corollary:
\begin{center}
\emph{For the \textbf{symmetric} matrix sensing problem, using \textbf{asymmetric} parameterization is \textbf{exponentially faster} than using symmetric parameterization}.
\end{center}

Also notice that different from Theorem~\ref{informal theorem:main theorem asymmetric exact}, the convergence rate of Theorem~\ref{informal theorem:main theorem assymmetric} also depends on the initialization scale $\alpha$ which we require it to be small.
Empirically we verify this dependency is necessary. See Figure~\ref{figure:loss asymmetric}.
We also study a special case in Section~\ref{sec:toycase} to show the dependency on the initialization scale is necessary in the worst case.

\begin{figure}[t]\label{figure:asymmetric}
    \centering
    \subfigure[Exact-parameterized case]{\label{figure:loss exact scale different}\includegraphics[scale = 0.28]{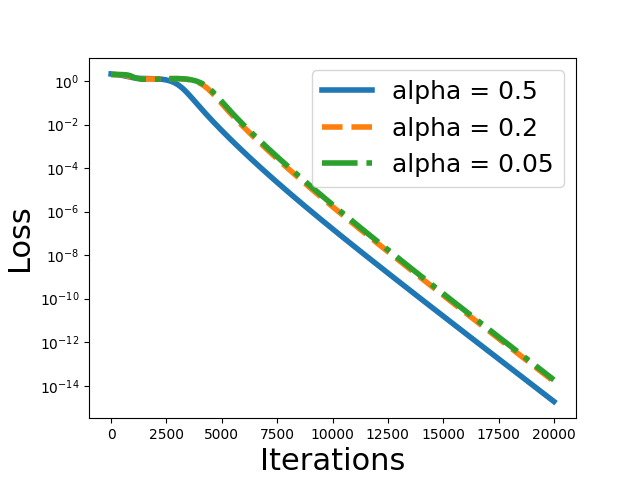}}
    \subfigure[Over-parameterized case]{\label{figure:loss asymmetric}\includegraphics[scale = 0.28]{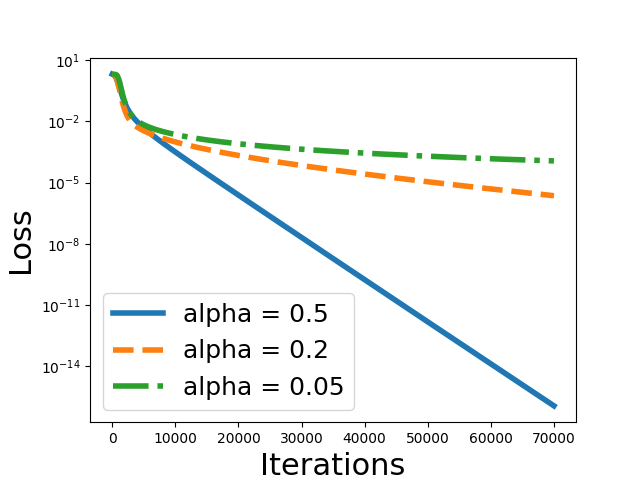}
    }
    \subfigure[Loss curve for our new method]{\label{figure:algfast}\includegraphics[scale=0.28]{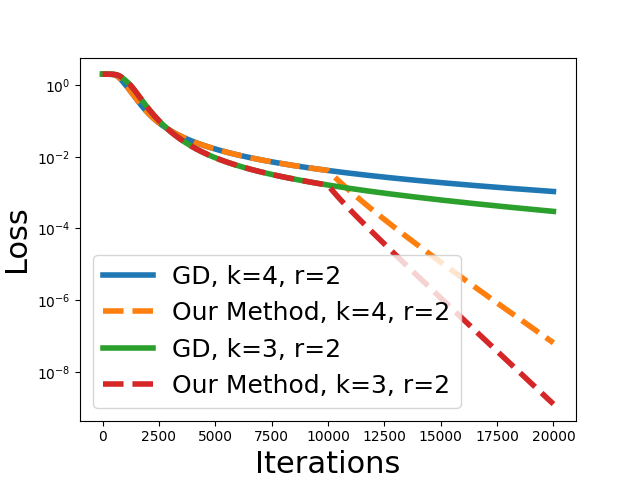}}
    \vspace{-0.4cm}
    \caption{Curve of asymmetric matrix sensing. Figure \ref{figure:loss exact scale different} shows that the convergence rate is linear and independent on the initialization scale under the exact-parameterized case. Figure \ref{figure:loss asymmetric} shows that the convergence rate is linear and dependent on the initialization scale under the over-parameterized case. When the initialization scale is larger, the convergence speed is faster.  Figure \ref{figure:algfast} shows the efficacy of our new method. See \S \ref{sec:exp} for experimental details.}
    \vspace{-0.6cm}
\end{figure}

\paragraph{Technical Insight:} Our key technical finding that gives the exponential acceleration is the \emph{imbalance} of $F$ and $G$.
Denote the imbalance matrix $\Delta_t = F_t^\top F_t - G_t^\top G_t$. 
We show that the converge rate is linear when $\Delta_t$ is positive definite, and the rate depends on the minimum eigenvalue of $\Delta_t.$ 
We use imbalance initialization so that the minimum eigenvalue of $\Delta_0$ is proportional to $\alpha$, we can further show that the minimum eigenvalue $\Delta_t$ will not decrease too much, so the final convergence rate is linear. Furthermore, such a connection to $\alpha$ also inspires us to design a faster algorithm below.


\paragraph{Contribution 4: A Simple Algorithm with Initialization-Independent Linear Convergence Rate for Asymmetric Over-Parameterization.}
Our key idea is to \textit{increase the degree of imbalance} when $F$ and $G$ are close to the optimum. We develop a new simple algorithm to accelerate GD. The algorithm only involves transforming the factor matrices $F$ and $G$ in one of iteration to intensify the degree of imbalance (cf. Equation \eqref{eq: main_matrix change}).
\begin{theorem}[Informal]\label{thm:fast method informal}
In the over-parameterization setting ($k>r$), suppose we initialize $F$ and $G$ using a Gaussian distribution with small enough variance $\alpha^2$, gradient descent  with a small enough constant step size, and the procedure described in Section~\ref{sec:fast} to optimize the asymmetric matrix sensing loss~\eqref{eq: asym_ms_loss}.
Let $F_t$ and $G_t$ denote the factor matrices at the $t$-the iteration. Then, with high probability over the random initialization, there exists $T^{(3)} > 0$ such that we have
\begin{align}\|F_tG_t^\top -M^\star\|_F^2 = \exp\left(-\Omega\left(t-T^{(3)}\right)\right), \forall t \ge T^{(3)}.
\end{align}

\end{theorem}

\section{Related Work}
\label{sec:rel}

The most relevant line of work studies the global convergence of randomly initialized gradient descent for matrix sensing with $L_2$ loss~\citep{zhuo2021computational, stoger2021small, soltanolkotabi2023implicit,tu2016low}.
We compare our results with them in Table~\ref{table:comparison}.

\begin{table}\footnotesize\centering
 \caption{Comparison of previous representative work. The second column shows that the results hold for symmetric matrix factorization/sensing or asymmetric matrix factorization/sensing. The third column lists different types of initialization, where ``Random" means the algorithm uses random initialization (typically Gaussian), ``Local" indicates a requirement for initialization to be close to the optimal point. 
 The fourth column ``exact-cnvrg" represents whether the loss will go to zero when round $T$ goes to infinity.
 The fifth column indicates whether the result applies to over-parameterization case or just the exact-parameterization case.
 The last column lists the convergence rate of algorithms with exact-convergence results. 
 }
 \scalebox{1.07}{
\begin{tabular}{ccccccc}

\midrule[1.5pt]
 &   &  Is Symmetric & Init.  & exact-cnvrg & $k$  Range & Rate \\ \hline & \cite{stoger2021small} & Symmetric & Random & \tiny\XSolidBrush & $k\ge r$ & N/A\\
 \hline & \cite{zhuo2021computational} & Symmetric & Local & \tiny\Checkmark & $k\ge r$ & $\cO(1/T^2)$\\
 \hline & \makecell{\cite{stoger2021small}\\+\\\cite{zhuo2021computational}} & Symmetric & Random & \tiny\Checkmark & $k\ge r$ & $\cO(1/T^2)$ \\
 \hline & \cite{soltanolkotabi2023implicit} & Asymmetric & Random & \tiny\XSolidBrush & $k\ge r$ & N/A\\
 \hline & \cite{tu2016low} & Both & Local & \tiny\Checkmark & $k=r$ & \makecell{$\exp(-\Omega(T))$ }\\
 \hline & \cite{ma2021beyond} & Asymmetric & Local & \tiny\Checkmark & $k=r$ & $\exp(-\Omega(T))$\\
 \hline &Theorem \ref{theorem: exact asymmetric} (our paper) & Asymmetric & Random & \tiny\Checkmark & $k= r$ & $\exp(-\Omega(T))$\\
 \hline & Theorem \ref{theorem:main theorem assymmetric} (our paper) & Asymmetric & Random & \tiny\Checkmark & $k> r$ & $\exp(-\Omega(\alpha^2 T))$\\
 \hline
 & Theorem \ref{thm:symmetric lower bound} (our paper) & Symmetric & Random &\tiny\Checkmark & $k\ge r$ & $\Omega(1/T^2)$\\
 \bottomrule[1.5pt]
\end{tabular}}

\vspace{0.6em}
\label{table:comparison}
\end{table}

\paragraph{Matrix Sensing.}
Matrix sensing aims to recover the low-rank matrix based on measurements. \cite{candes2012exact, liu2012robust} propose convex optimization-based algorithms, which minimize the nuclear norm of a matrix, and \cite{recht2010guaranteed} show that projected subgradient methods can recover the nuclear norm minimizer.
\cite{wu2021implicit} also propose a mirror descent algorithm, which guarantees to converge to a nuclear norm minimizer. 
See \citep{davenport2016overview} for a comprehensive review.

\paragraph{Non-Convex Low-Rank Factorization Approach.}
The nuclear norm minimization approach involves optimizing over a $n \times n$ matrix, which can be computationally prohibitive when $n$ is large.
The factorization approach tries to use the product of two matrices to recover the underlying matrix, but this formulation makes the optimization problem non-convex and is significantly more challenging for analysis. 
For the exact-parameterization setting ($k=r$), \citet{tu2016low,zheng2015convergent} shows the linear convergence of gradient descent when starting at a local point that is close to the optimal point. This initialization can be implemented by the spectral method.
For the over-parameterization scenario ($k > r$), in the symmetric setting, \citet{stoger2021small} shows that with a small initialization, the gradient descent achieves a small error that \emph{dependents} on the initialization scale, rather than the \textit{exact-convergence}. 
\cite{zhuo2021computational} shows exact convergence with $\cO(1/T^2)$ convergence rate in the overparamterization setting.
These two results together imply the global convergence of randomly initialized GD with an $O\left(1/T^2\right)$ convergence rate \emph{upper bound}.
\citet{jin2023understanding} also provides a fine-grained analysis of the GD dynamics.
More recently, \citet{zhang2021preconditioned,zhang2023preconditioned} empirically observe that in practice, in the over-parameterization case, GD converges with a sublinear rate, which is exponentially slower than the rate in the exact-parameterization case, and coincides with the prior theory's upper bound~\citep{zhuo2021computational}.
However, no rigorous proof of the \emph{lower bound} is given whereas we bridge this gap.
On the other hand, \citet{zhang2021preconditioned,zhang2023preconditioned} propose a preconditioned GD algorithm with a shrinking damping factor to recover the linear convergence rate.
\citet{xu2023power} show that the preconditioned GD algorithm with a constant damping factor coupled with small random initialization requires a less stringent assumption on $\cA$ and achieves a linear convergence rate up to some prespecified error. 
 \citet{ma2023global} study the performance of the subgradient method with $L_1$ loss under a different set of assumptions on $\cA$ and showed a linear convergence rate up to some error related to the initialization scale. 
We show that by simply using the \emph{asymmetric parameterization}, without changing the GD algorithm, we can still attain the linear rate.

For the asymmetric matrix setting, 
many previous works \citep{ye2021global,ma2021beyond,tong2021accelerating,ge2017no,du2018algorithmic,tu2016low, zhang2018unified, zhang2018primal,wang2017unified,zhao2015nonconvex} consider the exact-parameterization case ($k=r$).
\citet{tu2016low} adds a balancing regularization term $\frac{1}{8}\|F^\top F-G^\top G\|_F^2$ to the loss function, to make sure that $F$ and $G$ are balanced during the optimization procedure and obtain a local convergence result.
More recently, some works \citep{du2018algorithmic,ma2021beyond,ye2021global} show GD enjoys an \emph{auto-balancing} property where $F$ and $G$ are approximately balanced; therefore, additional balancing regularization is unnecessary. In the asymmetric matrix factorization setting, 
\citet{du2018algorithmic} proves a global convergence result of GD with a diminishing
step size and the GD recovers $M^\star$ up to some error. Later,
 \citet{ye2021global} gives the first global convergence result of GD with a constant step size. \citet{ma2021beyond} shows linear convergence of GD with a local initialization and a larger stepsize in the asymmetric matrix sensing setting.
Although exact-parameterized asymmetric matrix factorization and matrix sensing problems have been explored intensively in the last decade,  our understanding of the over-parameterization setting, i.e., $k>r$, remains limited.  
\citet{jiang2022algorithmic} considers the asymmetric matrix factorization setting, and proves that starting with a small initialization, the vanilla gradient descent sequentially recovers the principled component of the ground-truth matrix. \citet{soltanolkotabi2023implicit}  proves the convergence of gradient descent in the asymmetric matrix sensing setting. 
Unfortunately, both works only prove that GD achieves a small error when stopped early, and the error depends on the initialization scale. Whether the gradient descent can achieve \textit{exact-convergence} remains open, and we resolve this problem by novel analyses.
Furthermore, our analyses highlight the importance of the \emph{imbalance between $F$ and $G$}.

Lastly, we want to remark that we focus on gradient descent for $L_2$ loss, there are works on more advanced algorithms and more general losses~\citep{tong2021accelerating,zhang2021preconditioned,zhang2023preconditioned,zhang2018unified, zhang2018primal,ma2021sign,wang2017unified,zhao2015nonconvex,bhojanapalli2016dropping,xu2023power}. 
We believe our theoretical insights are also applicable to those setups.

\paragraph{Landscape Analysis of Non-convex Low-rank Problems.}
The aforementioned works mainly focus on studying the dynamics of GD.
There is also a complementary line of works that studies the landscape of the loss functions, and shows the loss functions enjoy benign landscape properties such as (1) all local minima are global, and (2) all saddle points are strict~\cite{ge2017no,zhu2018global,li2019symmetry,zhu2021global,zhang2023preconditioned}. 
Then, one can invoke a generic result on \emph{perturbed gradient descent}, which injects noise to GD~\cite{jin2017escape},  to obtain a convergence result.
There are some works establishing the general landscape analysis for the non-convex low-rank problems. 
We remark that injecting noise is required if one solely uses the landscape analysis alone because there exist exponential lower bounds for standard GD~\citep{du2017gradient}.

\paragraph{Slowdown Due to Over-parameterization.}
Similar exponential slowdown phenomena caused by over-parameterization have been observed in other problems beyond matrix recovery, such as teacher-student neural network training~\citep{xu2023over,richert2022soft} and Expectation-Maximization algorithm on Gaussian mixture model~\citep{wu2021randomly,dwivedi2020singularity}.

\section{Preliminaries}

\paragraph{Norm and Big-$\cO$ Notations.} Given a vector $v$, we use $\|v\|$ to denote its Euclidean norm. For a matrix $M$, we use $\|M\|$ to denote its spectral norm and $\|M\|_F$ Frobenius norm. The notations $\cO(\cdot),\Theta(\cdot),$ and $\Omega(\cdot)$ in the rest of the paper only omit absolute constants.

\paragraph{Asymmetric Matrix Sensing.} 
Our primary goal is to recover an unknown fixed rank $r$ matrix $M^\star \in \RR^{n_1\times n_2}$ from data $(y_i,A_i)$, $i=1,\dots,m$ satisfying
$
  y_i = \langle A_i, M^\star \rangle = \mathrm{tr}(A_i ^\top M^\star), i=1,\dots,m,\quad \text{or compactly}\quad y = \cA(M^\star),
$
where $y\in \RR^{m}$ and $\cA: \RR^{n_1\times n_2}\rightarrow \RR^m$ is a linear map with $[\cA(M)]_i= \mathrm{tr}(A_i ^\top M)$. We further denote the singular values of $M^\star$ as $\sigma_1\geq \dots \geq \sigma_r > \sigma_{r+1} =0 = \cdots = \sigma_n$, the condition number $\kappa = \frac{\sigma_1}{\sigma_r}$, and the diagonal singular value matrix as $\Sigma$ with $(\Sigma)_{ii} = \sigma_i$.  
To recover $M^\star$, we minimize the following loss function:
\begin{equation}\label{eq: asym.ms.loss}
    L_{\mathrm{tr}}(F,G) = \frac{1}{2} \|\cA(FG^\top) - y\|^2,
\end{equation}
where $F,G \in \RR^{n\times k}$, where $k\geq r$ is the user-specified rank. 
The gradient descent update rule with a step size $\eta > 0$ with respect to loss~\eqref{eq: asym.ms.loss} can be written explicitly as 
 \begin{align}\label{eq: asym.ms.gd}
     F_{t+1} = F_t- \eta\cA^*\cA(F_tG_t^\top - \Sigma)G_t,~~~
     G_{t+1} = G_t - \eta (\cA^*\cA(F_tG_t^\top-\Sigma))^\top F_t,
\end{align}
where $\cA^*:\RR^m \rightarrow \RR^{n\times n}$ is the adjoint map of $\cA$ and admits an explicit form:  
$\cA^*(z) = \sum_{i=1}^m z_i A_i$.

To make the problem approachable, we shall make the following standard assumption on $\cA$: Restricted Isometry Property (RIP) \citep{recht2010guaranteed}.
\begin{definition}[Restricted Isometry Property]
\label{def: RIP}
    An operator $\cA: \RR^{n_1\times n_2}\to \RR^m$ satisfies the Restricted Isometry Property of order $r$ with constant $\delta>0$ if for all matrices $M:\RR^{n_1\times n_2}$ with rank at most $r$, we have 
    $
       (1-\delta)\|M\|^2_F\le  \|\cA(M)\|^2 \le (1+\delta)\|M\|^2_F.
    $
\end{definition}

From \citep{candes2011tight}, if the matrix $A_i$ has i.i.d. $N(0,\frac{1}{m})$, the operator $\cA$ has RIP of order $2k+1$ with constant $\delta\in (0,1)$ when $m = \widetilde{\Omega}\left(\frac{nk}{\delta^2}\right)$. Thus, $m = \widetilde{\Omega}(nk^2r)$ is needed with \ref{eq: delta.requirement}.

\paragraph{Diagonal Matrix Simplification.} 
Since both the RIP and the loss are invariant to orthogonal transformation, we assume without loss generality that $M^\star = \Sigma$ in the rest of the paper for clarity, following prior work~\citep{ye2021global,jiang2022algorithmic}.
For the same reason, we also assume $n_1=n_2=n$ to simplify notations, and our results can be easily extended to $n_1\neq n_2$.



\paragraph{Symmetric Matrix Factorization.}
In this setting, we further assume $M^\star$ is symmetric and positive semidefinite, and $\cA$ is the identity map. Since $M^\star$ admits a factorization $M^\star = F_\star F_\star^\top$ for some $F_\star \in \RR^{n\times r}$, we can force the factor $F = G =X$ in \eqref{eq: asym.ms.loss} and the loss becomes 
$
L(X) = \frac{1}{2}\| XX^\top-\Sigma\|_F^2.
$
Here, the factor $X \in \RR^{n\times k}$. We shall focus on the over-parameterization setting, i.e., $k>r$ in the Setion~\ref{sec: lbsymf} below.
The gradient descent with a step size $\eta>0$  becomes
\begin{align}\label{eq: sym.gd}
    X_{t+1} = X_t - 2\eta (X_tX_t^\top -\Sigma)X_t.
\end{align}



\section{Lower Bound of Symmetric Matrix Factorization}
\label{sec: lbsymf}

We present a sublinear lower bound of the convergence rate of the gradient descent \eqref{eq: sym.gd} for symmetric matrix factorization starting from a small random initialization. Our result supports the empirical observations that overparmetrization slows down gradient descent \citep{zhuo2021computational,zhang2021preconditioned, zhang2023preconditioned} and Figure~\ref{figure:symmetric}.

%
%

\begin{theorem}\label{thm:symmetric lower bound}
   Let $X_0 = \alpha 
    \cdot \tilde{X_0}$, where each entry is independent initialized from Gaussian distribution $\cN(0,1/k)$.
    For some universal constants $c_i, 1\leq i\leq 7$, if the gradient descent method \eqref{eq: sym.gd} starting at $X_0$ with the initial scale $\alpha$, the search rank $k$, and the stepsize $\eta$ satisfying that 
    \begin{equation}\label{eq: sym.mf.parameter.rqmt}
    0<\alpha \le \frac{c_1\sqrt{\sigma_1}}{\sqrt{n}\log(r\sqrt{n})}, \quad k\ge c_2\left((r\kappa)^2\log(r\sqrt{\sigma_1}/\alpha)\right)^4,\quad   
        \text{and}\quad 
        0<\eta \le  \frac{c_3}{n^2\kappa\sigma_1},
    \end{equation}
    then with probability at least $1-2n^2 \exp(-\sqrt{c_4k})-2n\exp(-c_5k/4)$, for all $t\geq T^{(0)} = \frac{c_6\log(r\sqrt{\sigma_1})/\alpha}{\eta\sigma_r}$, we have
    \begin{equation}\label{eq: sym.mf.formal.lb}
        \|X_tX_t^\top -\Sigma\|_F^2 \ge \left(\frac{c_7\log(\sqrt{r\sigma_1}/\alpha)\alpha^2}{8\sigma_r \eta n t}\right)^2, \quad \forall t \ge T^{(0)}.
    \end{equation}
\end{theorem}

The proof of Theorem \ref{informal thm:symmetric lower bound} demonstrates that, following a rapid convergence phase, the gradient descent eventually transitions to a \emph{sublinear} convergence rate. Also, the over-parameterization rank $k$ 
 is subject to a lower bound requirement in Eq. \eqref{eq: sym.mf.parameter.rqmt} that depends on $\alpha$. However, since $\alpha$ only
 appears in the logarithmic term, this requirement is not overly restrictive. It is also consistent with the phenomenon that the gradient descent first converges to a small error that depends on $\alpha$ with a linear convergence rate \citep{stoger2021small}, since our lower bound has a term $\alpha^4$. 


\subsection{Proof Sketch of Theorem \ref{thm:symmetric lower bound}}
We provide a proof sketch of Theorem \ref{thm:symmetric lower bound} in this section,  deferring the details to Appendix \ref{appendix: proof of lowr bound}. 

The main intuition of Theorem \ref{thm:symmetric lower bound} is that the last $n-r$ rows of $X_t$, corresponding to the space of $0$ eigenvalues of $\Sigma$, converge to $0$ at speed no faster than $\frac{1}{T^2}$. To make this intuition precise, for each $t\geq 0$, we let $X_t \in \RR^{n\times k} = [x_1^t ,\cdots,x_n^t ]^\top $ where $x_i^t \in \RR^{k}$. We let the potential function be $A_t = \sum_{i>r}\|x_i^t\|^2$. We aim to show the following two key inequalities, 
\begin{subequations}\label{eq: key_inequalities}
\begin{gather}
    \|x_i^T\|^2  \ge  \alpha^2 /8,\text{ for all $i>r$},
    \label{eq: At.update_Initial}\\
    A_{t+1}    \ge A_t(1-\cO(\eta A_t)), \text{ for all $t\ge T^{(0)}$}. \label{eq: At.update}
\end{gather}
\end{subequations} 

Suppose \eqref{eq: key_inequalities} is true, then it
implies that $A_t \geq \cO\left(\frac{\alpha^2}{t}\right)$ for all $t\ge T^{(0)}$. Since $(X_tX_t^\top -\Sigma)_{ii} = \|x_i\|^2$, the lower bound \eqref{eq: sym.mf.formal.lb} is established by noting that 
$   \|X_tX_t^\top -\Sigma\|_F^2 \ge \sum_{i>r}\|x_i^t\|^4 \ge A_t^2/n$, where the last inequality uses the Cauchy's inequality. See more details in Appendix \ref{appendix: proof of lowr bound}. 

\section{Convergence of Asymmetric Matrix Sensing}\label{sec:assy}
Here, we investigate the dynamic of GD in the context of the asymmetric matrix sensing problem. Surprisingly, we demonstrate that the convergence rate of gradient descent for asymmetric matrix sensing problems is linear, so long as the initialization is \emph{imbalanced}. However, this linear rate is contingent upon the chosen initialization scale.

\subsection{A Toy Example of Asymmetric Matrix Factorization}\label{sec:toycase}
We first use a toy example of asymmetric matrix factorization to demonstrate the behavior of GD.
If we assume $\cA$ is the identity map, and the loss and the GD update become 
\begin{align}\label{eq: asymms.pop}
    L(F,G) = &\frac{1}{2}\|FG^\top -\Sigma\|_F^2.\\
    F_{t+1} =F_t- 
 \eta(F_tG_t^\top -\Sigma)G_t,~~~
 & G_{t+1}= G_t-\eta(F_tG_t^\top -\Sigma)^\top F_t \label{eq: asym.mf.gd}
\end{align} 

The following theorem tightly characterizes the convergence rate for a toy example.

\begin{theorem}\label{thm: asym.toy_case}
    Consider the asymmetric matrix factorization \eqref{eq: asymms.pop}, with $k=r+1$. Choose $\eta  \in [0,1/6]$ and $\alpha \in [0,1]$.  Assume that the diagonal matrix $\Sigma = \text{diag}(\sigma_1,\dots,\sigma_n)$, where $\sigma_i=1$ for $i\leq r$ and is $0$ otherwise. Also assume that gradient descent \eqref{eq: asym.mf.gd} starts at $F_0, G_0$, where $(F_0)_{ii} = \alpha$ for $1\le i\le k$, and $(G_0)_{ii} = \alpha$ for $1\le i\le r$, $(G_0)_{ii} = \alpha/3$ for $i=r+1,$ and all other entries of $F_0$ and $G_0$ are zero. Then, the iterate $(F_t,G_t)$ of \eqref{eq: asym.mf.gd} satisfies that  
    \begin{align*}
        \frac{\alpha^4}{36}(1-4\eta \alpha^2)^{2t} \le \|F_tG_t^\top -\Sigma\|_F^2 \le 4n\cdot (1-\eta\alpha^2/4)^{(t-T_1)}, \;\forall t \ge T_1.
    \end{align*}
    where $T_1 = c_1\log(1/\alpha)/\eta$, and $c_1$ is a universal constant.
\end{theorem}
The above initialization implicitly assumes that we know the singular vectors of $\Sigma$. Such an assumption greatly simplifies our presentations below. Note that we have a different initialization scale for $F_t$ and $G_t$. As we shall see, such an \emph{imbalance} is the key to establishing the convergence of $F_tG_t^\top$.

We introduce some notations before our proof. Denote the matrix of the first $r$ row of $F,G$ as $U,V \in \RR^{r\times k}$ respectively, and the matrix of the last $n-r$ row of $F,G$ as $J,K \in \RR^{(n-r)\times k}$ respectively. Further denote the corresponding iterate of gradient descent as $U_t$, $V_t$, $J_t$, and $K_t$. The difference $F_tG_t^\top -\Sigma$ can be written in a block form as 
$
F_tG_t^\top -\Sigma  = 
\begin{pmatrix}
U_tV_t^\top -\Sigma_r & J_tV_t^\top   \\
U_tK_t^\top  & J_tK_t^\top  \\
\end{pmatrix} 
$
where $\Sigma_r\in \RR^{r\times r}$ is the identity matrix. Hence, we may bound $\|F_tG_t^\top -\Sigma\|$ by 
\begin{align}
\|J_tK_t^\top \|\le \|F_tG_t^\top -\Sigma\|\le \|U_tV_t^\top -\Sigma_r\| + \|J_tV_t^\top \|+\|U_tK_t^\top \| + \|J_tK_t^\top \|.\label{eq: main_two bound of loss}
\end{align}
From \eqref{eq: main_two bound of loss}, we shall upper bound $\|U_tV_t^\top -\Sigma_r\|$, $\|J_tV_t^\top \|$, $\|U_tK_t^\top \|$, 
 and $\|J_tK_t^\top \|$, and lower bound $\|J_t K_t\|^\top$. Let us now prove Theorem \ref{thm: asym.toy_case}.

\begin{proof}
With our particular initialization and the formula of gradient descent \eqref{eq: asym.mf.gd}, we have the following equality for all $t$:
\begin{subequations}\label{eq: simpleUpdates}
\begin{gather}
U_tK_t^\top   = 0, \quad J_tV_t^\top  = 0,\quad  U_t = V_t,\quad  
    J_t = a_t A,\quad  K_t= b_tA, \quad U_t = (\alpha_t I_r,0), \\
    a_0  = \alpha, \quad b_0  = \alpha/3, \quad \alpha_0 = \alpha\\
    a_{t+1} = a_t - \eta a_t b_t^2,\label{toycase:updating rule of a}\\
    b_{t+1} = b_t - \eta a_t^2b_t.\label{toycase:updating rule of b}\\
    \alpha_{t+1} = 
    \alpha_t(1+\eta-\eta \alpha_t^2),
\end{gather}
\end{subequations}
where $A \in \RR^{(n-r)\times k}$ is the matrix that $(A)_{1k} = 1$ and other elements are all zero, and $a_t,b_t,\alpha_t \in \RR$.
We leave the detailed verification of \eqref{eq: simpleUpdates} to Appendix \ref{sec: proof_toy_case}. By considering \eqref{eq: main_two bound of loss} and \eqref{eq: simpleUpdates}, we see that we only need to keep track of three sequences $a_t,b_t$, $\alpha_t$. In particular, we have the following inequalities for $a_t,b_t,\alpha_t$ for all $t\geq T_1$:
\begin{equation}\label{eq: a_t,b_t,alpha_t_in_toy_case}
a_t \in \left[\frac{1}{2} \alpha, \alpha\right], \;
b_t \in \left[\frac{\alpha}{3}(1-4\eta\alpha^2)^t , \;\frac{\alpha}{3}(1-\frac{\eta\alpha^2}{4})^{t/2}\right], \;\text{and}\;|\alpha_{t}-1|\leq (1-\eta/2)^{t-T_1}.
\end{equation}
It is then easy to derive the upper and lower bounds. We leave the detail in checking \eqref{eq: a_t,b_t,alpha_t_in_toy_case} to Appendix \ref{sec: proof_toy_case}. Our proof is complete.
\end{proof}


\paragraph{Technical Insight.} This proof of the toy case tells us why the imbalance initialization in the asymmetric matrix factorization helps us to break the $\Omega(1/T^2)$ convergence rate lower bound of the symmetric case. Since we initialize $F_0$ and $G_0$ with a \emph{different scale}, this difference makes the norm of $K$ converge to zero at a linear rate while keeping $J$ larger than a constant. However, in the symmetric case, we have $a_t=b_t$, so they must both converge to zero when the loss goes to zero (as $\|F_tG_t^\top -\Sigma\| \ge a_tb_t$), leading to a sublinear convergence rate. In short, the imbalance property in the initialization causes faster convergence in the asymmetric case.

\subsection{Theoretical Results for Asymmetric Matrix Sensing}\label{sec: result and asymmetric}

Motivated by the toy case in Section \ref{sec:toycase}, the imbalance property is the key ingredient for a linear convergence rate. If we use a slightly imbalanced  initialization 
$F_0 = \alpha\cdot \tilde{F}_0, G_0 = (\alpha/3) \cdot \tilde{G}_0$, where the elements of $\tilde{F}_0$ and $\tilde{G}_0$ are $\cN(0,1/n)$, we have 
$\|F_0^\top F_0 -G_0^\top G_0\| = \Omega(\alpha^2)$. Then, we can show that
the imbalance property keeps true when the step size is small, and thus, the gradient descent \eqref{eq: asym.ms.gd} converges with a linear rate similar to the toy case.

Our result is built upon the recent work \citep{soltanolkotabi2023implicit} in dealing with the initial phase. Define the following quantities $\alpha_0,\eta_0$ according to \cite[Theorem 1] {soltanolkotabi2023implicit}:
\begin{equation}\label{eq: req.alpha0_eta0}
        \alpha_0 = \frac{c\sqrt{\sigma_1}}{k^5\max\{2n,k\}^2} \cdot \left(\frac{\sqrt{k}-\sqrt{r-1}}{\kappa^2\sqrt{\max\{2n,k\}}}\right)^{C\kappa},
        \eta_0 = \frac{c}{k^5\sigma_1\log\left(\frac{2\sqrt{2\sigma_1}}{\alpha (\sqrt{k}-\sqrt{r-1}}\right)},
\end{equation}
where $c$ and $C$ are some numerical constants. 
Below, we show exact convergence results for both $k=r$ and $k>r$.

\begin{theorem}\label{theorem:main theorem assymmetric}
Consider the matrix sensing problem \eqref{eq: asym_ms_loss} and the gradient descent \eqref{eq: asym.ms.gd}. For some numerical constants $c_i>0$, $1\leq i\le 7$,  if the search rank $k$ satisfies $r<k <\frac{n}{8}$, the initial scale $\alpha$ and $\eta$ satisfy
\begin{equation}\label{eq: alpha_eta_requirement}
    \alpha \leq  \min\left\{c_1\kappa^{-2}\sqrt{\sigma_r},\alpha_0 \right\},\ \ \ \ \ \eta \leq  \min\Big\{ c_1\alpha^4/\sigma_1^3,\eta_0\Big\},
\end{equation}
where $\alpha_0,\eta_0$ are defined in \eqref{eq: req.alpha0_eta0}, and the operator $\cA$ has the RIP of order $2k+1$ with constant $\delta$ satisfying
\begin{equation}\label{eq: delta.requirement}
\delta\sqrt{2k+1} \leq  \min\left\{c_1\kappa^{-6}\log^{-1}(\sqrt{\sigma_r}/(n\alpha)), \frac{c_1}{\kappa^3\sqrt{r}}, 1/128\right\},
\end{equation}
then the gradient descent \eqref{eq: asym.ms.gd} starting with $F_0 = \alpha\cdot \tilde{F}_0, G_0 = (\alpha/3) \cdot \tilde{G}_0$, where $\tilde{F}_0, \tilde{G}_0 \in \RR^{n\times k}$ whose entries are i.i.d. $\cN(0,1/n)$, 
 satisfies that 
    \begin{equation}\label{eq: asym.over.speed}
        \|F_tG_t^\top - \Sigma \|_F^2 \le \frac{\sigma_r^4n}{c_7\alpha^4\kappa^2}\left(1-\frac{\eta\alpha^2}{8}\right)^{t/4}, \quad \forall t \ge T^{(1)},
    \end{equation}
with probability at least $1-2e^{-c_2n}-c_3 e^{-c_4k} -(c_5\upsilon)^{(k-r+1)}$, where
 $T^{(1)} = c_6\log(\sqrt{\sigma_r}/n\alpha\upsilon)/\eta \sigma_r)$ and $v\in [0,1]$ is an arbitrary parameter. 
\end{theorem}

Next, we state our results on exact parametrization.
\begin{theorem}\label{theorem: exact asymmetric}
   Consider the same setting as Theorem \ref{theorem:main theorem assymmetric} except assuming $k=r$, then with probability at least $1-2e^{-c_2n} - c_3 e^{-c_4k} - c_5\upsilon,$ the gradient descent \eqref{eq: asym.ms.gd} achieves  
\begin{equation}\label{eq: asym.exact.speed}
        \|F_tG_t^\top - \Sigma\|_F^2 \le 2n\sigma_r \cdot \left(1-\frac{\eta\sigma_r^2}{64\sigma_1}\right)^{t}, \quad \forall t \ge T^{(2)},
\end{equation}
where $T^{(2)} =c_7\log(\sqrt{\sigma_r}/n\alpha\upsilon)/\eta \sigma_r)$ for some numerical constant $c_7$.
\end{theorem}
Now we highlight two bullet points of Theorem \ref{theorem:main theorem assymmetric} and \ref{theorem: exact asymmetric}.

\paragraph{Exact Convergence.}
The main difference between the above theorems and previous convergence results in \citep{soltanolkotabi2023implicit} is that we prove the \textit{exact convergence} property, i.e., the loss finally degenerates to zero when $T$ tends to infinity (cf. Table~\ref{table:comparison}).
Moreover, we prove that the convergence rate of the gradient descent depends on the initialization scale $\alpha$, which matches our empirical observations in Figure \ref{figure:asymmetric}.

\paragraph{Discussions about Parameters.} 
First, since we utilize the initial phase result in \citep{soltanolkotabi2023implicit} to guarantee that the loss degenerates to a small scale, our parameters $\delta$, $\alpha$, and $\eta$ should satisfy the requirement $\delta_0 = \cO(\frac{1}{\kappa^3\sqrt{r}}), \alpha_0, \eta_0$ in \citep{soltanolkotabi2023implicit}. We further require $\delta_{2k+1} = \widetilde{\cO}(\kappa^{-6})$, $\alpha = \cO(\kappa^{-2}\sqrt{\sigma_r})$, which are both polynomials of the conditional number $\kappa.$
In addition, the step size $\eta$ has the requirement  $\eta = \cO(\alpha^4/\sigma_1^3)$, which can be much smaller than the requirements $\eta = \widetilde{\cO}(1/\kappa^{5}\sigma_1)$ in \citep{soltanolkotabi2023implicit}.
In Section \ref{sec:fast}, we propose a novel algorithm that allows larger learning rate which is independent of $\alpha.$

\paragraph{Technical insight.} Similar to the asymmetric matrix factorization case in the proof of Theorem \ref{thm: asym.toy_case}, the main effort is in characterizing the behavior of $J_tK_t^\top$. In particular, the update rule of $K_t$ is 
\begin{align}
    K_{t+1}=K_t(1-\eta F_t^\top F_t) +\eta E\label{eq: updating rule of K maintext},
\end{align}
where $E$ is some error matrix since $\cA$ is not an identity. Because of our initialization, we know the following holds for $t=0$ and $\Delta_t = F^\top _ t F_t-G^\top_t G_t$,
\begin{equation}\label{eq: updating rule of Delta maintext}
    c \alpha^2 I \preceq \Delta_t \preceq C\alpha^2I.
\end{equation}
for some numerical constant $c,C>0$. Hence, we can show $\|K_t\|$ shrinks towards $0$ so long as \eqref{eq: updating rule of K maintext} is true, $E=0$, and $G_t$ is well-bounded. Indeed, we can prove \eqref{eq: updating rule of Delta maintext} and $G_t,J_t$ upper bounded for all $t\geq 0$ via a proper induction. We may then be tempted to conclude $J_tK_t^\top$ converges to $0$. However, the actual analysis of the gradient descent \eqref{eq: asym.ms.gd} for matrix sensing is much more complicated due to the error $E$. It is now unclear whether $\|K_t\|$ will shrink under \eqref{eq: updating rule of Delta maintext}.  To deal with it, we further consider the structure of $E$. We leave the details to Appendix \ref{appendix: proof of asymmetric}.

\section{A Simple and Fast Convergence Method}\label{sec:fast}
As discussed in Section \ref{sec:assy}, the fundamental reason that the convergence rate depends on the initialization scaling $\alpha$ is that the imlabace between $F$ and $G$ determines the convergence rate, but the imbalance between $F$ and $G$ remains at the initialization scale.
This observation motivates us to do a straightforward additional step in one iteration to \emph{intensify the imbalance}.
Specifically, suppose at the $T_0$ iteration we have reached a neighborhood of an optimum that satisfies: 
$\|\cA^*\cA(\widetilde{F}_{T^{(3)}}\widetilde{G}_{T^{(3)}}^\top -\Sigma)\| \le \gamma $  where the radius $\sigma_r^{1/4}\cdot \|F_{T^{(3)}}G_{T^{(3)}}^\top\|^{3/4}/8$ is chosen for some technical reasons (cf. Section~\ref{appendix: proof of fast}). Here, we use $\widetilde{F}_t$ and $\widetilde{G}_t$ to denote the iterates before we make the change we describe below and  $F_t$ and $G_t$ to denote the iterates after make the change.

Let the singular value decomposition of $\widetilde{F}_{T^{(3)}} = A\Sigma' B$ with the diagonal matrix $\Sigma' \in \RR^{k\times k}$ and $\Sigma'_{ii} = \sigma'_i$, then let $\Sigma_{inv} \in \RR^{k\times k}$ be a diagonal matrix and $(\Sigma_{inv})_{ii} = \beta /\sigma'_i$ for some small constant $\beta = O(\sigma_r)$, then we transform the matrix $F_{T^{(3)}}, G_{T^{(3)}}$ by
\begin{align}
    F_{T^{(3)}} = \widetilde{F}_{T^{(3)}} B^\top \Sigma_{inv}, G_{T^{(3)}} = \widetilde{G}_{T^{(3)}} B \Sigma_{inv}^{-1}\label{eq: main_matrix change}
\end{align}

We can show that, when $F$ and $G$ have reached a local region of an optimum, their magnitude will have similar scale as $M^\star$. Therefore, the step Equation~\eqref{eq: main_matrix change} can create an imbalance between $F$ and $G$ as large the magnitude of $M^\star$, which is significantly larger than the initial scaling $\alpha$.
The following theorem shows we can obtain a convergence rate independent of the initialization scaling $\alpha$. The proof is deferred to Appendix~\ref{appendix: proof of fast}.

\begin{theorem}\label{thm:fast method}
With the same setting as Theorem \ref{theorem:main theorem assymmetric}, suppose that at the step $T^{(3)}$ we have $\|\cA^*\cA(\widetilde{F}_{T^{(3)}}\widetilde{G}_{T^{(3)}}^\top -\Sigma)\| \le \gamma$ for some $\gamma > 0$,
and we do one step as in Equation~\eqref{eq: main_matrix change}. 
Then, with probability at least $1-2e^{-c_2n}-c_3 e^{-c_4k} -(c_5\upsilon)^{(k-r+1)}$, we have for all $t >T^{(3)}$,  
    \begin{align*}
        \|F_tG_t^\top -\Sigma\|_F^2 \le \frac{n\beta^{12}}{\sigma_1^4}\left(1-\frac{\eta \beta^2}{2}\right)^{2(t-T^{(3)})},
    \end{align*}
so long as  $0<c_7 \gamma^{1/6}\sigma_1^{1/3} \le \beta \le c_8\sigma_r$, and the step size satisfies
    $\eta \le c_{9}\beta^2/\sigma_1^2
$ from the iteration $T^{(3)}\leq  c_{10}\log(\sqrt{\sigma_r}/n\alpha\upsilon)/\eta \sigma_r$ for some positive numerical constants $c_{i}, i=1,\dots,10$.
    


    
\end{theorem}

\section{Conclusion}
This paper demonstrated qualitatively different behaviors of GD in the exact-parameterization and over-parameterization scenarios in symmetric and asymmetric settings.
For the symmetric matrix sensing problem, we provide a $\Omega(1/T^2)$ lower bound.
For the asymmetric matrix sensing problem, we show that the gradient descent converges at a linear rate, where the rate is dependent on the initialization scale. Moreover, we introduce a simple procedure to get rid of the initialization scale dependency. 
We believe our analyses are also useful for other problems, such as deep linear networks.

\newpage

\centerline{\begin{Large} \textbf{Appendix}\end{Large}}
\tableofcontents
\appendix
\newpage

\centerline{\begin{Large} \textbf{Appendix}\end{Large}}

\section{Related Work}
\label{sec:rel_appendix}

\paragraph{Matrix Sensing.}
Matrix sensing aims to recover the low-rank matrix based on measurements. \cite{candes2012exact, liu2012robust} propose convex optimization-based algorithms, which minimize the nuclear norm of a matrix, and \cite{recht2010guaranteed} show that projected subgradient methods can recover the nuclear norm minimizer.
\cite{wu2021implicit} also propose a mirror descent algorithm, which guarantees to converge to a nuclear norm minimizer. 
See \citep{davenport2016overview} for a comprehensive review.

\paragraph{Non-Convex Low-Rank Factorization Approach.}
The nuclear norm minimization approach involves optimizing over a $n \times n$ matrix, which can be computationally prohibitive when $n$ is large.
The factorization approach tries to use the product of two matrices to recover the underlying matrix, but this formulation makes the optimization problem non-convex and is significantly more challenging for analysis. 
For the exact-parameterization setting ($k=r$), \citet{tu2016low,zheng2015convergent} shows the linear convergence of gradient descent when starting at a local point that is close to the optimal point. This initialization can be implemented by the spectral method.
For the over-parameterization scenario ($k > r$), in the symmetric setting, \citet{stoger2021small} shows that with a small initialization, the gradient descent achieves a small error that \emph{dependents} on the initialization scale, rather than the \textit{exact-convergence}. 
\cite{zhuo2021computational} shows exact convergence with $\cO(1/T^2)$ convergence rate in the overparamterization setting.
These two results together imply the global convergence of randomly initialized GD with an $O\left(1/T^2\right)$ convergence rate \emph{upper bound}.
\citet{jin2023understanding} also provides a fine-grained analysis of the GD dynamics.
More recently, \citet{zhang2021preconditioned,zhang2023preconditioned} empirically observe that in practice, in the over-parameterization case, GD converges with a sublinear rate, which is exponentially slower than the rate in the exact-parameterization case, and coincides with the prior theory's upper bound~\citep{zhuo2021computational}.
However, no rigorous proof of the \emph{lower bound} is given whereas we bridge this gap.
On the other hand, \citet{zhang2021preconditioned,zhang2023preconditioned} propose a preconditioned GD algorithm with a shrinking damping factor to recover the linear convergence rate.
\citet{xu2023power} show that the preconditioned GD algorithm with a constant damping factor coupled with small random initialization requires a less stringent assumption on $\cA$ and achieves a linear convergence rate up to some prespecified error. 
 \citet{ma2023global} study the performance of the subgradient method with $L_1$ loss under a different set of assumptions on $\cA$ and showed a linear convergence rate up to some error related to the initialization scale. 
We show that by simply using the \emph{asymmetric parameterization}, without changing the GD algorithm, we can still attain the linear rate.

For the asymmetric matrix setting, 
many previous works \citep{ye2021global,ma2021beyond,tong2021accelerating,ge2017no,du2018algorithmic,tu2016low, zhang2018unified, zhang2018primal,wang2017unified,zhao2015nonconvex} consider the exact-parameterization case ($k=r$).
\citet{tu2016low} adds a balancing regularization term $\frac{1}{8}\|F^\top F-G^\top G\|_F^2$ to the loss function, to make sure that $F$ and $G$ are balanced during the optimization procedure and obtain a local convergence result.
More recently, some works \citep{du2018algorithmic,ma2021beyond,ye2021global} show GD enjoys an \emph{auto-balancing} property where $F$ and $G$ are approximately balanced; therefore, additional balancing regularization is unnecessary. In the asymmetric matrix factorization setting, 
\citet{du2018algorithmic} proves a global convergence result of GD with a diminishing
step size and the GD recovers $M^*$ up to some error. Later,
 \citet{ye2021global} gives the first global convergence result of GD with a constant step size. \citet{ma2021beyond} shows linear convergence of GD with a local initialization and a larger stepsize in the asymmetric matrix sensing setting.
Although exact-parameterized asymmetric matrix factorization and matrix sensing problems have been explored intensively in the last decade,  our understanding of the over-parameterization setting, i.e., $k>r$, remains limited.  
\citet{jiang2022algorithmic} considers the asymmetric matrix factorization setting, and proves that starting with a small initialization, the vanilla gradient descent sequentially recovers the principled component of the ground-truth matrix. \citet{soltanolkotabi2023implicit}  proves the convergence of gradient descent in the asymmetric matrix sensing setting. 
Unfortunately, both works only prove that GD achieves a small error when stopped early, and the error depends on the initialization scale. Whether the gradient descent can achieve \textit{exact-convergence} remains open, and we resolve this problem by novel analyses.
Furthermore, our analyses highlight the importance of the \emph{imbalance between $F$ and $G$}.

Lastly, we want to remark that we focus on gradient descent for $L_2$ loss, there are works on more advanced algorithms and more general losses~\citep{tong2021accelerating,zhang2021preconditioned,zhang2023preconditioned,zhang2018unified, zhang2018primal,ma2021sign,wang2017unified,zhao2015nonconvex,bhojanapalli2016dropping,xu2023power}. 
We believe our theoretical insights are also applicable to those setups.

\paragraph{Landscape Analysis of Non-convex Low-rank Problems.}
The aforementioned works mainly focus on studying the dynamics of GD.
There is also a complementary line of works that studies the landscape of the loss functions, and shows the loss functions enjoy benign landscape properties such as (1) all local minima are global, and (2) all saddle points are strict~\cite{ge2017no,zhu2018global,li2019symmetry,zhu2021global,zhang2023preconditioned}. 
Then, one can invoke a generic result on \emph{perturbed gradient descent}, which injects noise to GD~\cite{jin2017escape},  to obtain a convergence result.
There are some works establishing the general landscape analysis for the non-convex low-rank problems. 
\cite{zhang2021general} obtains less conservative conditions for guaranteeing the non-existence of spurious second-order critical points and the strict saddle property, for both symmetric and asymmetric low-rank minimization problems.  The paper \cite{bi2022local} analyzes the gradient descent for the symmetric case and asymmetric case with a regularized loss. They provide the local convergence result using PL inequality, and show the global convergence for the perturbed gradient descent.
We remark that injecting noise is required if one solely uses the landscape analysis alone because there exist exponential lower bounds for standard GD~\citep{du2017gradient}.

\paragraph{Slowdown Due to Over-parameterization.}
Similar exponential slowdown phenomena caused by over-parameterization have been observed in other problems beyond matrix recovery, such as teacher-student neural network training~\citep{xu2023over,richert2022soft} and Expectation-Maximization algorithm on Gaussian mixture model~\citep{wu2021randomly,dwivedi2020singularity}.

\section{Proof of Theorem \ref{thm:symmetric lower bound}}\label{appendix: proof of lowr bound}
In this proof, we denote \begin{align}X \in \RR^{n\times k} = \begin{bmatrix}
    x_1^\top \\
    x_2^\top \\
    \cdots\\
    x_n^\top 
\end{bmatrix},\label{eq:decompose X}\end{align} where $x_i \in \RR^{k\times 1}$ is the transpose of the row vector. Since the updating rule can be written as 
\begin{align*}
    X_{t+1} = X_t-\eta (X_tX_t^\top -\Sigma)X_t,
\end{align*}
where we choose $\eta$ instead of $2\eta$ for simplicity, which does not influence the subsequent proof.
By substituting the equation \eqref{eq:decompose X}, the updating rule can be written as 
\begin{align*}
(x_i^{t+1})^\top  = (1-\eta(\|x_i^t\|^2-\sigma_i))x_i^\top  - \sum_{j=1,j\neq i}^n\eta((x_i^t)^\top x_j^t(x_j^t)^\top )
\end{align*}
where $\sigma_i = 0$ for $i> r.$
Denote $$\theta = \max_{j,k}\frac{(x_j^\top x_k)^2}{\Vert x_j\Vert^2 \Vert x_k\Vert^2}$$ is the maximum angle between different vectors in $x_1,\cdots,x_n$.  We start with the outline of the proof. 
\subsection{Proof outline of Theorem \ref{thm:symmetric lower bound}}
Recall we want to establish the key inequalities \eqref{eq: key_inequalities}. 
The updating rule \eqref{eq: sym.gd} gives the following lower bound of $x_i^{t+1}$ for $i>r$:
\begin{equation}\label{eq: lb.t+1th.xi}   
    \|x_i^{t+1}\|^2 \ge  \|x_i^t\|^2 \left(1-2\eta \theta_{t}^U\sum_{j\le r}\|x_j^t\|^2  -2\eta\sum_{j> r}\|x_j^t\|^2 \right),
\end{equation}
where the quantity $\theta_t^U = \max_{i,j:\min\{i,j\} \le r}\theta_{ij,t}$ and the square cosine $\theta_{ij,t}=\cos^2 \angle (x_i,x_j)$. 
Thus, to establish the key inequalities \eqref{eq: key_inequalities}, we need to control the quantity $\theta_t^U$. Our analysis then consists of three phases. In the last phase, we show \eqref{eq: key_inequalities} holds and our proof is complete. 

In the first phase, we show that $\|x_i^t\|^2$ for $i\le r$ becomes large, while $\|x_i^t\|^2$ for $i > r$ still remains small yet bounded away from $0$. In addition, the quantity $\theta_{ij,t}$ remains small. Phase 1 terminates when $\|x_i^t\|^2$ is larger than or equal to $\frac{3}{4}\sigma_i$.

After the first phase terminates, in the second and third phases, we show that $\theta_t^U$ converges to $0$ linearly and the quantity 
$\theta_t^U\sigma_1/\sum_{j>r}\|x_j^t\|^2$ converges to zero at a linear rate as well. We also keep track of the magnitude of $\|x^t_i\|^2$ and show $\|x^t_i\|$ stays close to $\sigma_i$ for $i\leq r$, and $\|x^t_i\|^2 \leq 2\alpha^2$ for $i>r$.

%

The second phase terminates once $\theta_t^U \le \cO(\sum_{j>r}\|x_j^t\|^2/\sigma_1)$ and we enter the last phase: the convergence behavior of $\sum_{j>r}\|x_j^t\|^2$. Note with $\theta_t^U \le \cO(\sum_{j>r}\|x_j^t\|^2/\sigma_1)$ and $\|x^t_i\|^2\leq 2\sigma_r$ for $i\leq r$, we can prove \eqref{eq: At.update}. The condition \eqref{eq: At.update_Initial} can be proven since the first two phases are quite short and the updating formula of $x_i$ for $i>r$ shows $\|x_i\|^2$ cannot decrease too much. 

\subsection{Phase 1}
In this phase, we show that $\|x_i^t\|^2$ for $i\le r$ becomes large, while $\|x_i^t\|^2$ for $i > r$ still remains small. In addition, the maximum angle between different column vectors remains small.  Phase 1 terminates when $\|x_i^t\|^2$ is larger than a constant.

To be more specific,
we have the following two lemmas. Lemma \ref{lemma:initialangle} states that the initial angle $\theta_0 = \cO(\log^2(r\sqrt{\sigma_1}/\alpha)(r\kappa)^2)$ is small because the vectors in the high-dimensional space are nearly orthogonal.

\begin{lemma}\label{lemma:initialangle}
    For some constant $c_4$ and $c$, if $k\ge \frac{c^2}{16\log^4(r\sqrt{\sigma_1}/\alpha)(r\kappa)^4}$, with probability at least $1-c_4n^2k\exp(-\sqrt{k})$,  we have 
    \begin{align}
        \theta_0 \le \frac{c}{ \log^2(r\sqrt{\sigma_1}/\alpha)(r\kappa)^2}
    \end{align}
\end{lemma}
\begin{proof}
    See \S \ref{sec: proof of lemma initial angle} for proof.
\end{proof}
Lemma \ref{lemma:initial length} states that with the initialization scale $\alpha$, the norm of randomized vector $x_i^0$ is $\Theta(\alpha^2)$.
\begin{lemma}\label{lemma:initial length}
    With probability at least $1-2n\exp(-c_5k/4)$, for some constant $c$, we have 
    \begin{align*}
        \|x_i^0\|^2 \in [\alpha^2/2, 2\alpha^2].
    \end{align*}
\end{lemma}
\begin{proof}
    See \S \ref{proof of lemma initial legnth} for the proof.
\end{proof}
Now we prove the following three conditions by induction. 
\begin{lemma}
There exists a constant $C_1$, such that $T_1 \le C_1(\log(\sqrt{\sigma_1}/n\alpha)/\eta\sigma_r)$ and
then during the first $T_1$ rounds, with probability at least $1-2c_4n^2k\exp(-\sqrt{k})-2n\exp(-c_5k/4)$ for some constant $c_4$ and $c_5$, the following four statements always hold \begin{align}
    \Vert x_i^t\Vert^2 &\le 2\sigma_1 \label{phase 1 condition 1} \\
    \alpha^2/4\le \Vert x_i^t\Vert^2 &\le 2\alpha^2 \ \ \ (i>r) \label{phase 1 condition 2}  \\
    2\theta_{0}&\ge\theta_{t}\label{phase 1 condition 3}
\end{align}
Also, if $\|x_i^t\|^2 \le 3\sigma_i/4$, we have 
\begin{align}
    \|x_i^{t+1}\|^2 \ge (1+\eta\sigma_r/4)\|x_i^t\|^2.\label{phase 1 condition 4}
\end{align}
Moreover, at $T_1$ rounds, $\|x_i^{T_1}\|^2 \ge 3\sigma_i/4,$ and Phase 1 terminates.
\end{lemma}
\begin{proof}

By Lemma \ref{lemma:initialangle} and Lemma \ref{lemma:initial length}, with probability at least $1-2c_4n^2k\exp(-\sqrt{k})-2n\exp(-c_5k/4)$, we have
 $\|x_i^0\|^2 \in [\alpha^2/2, 2\alpha^2]$ for $i \in [n]$, and $\theta_0 \le \frac{c}{\log^2(r\sqrt{\sigma_1}/\alpha)(r\kappa)^2}$. Then
assume that the three conditions hold for rounds before $t$, then at the $t+1$ round, we proof the four statements above one by one. 
\paragraph{Proof of Eq.\eqref{phase 1 condition 2}}
 For $i> r$, 
 we have 
 \begin{align*}
     (x_i^{t+1})^\top  = (x_i^t)^\top  - \eta \sum_{j=1}^n (x_i^t)^\top  x_j^t (x_j^t)^\top 
 \end{align*}
 Then, the updating rule of $\|x_i^t\|^2$ can be written as 
\begin{equation}\label{eq: updata_xi^2}
    \Vert (x_i^{t+1}) \Vert_2^2  = \Vert x_i^t\Vert^2 - 2\eta \sum_{j=1}^n ((x_i^t)^\top x_j^t)^2  + \eta^2 (\sum_{j,k=1}^n (x_i^t)^\top  x_j^t (x_j^t)^\top x_k^t (x_k^t)^\top x_i^t)\le \Vert x_i^t\Vert^2.
\end{equation}
The last inequality in \eqref{eq: updata_xi^2} is because \begin{align}(x_i^t)^\top x_j^t (x_j^t)^\top x_k^t (x_k^t)^\top  (x_i^t) &\le (x_j^t)^\top  x_k^t(((x_i^t)^\top x_j^t)^2 + ((x_k^t)^\top x_i^t)^2)/2 \\&\le \sigma_1 ((x_i^t)^\top x_j^t)^2 + ((x_k^t)^\top x_i^t)^2),\end{align}
and then
\begin{align}
    \eta^2 \sum_{j,k=1}^n (x_i^t)^\top x_j^t (x_j^t)^\top x_k^t (x_k^t)^\top  (x_i^t) &\le \eta^2 \sum_{j,k=1}^n \sigma_1((x_i^t)^\top x_j^t)^2 + ((x_k^t)^\top x_i^t)^2)\nonumber\\
    & = \eta^2 \cdot n\sigma_1 \sum_{j=1}^n ((x_i^t)^\top x_j^t)^2\nonumber\\&\le \eta \sum_{j=1}^n ((x_i^t)^\top x_j^t)^2. \label{ineq: eta^2 term bound}
\end{align}
where the last inequality holds because $\eta \le 1/n\sigma_1.$
Thus, the $\ell_2$-norm of $x_i^\top $ does not increase, and the right side of Eq.\eqref{phase 1 condition 2} holds.

Also, we have 
\begin{align}
    \|x_i^{t+1}\|^2 & \geq  \Vert x_i^t\Vert^2 - 2\eta \sum_{j=1}^n ((x_i^t)^\top x_j^t)^2 +\eta^2 \left \|\sum_{j=1}^n (x_i^t)^\top  x_j^t (x_j^t)^\top
    \right \|^2 \nonumber\\
     &\ge \Vert x_i^t\Vert^2 - \|x_i^t\|^2 \cdot 2\eta \theta_t \cdot  \sum_{j\neq i}^n \|x_j^t\|^2-2\eta \|x_i\|^4
     \end{align}\label{explain lower bound 2}
    Equation \eqref{explain lower bound 2} is because $\frac{((x_i^t)^\top x_j^t)^2}{\|x_i^t\|^2 \|x_j^t\|^2} = \theta_{ij,t} \le \theta_t.$
    Now by \eqref{phase 1 condition 1} and \eqref{phase 1 condition 2}, we can get 
    \begin{align*}
        \sum_{j\neq i}^n \|x_j^t\|^2 \le r\cdot 2\sigma_ 1+ (n-r)\cdot 2\alpha^2 \le 2\sigma_1 + 2n\alpha^2
    \end{align*}
    Hence, we can further derive 
     \begin{align*}
    \|x_i^{t+1}\|^2&\ge \|x_i^t\|^2\cdot \left(1-2\eta\theta_t(2r\sigma_1 + 2n\alpha^2)-2\eta\cdot 2\alpha^2\right)\\
    &\ge \|x_i^t\|^2 \cdot 
    \left(1-\eta( 8\theta_t \sigma_1+4\alpha^2) \right),
\end{align*}
 where the last inequality is because $\alpha \le \sqrt{r\sigma_1}/\sqrt{n}$.
Thus, by $(1-a)(1-b)\ge (1-a-b)$ for $a,b>0$, we can get \begin{align}
    \|x_i^{T_1}\|^2  & \ge \|x_i^{0}\|^2\cdot (1-\eta ( 8\theta_t \sigma_1+4\alpha^2) )^{T_1}\nonumber\\
    &\ge \frac{\alpha^2}{2} \cdot (1-T_1\eta ( 8\cdot(2\theta_0) \sigma_1+4\alpha^2) )
\label{need explain lower bound 3}\\ &\ge \frac{\alpha^2}{4}.
\end{align}
Equation \eqref{need explain lower bound 3} holds by induction hypothesis \eqref{phase 1 condition 3}, and 
the last inequality is because of our choice on $T_1$, $\alpha$, and  $\theta_0 \le O(\frac{1}{r\kappa\log(\sqrt{\sigma_1}/\alpha)})$ from the induction hypothesis. 
%
%
Hence, we complete the proof of Eq.\eqref{phase 1 condition 2}.

 \paragraph{Proof of Eq.\eqref{phase 1 condition 4}}
For $i\le r$, if $\|x_i^t\|^2 \le 3\sigma_i/4$, by the updating rule, 
\begin{align}
    \Vert x_i^{t+1}\Vert_2^2 &\ge (1-\eta(\|x_i^t\|^2-\sigma_i))^2 \|x_i^t\|^2 - 2\eta \sum_{j\neq i}^n ((x_i^t)^\top x_j^t)^2 + \eta^2 (\|x_i^t\|^2-\sigma_i) \sum_{j\neq i}^n ((x_i^t)^\top x_j^t)^2 \label{last four inequality}\\&\ge (1-\eta(\|x_i^t\|^2-\sigma_i))^2 \|x_i^t\|^2 - 2\eta \sum_{j\neq i}^n ((x_i^t)^\top x_j^t)^2 - \eta^2 |\|x_i^t\|^2-\sigma_i|\cdot  \sum_{j\neq i}^n \|x_i^t\|^2\|x_j^t\|^2\nonumber\\&\ge (1-\eta(\|x_i^t\|^2-\sigma_i))^2 \|x_i^t\|^2 - 2\eta \sum_{j\neq i}^n ((x_i^t)^\top x_j^t)^2 - 4\eta^2 (n\sigma_1^2)\|x_i^t\|^2\nonumber.
    \end{align}
    THe last inequality uses the fact that $|\|x_i^t\|^2 - \sigma_i| \le 2\sigma_1$ and $\|x_j^t\|^2 \le 2\sigma_1$. Then, by $((x_i^t)^\top x_j^t)^2 \le \|x_i^t\|^2 \|x_j^t\|^2 \cdot \theta,$ we can further get
    \begin{align}\|x_i^{t+1}\|^2
    &\ge  \left(1-2\eta(\|x_i^t\|^2-\sigma_i)-2\eta \sum_{j\neq i}^n \Vert x_j^t\Vert^2\theta - 2\eta^2 (n\sigma_1^2)\right)\|x_i^t\|^2\nonumber\\
    &\ge (1+\eta\sigma_i/2-2\eta^2 (n\sigma_1^2)-\eta\sigma_r/16 )\|x_i^t\|^2\label{last three inequality}\\&
    \ge (1+\sigma_i(\eta/2-\eta/16-\eta/16))\|x_i^t\|^2\label{last two inequality}\\
    &\ge (1+\eta\sigma_i/4)\|x_i^t\|^2.\nonumber
\end{align}
The inequality \eqref{last three inequality} uses the fact $\theta\le 2\theta_0 \le \frac{1}{128\kappa r}$ and $\sum_{j\neq i}^n \Vert x_j\Vert^2 \le 2\sigma_1r + 2n\alpha^2\le 4\sigma_1 r\le \frac{\sigma_r}{32\theta}$. The inequality \eqref{last two inequality} uses the fact that $\eta \le \frac{1}{32n\sigma_1^2}$.

\paragraph{Proof of Eq.\eqref{phase 1 condition 1}}
If $\|x_i^t\|^2 \ge 3\sigma_i/4$, by the updating rule, we can get
\begin{align}
    |\Vert x_i^{t+1}\Vert_2^2-\sigma_i| &\le  \left(1-2\eta \|x_i^t\|^2+\eta^2(\|x_i^t\|^2-\sigma_i)\|x_i^t\|^2+\eta^2 \sum_{j\neq i}^n ((x_i^t)^\top x_j^t)^2\right) |\|x_i^t\|^2-\sigma_i|\nonumber\\&\qquad  +2\eta \sum_{j\neq i}^n ((x_i^t)^\top x_j^t)^2 +\eta^2\left(\sum_{j,k\neq i}^n ((x_i^t)^\top x_j^t(x_j^t)^\top x_k^t (x_k^t)^\top  x_i^t)\right)\nonumber  \\&\le (1-\eta\sigma_i) |\|x_i^t\|^2-\sigma_i| + 3\underbrace{\eta \sum_{j\neq i}^n ((x_i^t)^\top x_j^t)^2}_{\text{(a)}}\label{ineq: when xi norm large, bound}
\end{align}
The last inequality holds by Eq.\eqref{ineq: eta^2 term bound} and 
\begin{align}
    &2\eta \|x_i^t\|^2-\eta^2(\|x_i^t\|^2-\sigma_i)\|x_i^t\|^2-2\eta^2 \sum_{j\neq i}^n ((x_i^t)^\top x_j^t)^2\\
    &\ge \frac{3\eta}{2}\sigma_i-\eta^2(2\sigma_1)\cdot 2\sigma_1 - 2\eta^2 n \sigma_1^2 \label{eq:explain 4 lower bound}\\
    &\ge \eta\sigma_i, \label{eq:explain 5 lower bound}
\end{align}
where \eqref{eq:explain 4 lower bound} holds by $\|x_i^t\|^2 \ge \frac{3\sigma_i}{4}$, $\|x_i^t\|^2 \le 2\sigma_1$ for all $i \in [n].$ The last inequality \eqref{eq:explain 5 lower bound} holds by $\eta \le C (\frac{1}{n\sigma_1\kappa})$ for small constant $C$. The first term of \eqref{ineq: when xi norm large, bound} represents the main converge part, and (a) represents the perturbation term. Now for the perturbation term (a),   since $\alpha \le \frac{1}{4\kappa n^2}$ and $\theta \le 2\theta_0 \le \frac{1}{20r\kappa^2}= \frac{\sigma_i^2}{20r\sigma_1^2}$,  we can get\begin{align}
    \text{(a)} &= \sum_{j\neq i, j\le r}((x_i^t)^\top x_j^t)^2 +\sum_{j\neq i, j > r} ((x_i^t)^\top x_j^t)^2\\&\le (r\sigma_1 
 + 2n\alpha^2)\theta_t \cdot 2\sigma_1\label{eq: explain 6 lower bound}\\&\le 2r\sigma_1\cdot \theta_t \cdot 2\sigma_1\label{eq: explain 7 lower bound}\\&= 4r\sigma_1^2\cdot \theta_t\nonumber\\&\le \sigma_i^2/5,\label{eq: exaplin 8 lower bound}\end{align}
 where \eqref{eq: explain 6 lower bound} holds by \eqref{phase 1 condition 1} and \eqref{phase 1 condition 2}.  \eqref{eq: explain 7 lower bound} holds by $\alpha = \cO(\sqrt{r\sigma_1/n})$, and the last inequality \eqref{eq: exaplin 8 lower bound} holds by $\theta$ is small, i.e. $\theta_t \le 2\theta_0 = \cO(1/r\kappa^2)$.
 Now
 it is easy to get that $(x_i^{t+1})^\top x_i^{t+1} \le 2\sigma_i$ by 
 \begin{align}
     |\|x_i^{t+1}\|^2-\sigma_i| \le (1-\eta\sigma_i)(\|x_i^t\|^2-\sigma_i) + \frac{3\eta\sigma_i^2}{5}\le (1-\eta\sigma_i)\sigma_i +\frac{3\eta\sigma_i^2}{5} \le \sigma_i. 
 \end{align}
 Hence, we complete the proof of Eq.\eqref{phase 1 condition 1}.
 \paragraph{Proof of Eq.\eqref{phase 1 condition 3}}

Now we consider the change of $\theta$.
For $i\neq j$, denote
\begin{align*}
    \theta_{ij,t}=\frac{((x_i^t)^\top x_j^t)^2}{\Vert x_i \Vert^2 \Vert x_j\Vert^2}
\end{align*}
Now we first calculate the $(x_i^{t+1})^\top x_j^{t+1}$ by the updating rule:
\begin{align*}
    &\quad (x_i^{t+1})^\top x_j^{t+1}\\&= \underbrace{\left(1-\eta(\|x_i^t\|^2-\sigma_i)\right)\left(1-\eta(\|x_j^t\|^2-\sigma_j)\right)(x_i^t)^\top x_j^t}_{\text{A}} \underbrace{-\eta\|x_j^t\|^2(1-\eta(\|x_j^t\|^2-\sigma_j))(x_i^t)^\top x_j^t}_{\text{B}} \\&\qquad \underbrace{-\eta\|x_i^t\|^2(1-\eta(\|x_i^t\|^2-\sigma_j))(x_i^t)^\top x_j^t}_{\text{C}}+\underbrace{\eta^2\sum_{k,l\neq i,j}(x_i^t)^\top x_k^t(x_k^t)^\top x_l^t (x_l^t)^\top x_j^t }_{\text{D}}
    \\&\qquad \qquad \underbrace{-\eta (2-\eta(\|x_i^t\|^2-\sigma_i)-\eta(\|x_j^t\|^2-\sigma_j))\sum_{k\neq i,j}^n (x_i^t)^\top x_k^t (x_k^t)^\top x_j}_{\text{E}}\\&\qquad \qquad \qquad \underbrace{+ \eta^2 \sum_{k\neq i,j}x_i^\top x_j^t(x_j^t)^\top x_k^t(x_k^t)^\top x_j^t+\eta^2\sum_{k\neq i,j}(x_i^t)^\top x_k^t(x_k^t)^\top x_i^t(x_i^t)^\top x_j^t}_{\text{F}}.
    \end{align*}
    Now we bound A, B, C, D, E and F respectively. First, by $\|x_i^t\|^2 \le 2\sigma_1$ for any $i \in [m]$, we have 
\begin{align}
    \text{A} &\le \left(1-\eta (\|x_i^t\|^2-\sigma_i)-\eta(\|x_j^t\|^2-\sigma_j) + \eta^2 (\|x_i^t\|^2-\sigma_i)\left(\|x_j^t\|^2-\sigma_j)\right)\right)(x_i^t)^\top x_j^t\nonumber\\&\le \left(1-\eta\left( \|x_i^t\|^2 + \|x_j^t\|^2 - \sigma_i-\sigma_j \right) + \eta^2 \cdot 4\sigma_1^2\right)(x_i^t)^\top x_j^t,\label{Aupper bound}
\end{align}
Now we bound term B. We have
\begin{align}
    \text{B} + \text{C} &= \left(-\eta (\|x_i^t\|^2 + \|x_j^t\|^2) + \eta^2 \left((\|x_j^t\|^2-\sigma_j)\|x_j^t\|^2  +  (\|x_i^t\|^2-\sigma_i)\|x_i^t\|^2 \right)\right) (x_i^t)^\top x_j^t\nonumber\\
    &\le \left(-\eta (\|x_i^t\|^2 + \|x_j^t\|^2) + \eta^2 \cdot (8\sigma_1^2)\right) (x_i^t)^\top x_j^t.\label{B+C bound}
\end{align}
Then, for D, by $\theta_t \le 1$, we have 
\begin{align}
    \text{D} &= \eta^2 \left(\sum_{k,l\neq i,j}\|x_k^t\|^2 \|x_l^t\|^2 \cdot \sqrt{\theta_{ik,t}\theta_{kl,t}\theta_{lj,t}/\theta_{ij,t}} \right) (x_i^t)^\top x_j^t\nonumber\\
    &\le \left(\eta^2 \cdot n^2 \cdot 4\sigma_1^2 \cdot \theta_t/\sqrt{\theta_{ij,t}} \right) (x_i^t)^\top x_j^t.\label{D upper bound}
\end{align}
For E, since  we have 
\begin{align}
    \text{E} &\le  2\eta \sum_{k\neq i,j} |(x_i^t)^\top x_k^t (x_k^t)^\top x_j^t | + 4\sigma_1\eta^2 \sum_{k\neq i,j} |(x_i^t)^\top x_k^t (x_k^t)^\top x_j^t |\nonumber \\
    &\le \left(2\eta \sum_{k\neq i,j}\|x_k^t\|^2 \cdot \sqrt{\theta_{ik,t}\theta_{kj,t}/\theta_{ij,t}}+4\sigma_1\eta^2 \sum_{k\neq i,j}\|x_k^t\|^2 \cdot \sqrt{\theta_{ik,t}\theta_{kj,t}/\theta_{ij,t}}\right) (x_i^t)^\top x_j^t\nonumber \\
    &\le \left(2\eta \sum_{k\neq i,j}\|x_k^t\|^2 \cdot \sqrt{\theta_{ik,t}\theta_{kj,t}/\theta_{ij,t}}+4n\sigma_1\eta^2 \cdot (2\sigma_1) \cdot \theta_t/\sqrt{\theta_{ij,t}}\right) (x_i^t)^\top x_j^t.\label{E upper bound}
\end{align}
Lastly, for F, since $(x_j^t)^\top x_k^t (x_k^t)^\top x_j^t \le \|x_j^t\|^2 \|x_k^t\|^2 \le 4\sigma_1^2,$ we have 
\begin{align}
    \text{F} \le \eta^2 8n\sigma_1^2 (x_i^t)^\top x_j^t.\label{F upper bound}
\end{align}
Now combining \eqref{Aupper bound}, \eqref{B+C bound}, \eqref{D upper bound}, \eqref{E upper bound} and \eqref{F upper bound}, we can get 
\begin{small}
\begin{align}
    &(x_i^{t+1})^\top x_j^{t+1} \\&\quad \le   \left(1-\eta(2\Vert x_i\Vert^2 + 2\Vert x_j\Vert^2 - \sigma_i -\sigma_j)  + 2\eta \sum_{k\neq i,j} \Vert x_k\Vert^2\cdot \sqrt{\theta_{ik,t}\theta_{kj,t}/\theta_{ij,t}} + 30n^2\sigma_1^2\eta^2\theta_t/\sqrt{\theta_{ij,t}})\right) 
 (x_i^t)^\top  x_j^t.\label{update xixj}
\end{align}
\end{small}
    


On the other hand, consider the change of $\|x_i^t\|^2$. By Eq.\eqref{last four inequality},
\begin{align*}
    \|x_i^{t+1}\|^2 &\ge (1-\eta(\|x_i^t\|^2-\sigma_i))^2 \|x_i^t\|^2 - 2\eta \sum_{j\neq i}^n ((x_i^t)^\top x_j^t)^2 + \eta^2 (\|x_i^t\|^2-\sigma_i) \sum_{j\neq i}^n ((x_i^t)^\top x_j^t)^2\\&\ge (1-2\eta(\|x_i^t\|-\sigma_i)-2\eta \sum_{j\neq i}^n \Vert x_j^t\Vert^2 \theta_{ij,t} - 4\eta^2n\theta_t\sigma_1^2)\|x_i^t\|^2\\
    &\ge (1-2\eta(\|x_i^t\|-\sigma_i)-2\eta \sum_{k=1}^n \Vert x_j^t\Vert^2 \theta_{ij,t} - 4\eta^2n\theta_t\sigma_1^2)\|x_i^t\|^2
\end{align*}
Hence, the norm of $x_i^{t+1}$ and $x_j^{t+1}$ can be lower bounded by 
\begin{align}
    &\|x_i^{t+1}\|^2 \|x_j^{t+1}\|^2 \nonumber\\&\ge \Big(1-2\eta(\|x_i^t\|^2-\sigma_i)-2\eta(\|x_j^t\|^2-\sigma_j)-2\eta\sum_{k\neq i,j} \Vert x_k\Vert^2 (\theta_{ik,t}+\theta_{jk,t}) - 2\eta (\|x_j\|^2 + \|x_i\|^2)\theta_{ij,t} \nonumber\\&\qquad \qquad -4\eta^2 \theta_tn^2\sigma_1^2+\sum_{l=i,j}4\eta^2(\|x_l^t\|^2-\sigma_l)\sum_{k=1}^n \|x_k^t\|^2\theta_{ik,t}  +\sum_{l=i,j} 2\eta (\|x_l^t\|^2-\sigma_l)\eta^2n^2\theta_t\sigma_1^2\Big)\|x_i^t\|^2 \|x_j^t\|^2 \nonumber\\&\ge \Big(1-2\eta(\|x_i^t\|^2-\sigma_i)-2\eta(\|x_j^t\|^2-\sigma_j)-2\eta\sum_{k\neq i,j} \Vert x_k\Vert^2 (\theta_{ik,t}+\theta_{jk,t}) - 2\eta (\|x_j\|^2 + \|x_i\|^2)\theta_{ij,t} \nonumber\\&\qquad \qquad -4\eta^2 \theta_tn^2\sigma_1^2 -2\cdot 4\eta^2\cdot(2\sigma_1) n\cdot (2\sigma_1)\theta_t-2\cdot 4\eta \sigma_1 \cdot \eta^2n^2\theta_t\sigma_1^2\Big)\|x_i^t\|^2 \|x_j^t\|^2\label{ineq: last two key inequality}\\&\ge \Big(1-2\eta(\|x_i^t\|^2-\sigma_i)-2\eta(\|x_j^t\|^2-\sigma_j)-2\eta\sum_{k\neq i,j} \Vert x_k\Vert^2 (\theta_{ik,t}+\theta_{jk,t}) - 2\eta (\|x_j\|^2 + \|x_i\|^2)\theta_{ij,t}\nonumber 
    \\& \qquad \qquad -6\eta^2 \theta_tn^2\sigma_1^2\Big)\|x_i^t\|^2 \|x_j^t\|^2,\label{last inequality: explain}
\end{align}
where \eqref{last inequality: explain} holds by $n>8k\ge 8$ and $2\eta (\|x_i^t\|^2-\sigma_i) \le 4\eta \sigma_1 \le 1.$ Then, by \eqref{update xixj} and \eqref{last inequality: explain}, we have

\begin{align}
    \theta_{ij,t+1}& = \theta_{ij,t}\cdot \frac{(x_i^{t+1})^\top x_j^{t+1}}{(x_i^t)^\top x_j^t} \cdot \frac{\|x_i^{t+1}\|^2 \|x_j^{t+1}\|^2}{\|x_i^t\|^2 \|x_j^t\|^2}\nonumber\\
    &\le  \theta_{ij,t}\cdot \left(\frac{1-A+B}{1-A-C}\right)\label{theta update}
\end{align}
where \begin{gather}A = 2\eta(\|x_i^t\|^2-\sigma_i+\|x_j^t\|^2-\sigma_i))\le 4\eta \sigma_1\\ B = 2\eta\Vert x_k\Vert^2 \cdot \sqrt{\theta_{ik,t}\theta_{kj,t}/\theta_{ij,t}}+30n^2\sigma_1^2\eta^2\theta_t/\sqrt{\theta_{ij,t}} \end{gather}
and 
\begin{align}C &= 2\eta\sum_{k\neq i,j} \Vert x_k\Vert^2 (\theta_{ik,t}+\theta_{jk,t}) + 2\eta (\|x_j\|^2 + \|x_i\|^2)\theta_{ij,t} + 6\eta^2n^2\theta_t\sigma_1^2\\&\le 
\left(8\eta \sigma_1 + 2\eta (2n\alpha^2 + 2r\sigma_1) +6\eta^2n^2 \sigma_1^2\right)\theta_t
,\end{align}
where the last inequality uses the fact that 
$$\sum_{k\neq i,j} \|x_k^t\|^2 \le \sum_{k\le r}\|x_k^t\|^2 + \sum_{k>r} \|x_k^t\|^2 \le 2r\sigma_1 + 2n\alpha^2 . $$

Hence, we choose $\eta \le \frac{1}{1000n\sigma_1}$ to be sufficiently small so that $\max\{A,C\} \le 1/100$, 
then by $\frac{1-A+B}{1-A-C} \le 1 + 2B + 2C$ for $\max\{A,C\}\le 1/100$,
\begin{align*}
    &\quad \theta_{ij,t}\cdot \left(\frac{1-A+B}{1-A-C}\right)\\
    &\le \theta_{ij,t}(1+2B+2C)\\
    &\le \theta_{ij,t} + 4\eta\sum_{k\neq i,j}\Vert x_k\Vert^2 \cdot \sqrt{\theta_{ik,t}\theta_{kj,t}\theta_{ij,t}}+60n^2\sigma_1^2 \eta^2\theta_t \sqrt{\theta_{ij,t}} \\&\qquad \qquad+\theta_t^2\Big(8\eta \sigma_1 + 2\eta (2n\alpha^2 + 2r\sigma_1) +6\eta^2n^2 \sigma_1^2\Big)\\
     &\le \theta_{ij,t} + 4\eta(2r\sigma_1 + 2n\alpha^2) \theta_t^{3/2} +60n^2\sigma_1^2 \eta^2\theta_t^{3/2} \\&\qquad \qquad+\theta_t^2\Big(8\eta \sigma_1 + 2\eta (2n\alpha^2 + 2r\sigma_1) +6\eta^2n^2 \sigma_1^2\Big)\\
    &\le \theta_{ij,t} + 6\eta (2r\sigma_1 + 2n\alpha^2)\theta_t^{3/2} + 60n^2\sigma_1^2\eta^2\theta_t^{3/2} + 8\eta \sigma_1\theta_t^2 +6n^2\eta^2\sigma_1^2\theta_t^2)\\
    &\le \theta_{ij,t} + 98\eta \cdot (r\sigma_1\theta_t^{3/2})
\end{align*}
The last inequality holds by $\alpha\le \sqrt{\sigma_1}/\sqrt{n}$, 
 and $n^2\sigma_1\eta^2\le \eta$ because
$    \eta \le \frac{1}{n^2\sigma_1}.
$

Hence, \begin{align}
    \theta_{t+1} \le \theta_{t} + 98\eta (r\sigma_1)\theta_t^{3/2}
\end{align}

The Phase 1 terminates when $\|x_i^{T_1}\|^2 \ge \frac{3\sigma_i}{4}$. Since $\|x_i^0\|^2 \ge \alpha^2/2$ and 
\begin{align}
    \|x_i^{t+1}\|^2 \ge (1+\eta\sigma_i/4)\|x_i^t\|^2,
\end{align} there is a constant $C_3$ such that  $T_1 \le C_1(\log (\sqrt{\sigma_1}/\alpha)/\eta\sigma_i)$.
Hence, before round $T_1$,  $$\theta_{T_1}\le \theta_0 + 98\eta T_1\cdot r\sigma_1\cdot (2\theta_0)^{3/2}\le \theta_0 + 98 C_1r\kappa(2\theta_0)^{3/2}\log(\sqrt{\sigma_1}/\alpha)\le  2\theta_0.$$ This is because 
\begin{align*}
    \theta_{0} =\cO((\log^2 (r\sqrt{\sigma_1}/\alpha) (r\kappa))^2)
\end{align*}
by Lemma \ref{lemma:initialangle} and choosing $k\ge c_2((r\kappa)^2 \log (r\sqrt{\sigma_1/\alpha}))^4$ for large enough $c_2$
 \end{proof}
\subsection{Phase 2}
Denote $\theta_t^U = \max_{\min\{i,j\} \le r}\theta_{ij,t}$. 
In this phase, we prove that $\theta_t^U$ is linear convergence, and the convergence rate of the loss is at least $\Omega(1/T^2)$. To be more specific,  we will show that 
\begin{gather}
    \theta_{t+1}^U\le \theta_t^U \cdot (1-\eta\cdot \sigma_r/4) \le \theta_t^U\label{ineq:sym phase 2 condition 0}\\
    \frac{\theta_{t+1}^U}{\sum_{i>r}\|x_i^{t+1}\|^2} \le \frac{\theta_{t}^U}{\sum_{i>r}\|x_i^{t}\|^2}\cdot \left(1-\frac{\eta\sigma_r}{8}\right)
    \label{ineq:sym phase 2 condition 1}\\|\|x_i^t\|^2-\sigma_i|\le \frac{1}{4}\sigma_i \ \ (i\le r)\label{ineq:sym phase 2 condition 2}\\
    \|x_i^t\|^2 \le 2\alpha^2 \ \ (i>r)\label{ineq:sym phase 2 condition 3}
\end{gather}

First, the condition \eqref{ineq:sym phase 2 condition 2} and \eqref{ineq:sym phase 2 condition 3} hold at round $T_1.$
 Then, if it holds before round $t$, consider round $t+1$, similar to Phase 1, condition \eqref{ineq:sym phase 2 condition 3} also holds. Now we prove Eq.\eqref{ineq:sym phase 2 condition 0}, \eqref{ineq:sym phase 2 condition 1} and \eqref{ineq:sym phase 2 condition 2} one by one. 
 \paragraph{Proof of Eq.\eqref{ineq:sym phase 2 condition 2}}
 For $i\le r$, if $\|x_i^t\|^2 \ge 3\sigma_i/4$, by Eq.\eqref{ineq: when xi norm large, bound}
\begin{align}
    |\Vert x_i^{t+1}\Vert_2^2-\sigma_i| \le (1-\eta\sigma_i) |\|x_i^t\|^2-\sigma_i| + 3\eta \sum_{j\neq i}^n ((x_i^t)^\top x_j^t)^2\label{connect lower bound}
\end{align}
Hence, by \eqref{ineq:sym phase 2 condition 2} and \eqref{ineq:sym phase 2 condition 3}, we can get
\begin{align}
    \sum_{j\neq i}^n ((x_i^t)^\top x_j^t)^2 &\le \sum_{j\neq i, j\le r}((x_i^t)^\top x_j^t)^2 +\sum_{j\neq i, j > r} ((x_i^t)^\top x_j^t)^2\nonumber\\&\le (r\sigma_1 
 + 4n\sigma_1\alpha^2)\theta^U_t \nonumber\\&\le 2r\sigma_1 
 \theta^U_t\label{explain 9 lower bound}\\&\le 2r\sigma_1 \theta_{T_1}^U\label{explain 10 lower bound}\\&\le 2r\sigma_1 \cdot 2\theta_0\le \sigma_i/20.\label{expllain 11 lower bound}\end{align}  The inequality \eqref{explain 9 lower bound} is because $\alpha \le \frac{1}{4n\sigma_1}$, the inequality \eqref{explain 10 lower bound} holds by induction hypothesis \eqref{ineq:sym phase 2 condition 0}, and the last inequality \eqref{expllain 11 lower bound} is because of \eqref{phase 1 condition 3} and 
 $    \theta_0 \le \frac{1}{80r\kappa}.$
 
  Hence, if $|\|x_i^t\|^2-\sigma_i|\le \sigma_i/4$, by combining \eqref{connect lower bound} and \eqref{expllain 11 lower bound}, we have 
 \begin{align*}
     |\|x_i^{t+1}\|^2-\sigma_i|\le (1-\eta\sigma_i)|\|x_i^t\|-\sigma_i| + 3\eta\sigma_i/20\le \sigma_i/4.
 \end{align*} Now it is easy to get that $|\|x_i^t\|^2- \sigma_i|\le 0.25\sigma_i$ for $t\ge T_1$ by induction because of  $|\|x_i^{T_1}\|^2-\sigma_i|\le 0.25\sigma_i$. Thus, we complete the proof of Eq.\eqref{ineq:sym phase 2 condition 2}.

\paragraph{Proof of Eq.\eqref{ineq:sym phase 2 condition 0}}

First, 
we consider $i\le r, j\neq i\in [n]$ and $\theta_{ij,t}>\theta_t^U/2$, since \eqref{phase 1 condition 1} and \eqref{phase 1 condition 2} still holds with \eqref{ineq:sym phase 2 condition 2} and \eqref{ineq:sym phase 2 condition 3}, similarly, we can still have equation \eqref{theta update}, i.e.
\begin{align*}
    \theta_{ij,t+1}
    =\theta_{ij,t}\cdot \left(\frac{1-A-B}{1-A-C}\right).
\end{align*}
where 
\begin{align}
    A&= 2\eta(\|x_i^t\|^2-\sigma_i)+2\eta(\|x_j^t\|^2-\sigma_j)\ge -2\eta (2\cdot (\sigma_i/4))\ge -1/100.\nonumber\\
    B&= 2\eta (\|x_i^t\|^2 + \|x_j^t\|^2) - 2\eta\sum_{k\neq i,j}\Vert x_k\Vert^2 \cdot \sqrt{\theta_{ik,t}\theta_{kj,t}/\theta_{ij,t}}-30n^2\eta^2\sigma_1^2 \sqrt{\theta_t^U}/\sqrt{\theta_{ij,t}}\nonumber\\&\ge2\eta (\|x_i^t\|^2 + \|x_j^t\|^2)- 4\eta\sum_{k\le r}\Vert x_k\Vert^2\sqrt{\theta^U} -4n\eta\alpha^2-40n^2\eta^2\sigma_1^2\label{ineq: B}\\&\ge 2\eta\cdot \frac{3\sigma_i}{4} - 8\eta r\sigma_1\sqrt{2\theta_{T_0}}-4n\eta\alpha^2-40n^2\eta^2\sigma_1^2 \label{explain 13 lower bound}\\&\ge \eta \cdot \sigma_r\label{explain 14 lower bound}
\end{align}
The inequality Eq.\eqref{ineq: B} holds by $\theta_{ij,t} >\theta_t^U/2$,  the inequality \eqref{explain 13 lower bound} holds by \eqref{ineq:sym phase 2 condition 0}, and \eqref{explain 14 lower bound} holds by
\begin{align}
    \theta_{T_0}=\cO\left(\frac{1}{r^2\kappa^2}\right), \quad 
    \alpha = \cO(\sqrt{\sigma_r/n}), \quad 
    \eta = \cO(1/n^2\kappa\sigma_1).\label{range of parameter lower bound}
\end{align}
The term $C$ is defined and can be bounded by 
\begin{align}
    C &= 2\eta\sum_{k\neq i,j} \Vert x_k\Vert^2 (\theta_{ik,t} + \theta_{jk,t}) + 2\eta (\|x_i\|^2 + \|x_j\|^2)\theta_{ij,t}+6\eta^2\theta_tn^2\sigma_1^2\nonumber\\&\le 4\eta \sum_{k\le r} \|x_k\|^2 \theta_t^U + 4\eta n\alpha^2 \theta_t + 6\eta^2\theta_t n^2\sigma_1^2 \nonumber\\&\le 8r\eta\sigma_1 \theta_t^U +  4\eta n\alpha^2 + 6\eta^2n^2\sigma_1^2 \nonumber\\&\le 8r\eta\sigma_1 \theta_{T_0} +  4\eta n\alpha^2 + 6\eta^2n^2\sigma_1^2\label{explain 15 lb}\\&\le \eta\cdot \sigma_r/2.\label{explain 16 lb}
\end{align}
The inequality \eqref{explain 15 lb} holds by \eqref{ineq:sym phase 2 condition 0}, and the inequality \eqref{explain 16 lb} holds by \eqref{range of parameter lower bound}.

Then, for $i\le r, j\neq i\in [n]$ and $\theta_{ij,t}>\theta_t^U/2$, we can get 
\begin{align}
    \theta_{ij,t+1} &\le \theta_{ij,t} \cdot \left(\frac{1-A-B}{1-A-C}\right)\nonumber\\
    &\le \theta_{ij,t}\cdot \left(\frac{2-\eta\cdot \sigma_r}{2-\eta\cdot \sigma_r/2}\right)\nonumber\\&\le \theta_{ij,t}\cdot \left(\frac{1-\eta\cdot \sigma_r/2}{1-\eta\cdot \sigma_r/4}\right)\le \theta_{ij,t}\cdot (1-\eta\cdot \sigma_r/4)\label{ineq: case 1, theta ij bound}
\end{align}
For $i\le r, j \in [n]$ and $\theta_{ij,t}\le \theta_{t}^U/2$, we have 
\begin{align}
    B&\ge -2\eta \sum_{k\le r}\|x_k\|^2 \theta_t^U/\sqrt{\theta_{ij,t}} - 2\eta \sum_{k>r}\|x_k\|^2 \sqrt{\theta_t^U}/\sqrt{\theta_{ij,t}}-30n^2\eta^2\sigma_1^2\sqrt{\theta_t^U}/\sqrt{\theta_{ij,t}}\\
    &\ge -4\eta r\sigma_1 \theta_t^U/\sqrt{\theta_{ij,t}}-(4n\eta \alpha^2+ 30n^2\eta^2\sigma_1^2)\sqrt{\theta_t^U}/\sqrt{\theta_{ij,t}}
\end{align}
\begin{align}
    \theta_{ij,t+1}&\le \theta_{ij,t}\cdot \left(\frac{1-A-B}{1-A-C}\right)\nonumber\\&\le \theta_{ij,t}\cdot (1-2B+2C)\nonumber\\
    &\le \theta_{ij,t} + 8\eta r\sigma_1\theta_t^U\sqrt{\theta_{ij,t}} + (4n\eta \alpha^2+ 30n^2\eta^2\sigma_1^2)\sqrt{\theta_t^U\theta_{ij,t}} + 2C\theta_{ij,t}\nonumber\\
    &\le \frac{\theta_t^U}{2} + 8\eta r\sigma_1\theta_t^U + (4n\eta \alpha^2+ 30n^2\eta^2\sigma_1^2)\theta_t^U + \eta \sigma_r\theta_{t}^U\nonumber\\&\le \frac{3\theta_t^U}{4}.\label{ineq: case 2, theta ij bound}
\end{align}
The last inequality is because $8\eta r\sigma_1 + 4n\eta \alpha^2+ 30n^2\eta^2\sigma_1^2  + \eta\sigma_r \le \frac{1}{4}$ by $\eta \le \cO(1/n\sigma_1)$ and $\eta \le \cO(1/n\alpha^2)$.
Hence, by Eq.\eqref{ineq: case 1, theta ij bound} and \eqref{ineq: case 2, theta ij bound} and the fact that $\eta \sigma_r/4 \le 1/4,$
\begin{align}
\theta_{t+1}^U \le \theta_{t}^U \cdot \max\Big\{\frac{3}{4},1-\eta\cdot \sigma_r/4\Big\} =(1-\eta\cdot \sigma_r/4)\theta_t^U. 
\end{align}
Thus, we complete the proof of Eq.\eqref{ineq:sym phase 2 condition 0}
\paragraph{Proof of Eq.\eqref{ineq:sym phase 2 condition 1}}

Also, for $i> r$, denote $\theta_{ii,t}=1$, then
\begin{align}
    \|x_i^{t+1}\|^2 &=\Vert x_i^t\Vert^2 - 2\eta \sum_{j=1}^n ((x_i^t)^\top x_j^t)^2  + \eta^2 \left(\sum_{j,k=1}^n (x_i^t)^\top  x_j^t (x_j^t)^\top \right)^2\nonumber\\&\ge  \Vert x_i^t \Vert^2 (1-2\eta \sum_{j=1}^n \Vert x_j^t\Vert^2 \theta_{ij,t}) \label{x_i norm lower bound}\\
    &\ge \Vert x_i \Vert^2 (1-2\eta r\sigma_1 \theta_t^U- 2\eta n \alpha^2)\nonumber\\&\ge \Vert x_i \Vert^2(1-\eta\cdot \sigma_r/8)\nonumber
\end{align}
The last inequality holds because 
\begin{gather}
    \theta_t^U\le \theta_0 \le \cO(1/r\kappa)\\
    \alpha \le \sqrt{\sigma_r/n}
\end{gather}
Hence, the term $\theta^U/\Vert x_i\Vert^2$ for $i>r$ is also linear convergence by 
\begin{align*}
    \frac{\theta^U_{t+1}}{\sum_{i>r}\|x_i^{t+1}\|^2} \le \frac{\theta^U_{t}}{\sum_{i>r}\|x_i^{t}\|^2}\cdot \frac{1-\eta\cdot \sigma_r/4}{1-\eta\cdot \sigma_r/8} \le \frac{\theta^U_{t}}{\sum_{i>r}\|x_i^{t}\|^2}\cdot\left(1-\frac{\eta \sigma_r}{8}\right).
\end{align*}
Hence, we complete the proof of Eq.\eqref{ineq:sym phase 2 condition 1}.
\subsection{Phase 3: lower bound of convergence rate}
Now by \eqref{ineq:sym phase 2 condition 1}, there are constants $c_6$ and $c_7$ such that, if we  denote $T_2=T_1+ c_7(\log(\sqrt{r\sigma_1}/\alpha)/\eta\sigma_r) = c_6 (\log(\sqrt{r\sigma_1}/\alpha)/\eta\sigma_r) $, then we will have 
\begin{align}\label{ineq:theta_t^U<sum}
    \theta_{T_2}^U < \sum_{i>r}\Vert x_i^{T_2} \Vert^2/r\sigma_1
\end{align}
because of the fact that  $\theta_{T_1}^U/\sum_{i>r}\|x_i^{T_1}\|^2 \le \frac{4}{n\cdot \alpha^2}\le 4/\alpha^2$. Now
after round $T_2$, 
 consider $i>r$, we can have
\begin{align*}
    \Vert x_i^{t+1} \Vert^2 &\ge  \Vert x_i^t \Vert^2 (1-2\eta \sum_{j=1}^n \Vert x_j^t\Vert^2 \theta_{ij,t}) \\&\ge \Vert x_i ^t\Vert^2 (1-2\eta r \sigma_1\theta_t^U- 2\eta  \sum_{j>r}\Vert x_j^t\Vert^2)
\end{align*}
Hence, by Eq.\eqref{x_i norm lower bound}, we have
\begin{align}
    \sum_{j>r}\Vert x_j^{t+1}\Vert^2 &\ge \left(\sum_{j>r}\Vert x_j^t\Vert^2\right) \left(1-2\eta r\sigma_1 \theta_t^U - 2\eta \sum_{j>r}\Vert x_j^t\Vert^2\right)\\
    &\ge \left(\sum_{j>r}\Vert x_j^t\Vert^2\right) \left(1 - 4\eta \sum_{j>r}\Vert x_j^t\Vert^2\right),\label{ineq:updating rule of sum xj}
\end{align}
where the second inequality is derived from \eqref{ineq:theta_t^U<sum}.

Hence, we can show that  $\sum_{j>r}\Vert x_j^t \Vert^2 = \Omega(1/T^2)$. In fact, suppose at round $T_2$, we denote $A_{T_2} = \sum_{j>r}\Vert x_j^{T_2}\Vert^2$, then
by 
\begin{align*}
    \|x_i^{t+1}\|^2 &\ge  \Vert x_i^t \Vert^2 (1-2\eta \sum_{k=1}^n \Vert x_k^t\Vert^2 \theta_{ik,t}))\\
    &\ge \Vert x_i^t \Vert^2 (1-2\eta r\sigma_1 \theta^U- 2\eta n \alpha^2)
\end{align*}
we can get 
\begin{align}
    \|x_i^{T_2}\|^2 
    &\ge \Vert x_i^{T_1} \Vert^2 (1-2\eta r\sigma_1 \theta_{T_1}^U- 2\eta n \alpha^2)^{T_2-T_1}\nonumber\\
    &\ge \Vert x_i^{T_1} \Vert^2 \cdot (1-c_5(\log(r\sqrt{\sigma_1}/\alpha)/\eta\sigma_r)\cdot \left(2\eta  r\sigma_1 \theta_{T_1}+2\eta  n\alpha^2\right))\nonumber\\
    &\ge \|x_i^{T_1}\|^2 \cdot (1-c_5\log(r\sqrt{\sigma_1}/\alpha) \cdot (4 r\kappa \theta_{0} + 2n\alpha^2/\sigma_r))\nonumber\\
    &\ge\frac{1}{2} \Vert x_i^{T_1} \Vert^2 \label{ineq: explain}\\&\ge \frac{\alpha^2}{8}\nonumber
\end{align}
where the inequality \eqref{ineq: explain} is because 
\begin{align}
    \theta_0 &\le \cO\left(\frac{1}{r\kappa \log(r\sqrt{\sigma_1}/\alpha)}\right)\\
    \alpha^2 &\le \cO\left(\frac{\sqrt{\sigma_r}}{n\log(r\sqrt{\sigma_1}/\alpha)}\right).
\end{align}
Hence, 
\begin{align}
    T_2A_{T_2}\ge T_2\cdot (n-r)\frac{\alpha^2}{8}\ge c_7(\log(\sqrt{r\sigma_1}/\alpha)/\eta \sigma_r)\cdot \frac{\alpha^2}{8}.\label{t2dt2 lower bound}
\end{align}
by $n>r$.
Define $A_{T_2+i+1} = A_{T_2 + i}(1-4\eta A_{T_2+i})$, by Eq.\eqref{ineq:updating rule of sum xj}, we have  \begin{align}
    A_{T_2+i} \le A_{T_2} = \sum_{i>r}\|x_i^{T_2}\|^2 \le 2n\alpha^2.\label{DT upper bound}
\end{align} On the other hand, if $\eta (T_2+i)A_{T_2+i} \le 1/8$, and then 
\begin{align}
    \eta(T_2+i+1)A_{T_2+i+1} &= \eta(T_2+i+1)A_{T_2 + i}(1-4\eta A_{T_2+i}) \nonumber\\&= \eta (T_2+i)A_{T_2+i}-(T_2+i)4\eta^2 A_{T_2+i}^2 +  \eta A_{T_2+i}(1-4\eta A_{T_2+i})\nonumber\\&\ge \eta (T_2+i)A_{T_2+i}-(T_2+i)4\eta^2 A_{T_2+i}^2 +  \eta A_{T_2+i}/2\label{explain lower bound}\\&\ge \eta(T_2+i)A_{T_2+i}-\eta A_{T_2+i}/2 +  \eta A_{T_2+i}/2\nonumber\\
    &\ge  \eta (T_2+i)A_{T_2+i},\nonumber
\end{align}
where \eqref{explain lower bound} holds by $\eta A_{T_2+i} \le 2n\eta \alpha^2 \le 1/8.$

If $\eta (T_2+i)A_{T_2+i} > 1/8$, since $\eta A_{T_2+i}\le 1/8$, we have $\eta A_{T_2}\le 2n\eta \alpha^2 \le 1/8.$
\begin{align*}
    \eta (T_2+i+1)A_{T_2+i+1} &\ge \eta (T_2+i)A_{T_2+i}(1-4\eta A_{T_2+i}) +  \eta A_{T_2+i}(1-4\eta A_{T_2+i})\\
    &\ge \frac{1}{8}\cdot \frac{1}{2} +\eta A_{T_2+i}\cdot \frac{1}{2}\\
    &\ge \frac{1}{16}.
\end{align*}

Thus, by the two inequalities above, at round $t\ge T_2$, we can have 
\begin{align*}
    \eta tA_t\ge \min\{\eta T_2A_{T_2}, 1/16\}.
\end{align*}
Now by \eqref{t2dt2 lower bound}, \begin{align}
    \eta T_2A_{T_2} \ge  \frac{c_7\log(\sqrt{r\sigma_1}/\alpha)\alpha^2}{8\sigma_r},
\end{align} then for any $t \ge T_2,$ we have \begin{align}\eta tA_t \ge \min\left\{\frac{c_7\log(\sqrt{r\sigma_1}/\alpha)\alpha^2}{8\sigma_r}, 1/16\right\}\end{align} Now by choosing $\alpha = \widetilde{\cO}(\sqrt{\sigma_r})$ so that $\frac{c_7\log(\sqrt{r\sigma_1}/\alpha)\alpha^2}{8\sigma_r} \le 1/16$, we can derive 
\begin{align}
    A_t\ge \frac{c_7\log(\sqrt{r\sigma_1}/\alpha)\alpha^2}{8\sigma_r \eta t}.
\end{align}
Since for $j>r$, $(X_tX_t^\top -\Sigma)_{jj} = \|x_j^t\|^2$, we have $\|X_tX_t^\top -\Sigma\|^2 \ge \sum_{j>r}\|x_j^t\|^4 \ge A_t^2/n$ and 
\begin{align*}
    \|X_tX_t^\top -\Sigma\|^2  \ge A_t^2/n \ge \left(\frac{c_7\log(\sqrt{r\sigma_1}/\alpha)\alpha^2}{8\sigma_r \eta \sqrt{n}t}\right)^2.
\end{align*}

\section{Proof of Theorem \ref{thm: asym.toy_case}}
\label{sec: proof_toy_case}
Denote the matrix of the first $r$ row of $F,G$ as $U,V$ respectively, and the matrix of the last $n-r$ row of $F,G$ as $J,K$ respectively. Hence, $U,V \in \RR^{r\times k}, J,K \in \RR^{(n-r)\times k}$. 
In this case, the difference $F_tG_t^\top -\Sigma$ can be written in a block form as 
\begin{align}
    F_tG_t^\top -\Sigma  = 
\begin{pmatrix}
U_tV_t^\top -\Sigma_r & J_tV_t^\top   \\
U_tK_t^\top  & J_tK_t^\top  \\
\end{pmatrix},
\end{align}
where $\Sigma_r = I \in \RR^{r\times r}$.
Hence, the loss can be bounded by \begin{align}\|J_tK_t^\top \|\le \|F_tG_t^\top -\Sigma\|\le \|U_tV_t^\top -\Sigma_r\| + \|J_tV_t^\top \|+\|U_tK_t^\top \| + \|J_tK_t^\top \|.\label{eq:two bound of loss}\end{align}

The updating rule for $(U,V,J,K)$ under gradient descent in \eqref{eq: asym.mf.gd} can be rewritten explicitly as
\begin{align*}
    U_{t+1} &= U_t + \eta \Sigma_r V_t - \eta U_t(V_t^\top V_t+K_t^\top K_t) \\
    V_{t+1} &= V_t + \eta \Sigma_r U_t -\eta V_t(U_t^\top U_t + J_t^\top J_t) \\
    J_{t+1} &= J_t -\eta J_t (V_t^\top V_t + K_t^\top K_t) \\
    K_{t+1} &= K _t- \eta K_t (U_t^\top U_t+J_t^\top J_t).
\end{align*}

Note that with our particular initialization, we have the following equality for all $t$:
\begin{equation} \label{eq: simpleAssumptionToyCase}
U_tK_t^\top  = 0, J_tV_t^\top  = 0,\quad \text{and}\quad U_t = V_t. 
\end{equation} 
Indeed, the conditions \eqref{eq: simpleAssumptionToyCase} are satisfied for $t=0$. For $t+1$, we have
\begin{gather*}
    U_{t+1} = U_t + \eta (\Sigma_r - U_tV_t^\top ) V_t = V_t + \eta (\Sigma_r - U_tV_t^\top ) U_t=V_{t+1}, \ \ 
    K_{t+1} = K_t -\eta K_tJ_t^\top J_t\\
    U_{t+1}K_{t+1}^\top  = U_tK_t^\top  + \eta(\Sigma_r-U_tV_t^\top )U_tK_t^\top  - \eta V_tJ_t^\top J_tK_t^\top  - \eta^2 (\Sigma_r - U_tV_t^\top )U_tJ_t^\top J_tK_t^\top  = 0
\end{gather*}
The last equality arises from the fact that $U_tK_t^\top  = 0, J_tV_t^\top  = 0$ and $U_t = V_t$. Similarly, we can get $J_{t+1}V_{t+1}^\top  = 0$. 
Hence, we can rewrite the updating rule of $J_t$ and $K_t$ as 
\begin{align}
    J_{t+1} &= J_t -\eta J_t  K_t^\top K_t \label{toycase:updating rule of J}\\
    K_{t+1} &= K _t- \eta K_t J_t^\top J_t.\label{toycase:updating rule of K}
\end{align}
Let us now argue why the convergence rate can not be faster than $\Omega((1-6\eta\alpha^2)^t)$. Denote $A \in \RR^{(n-r)\times k}$ as the matrix that $(A)_{1k} = 1$ and other elements are all zero. We have that $J_0 = \alpha A$ and $K_0 = (\alpha/3)\cdot A$. Combining this with Eq.\eqref{toycase:updating rule of J} and Eq.\eqref{toycase:updating rule of K}, we have $J_{t} = a_t A, K_t = b_t A$, where
\begin{subequations}\label{eq: toy_case.JKevolution}
\begin{align}
    a_0 = \alpha, b_0  = \alpha/3,\\
    a_{t+1} = a_t - \eta a_t b_t^2,\label{appendix_toycase:updating rule of a}\\
    b_{t+1} = b_t - \eta a_t^2b_t.\label{appendix_toycase:updating rule of b}
\end{align}
\end{subequations}

It is immediate that $0\le  a_{t+1}\le a_t, 0\le b_{t+1} \le b_t$ , $\max\{a_t,b_t\}\le \alpha$ because of $\eta b_t^2\le \eta b_0^2 
 = \eta \alpha^2 \le 1$ and similarly $\eta a_t^2 \le 1$. Now by $\eta \alpha^2 \le 1/4,$ 
\begin{align}
    \|J_{t+1}K_{t+1}^\top \| = a_{t+1}b_{t+1}=(1-\eta a_t^2)(1-\eta b_t^2)a_tb_t \ge (1-2\eta \alpha^2)^2 a_tb_t\ge (1-4\eta \alpha^2)a_tb_t.\label{toycase:slow convergence}
\end{align}


By Eq.\eqref{eq:two bound of loss} that $\|F_tG_t^\top -\Sigma\|\geq \|J_tK_t^\top \|$, 
the convergence rate of  $\|F_tG_t^\top -\Sigma\|$ can not be faster  than $a_0b_0(1-4\eta\alpha^2)^t\ge \frac{\alpha^2}{3}(1-4\eta \alpha^2)^t.$ 

Next, we show why the convergence rate is exactly $\Theta((1-\Theta(\eta \alpha^2))^t)$ in this toy case. By Eq.\eqref{eq: simpleAssumptionToyCase}, the loss $\|F_tG_t^\top -\Sigma\| \le \|U_tU_t^\top -\Sigma_r\| + \|J_tK_t^\top \|.$ First, we consider the norm $ \|U_tU_t^\top -\Sigma_r\|$. Since in this toy case, $\Sigma_r = I_r$ and $U_t = V_t$ for all $t$, the updating rule of $U_t$ can be written as 
\begin{align}
    U_{t+1} = U_t - \eta(U_tU_t^\top -I)U_t
\end{align}
Note that $U_0 = (\alpha I_r,0) \in \RR^{r\times k}$. By induction, we can show that 
%
%
$U_{t} = (\alpha_{t} I_r, 0)$ and $\alpha_{t+1} = \alpha_t-\eta (\alpha_t^2-1)\alpha_t$ for all $t\geq 0$. 
If $\alpha_t \le 1/2$, we have \begin{align*}
    \alpha_{t+1}&= \alpha_t(1+\eta-\eta \alpha_t^2)\ge \alpha_t(1+\eta/2).
\end{align*}
Then, there exists a constant $c_1$ and $T_1 = c_1(\log(1/\alpha)/\eta)$ such that after $T_1$ rounds, we can get $\alpha_t\ge 1/2$. By the fact that $\alpha_{t+1} = \alpha_t(1+\eta(1-\alpha_t^2)) \le \max\{\alpha_t,2\}$ when $\eta < 1$, it is easy to show $\alpha_t \le 2$ for all $t \ge 0.$ Thus, when $\eta < 1/6$, we can get $1-\eta(\alpha_t+1)\alpha_t >0$ and then  \begin{align*}|\alpha_{t+1}-1| &= |(\alpha_t-1) - \eta (\alpha_t-1)(\alpha_t+1)\alpha_t| \\&= |\alpha_t-1|(1-\eta (\alpha_t+1)\alpha_t)\\&\le |\alpha_t-1|(1-\eta/2).\end{align*}
we know that $\|U_tU_t^\top -\Sigma_r\| = \alpha_t^2-1$ converges at a linear rate \begin{align}
    \|U_tU_t^\top - \Sigma \| \le (1-\eta/2)^{t-T_1} \overset{\text{(a)}}{\le} (1-\eta\alpha^2/4)^{(t-T_1)/2},\label{converge part 1 toy case}
\end{align}
where (a) uses the fact that 
\begin{align}
    1-\eta \alpha^2/4 \ge 1-\eta \ge (1-\eta/2)^2
\end{align}

Hence, we only need to show that $\|J_tK_t^\top \|$ converges at a relatively slower speed $\cO((1-\Theta(\eta \alpha^2))^t)$.
To do this, we prove the following statements by induction.
\begin{gather}\label{eq: toy_JtKt}
    \alpha \ge a_t \ge \alpha/2, \ \ 
    b_{t+1}^2\le b_t^2(1-\eta\alpha^2/4)
\end{gather}
Using $b_0 = \alpha/3$, we see the above implies that 
$\|J_tK_t^\top \|=a_tb_t\leq \cO((1-\Theta(\eta \alpha^2))^t)$. 

Let us prove \eqref{eq: toy_JtKt} via induction. It is trivial to show it holds at $t=0$ and the upper bound of $a_t$ by \eqref{eq: toy_case.JKevolution}. Suppose \eqref{eq: toy_JtKt} holds for $t' \le t$, then at round $t+1$, we have
\begin{align}
    b_{t+1}^2=b_t^2(1-\eta a_t^2)^2\le b_t^2(1-\eta\alpha^2/4)^2\le b_t^2(1-\eta\alpha^2/4).
\end{align}
Using $a_{t+1} = a_t(1-\eta b_t^2)$, we have 
\begin{align}
    a_{t+1} = a_0 \prod_{i=1}^t (1-\eta b_i^2) \overset{(a)}{\ge} a_0 \left(1-\eta \sum_{i=1}^t b_i^2\right) \overset{(b)}{\ge} \alpha \cdot \left(1-\eta \cdot \frac{\alpha^2}{9}\cdot \frac{4}{\eta\alpha^2}\right)\ge \alpha/2.
\end{align}
where the step $(a)$ holds by recursively using $(1-a)(1-b) \ge (1-(a+b))$ for $a,b\in (0,1)$, and the step $(b)$ is due to $b_i^2 \le b_0^2 \cdot (1-\eta\alpha^2/4)^{t}\le \frac{\alpha^2}{9}\cdot (1-\frac{\eta\alpha^2}{4})^t$ and the sum formula for geometric series. Thus, the induction is complete, and 
\begin{align}
    \|J_tK_t^\top \| =a_tb_t\le (\alpha^2/3)\cdot (1-\eta\alpha^2/4)^{t/2} \le (1-\eta \alpha^2/4)^{t/2} \le (1-\eta\alpha^2/4)^{(t-T_1)/2}. \label{toy case JK}
\end{align}
Combining \eqref{converge part 1 toy case} and \eqref{toy case JK}, with $\|A\|_2 \le \|A\|_\mathrm{F} \le \text{rank}(A)\cdot \|A\|_2$, we complete the proof.

\section{Proof of Theorem \ref{theorem:main theorem assymmetric}}\label{appendix: proof of asymmetric}
We prove Theorem \ref{theorem:main theorem assymmetric} in this section. We start with some preliminaries.
\subsection{Preliminaries}\label{sec: Preliminaries_asym_ms}
In the following, we denote $\delta_{2k+1} = \sqrt{2k+1}\delta$.
Also denote the matrix of the first $r$ row of $F,G$ as $U,V$ respectively, and the matrix of the last $n-r$ row of $F,G$ as $J,K$ respectively. Hence, $U,V \in \RR^{r\times k}, J,K \in \RR^{(n-r)\times k}$. We denote the corresponding iterates as $U_t$, $V_t$, $J_t$, and $K_t$.

Also, define $E(X) = \cA^*\cA(X)-X$. We also denote $\Gamma(X) = \cA^*\cA(X)$. By Lemma \ref{lemma: rip matrix}, we can show that $\Vert E(X)\Vert \le \delta_{2k+1}\cdot \|X\|$ for matrix $X$ with rank less than $2k$ by Lemma \ref{lemma: rip matrix}. 
Decompose the error matrix $E(X)$ into four submatrices by 
\begin{align*}
    E(X) = \begin{pmatrix}
        E_1(X) & E_2(X)\\
        E_3(X) & E_4(X)
    \end{pmatrix},
\end{align*}
where $E_1(X) \in \RR^{r\times r}, E_2(X) \in \RR^{r\times (n-r)}, E_3(X) \in \RR^{(n-r)\times r}, E_4(X) \in \RR^{(n-r)\times (n-r)}$.
Then
the updating rule can be rewritten in this form:
\begin{align}
    U_{t+1} &= U_t + \eta \Sigma V_t - \eta U_t(V_t^\top V_t+K_t^\top K_t) + \eta E_1(F_tG_t^\top -\Sigma)V_t  + \eta E_2(F_tG_t^\top -\Sigma)K_t\label{update rule: U}\\
    V_{t+1} &= V_t + \eta \Sigma U_t -\eta V_t(U_t^\top U_t + J_t^\top J_t) + \eta E_1^\top (F_tG_t^\top -\Sigma)U_t  + \eta E_3^\top (F_tG_t^\top -\Sigma)J_t\label{update rule: V}\\
    J_{t+1} &= J_t -\eta J_t (V_t^\top V_t + K_t^\top K_t) + \eta E_3(F_tG_t^\top -\Sigma)V_t  + \eta E_4(F_tG_t^\top -\Sigma)K_t\label{update rule: J}\\
    K_{t+1} &= K_t - \eta K_t (U_t^\top U_t+J_t^\top J_t) + \eta E_2^\top (F_tG_t^\top -\Sigma)U_t  + \eta E_4^\top (F_tG_t^\top -\Sigma)J_t.\label{update rule: K}
\end{align}
Since the submatrices' operator norm is less than the operator norm of the whole matrix, the matrices 
$E_i(F_tG_t^\top -\Sigma)$, $i=1,\;\dots,\;4$ satisfy  that 
\[
\Vert E_i(F_tG_t^\top -\Sigma)\Vert \le \|E(F_tG_t^\top -\Sigma)\| \le \delta_{2k+1} \Vert F_tG_t^\top -\Sigma\Vert, \quad i=1,\;\dots,\;4.
\]

\paragraph{Imbalance term} 
An important property in analyzing the asymmetric matrix sensing problem is that $F^\top F - G^\top G =U^\top U+J^\top J-V^\top V-K^\top K$ remains almost unchanged when step size $\eta$ is sufficiently small, i.e., the balance between two factors $F$ and $G$ are does not change much throughout the process. To be more specific, by \begin{gather}
    F_{t+1}=F_t - \eta (F_tG_t^\top -\Sigma)G_t - E(F_tG_t^\top -\Sigma)G_t\nonumber\\
    G_{t+1}=G_t - \eta (F_tG_t^\top -\Sigma)^\top F_t - (E(F_tG_t^\top -\Sigma))^\top F_t\nonumber
\end{gather}we have
\begin{align}
\left \|
\left(F_{t+1}^\top F_{t+1}-G_{t+1} ^\top G_{t+1} \right)
- \left(F_{t}^\top F_{t}-G_{t} ^\top G_{t} \right)
\right\| \le  2 \eta^2\cdot \|F_tG_t^\top -\Sigma\|^2 \cdot \max\{\|F_t\|, \|G_t\|\}^2.\label{eq: Delta_t change little}
\end{align}
In fact, by the updating rule, we have
\begin{align}
    &F_{t+1}^\top F_{t+1}-G_{t+1} ^\top G_{t+1}\nonumber\\& =F_t^\top  F_t-G_t^\top G_t + \eta^2 \Big(G_t^\top (F_tG_t^\top -\Sigma)^\top (F_tG_t^\top -\Sigma)G_t-F_t^\top (F_tG_t^\top -\Sigma)(F_tG_t^\top -\Sigma)^\top F_t\Big),\nonumber
\end{align}
so that 

\begin{align}
    \|F_{t+1}^\top F_{t+1}-G_{t+1} ^\top G_{t+1}-(F_t^\top  F_t-G_t^\top G_t)\| 
    &\le  2\eta^2 \|F_t\|^2\|G_t\|^2\|F_tG_t^\top -\Sigma\|^2 \nonumber
    \\
    &\le  2\eta^2 \cdot \|F_t G_t^\top -\Sigma \|\cdot \max\{\|F_t\|^2,\|G_t\|^2\}\nonumber
    \end{align}

Thus, we will prove that, during the proof process, the following inequality holds with high probability during all $t\geq 0$:
\begin{align}\label{eq:keyproperty}
        2\alpha^2 I \ge U_{t}^\top U_{t} + J_{t}^\top J_{t}-V_{t}^\top V_t-K_{t}^\top K_t \ge \frac{\alpha^2}{8}I.
\end{align}

Next, we give the outline of our proof. 

\subsection{Proof Outline}
In this subsection, we give our proof outline.

$\bullet$ Recall $\Delta_t = F_t^\top F_t - G_t^\top G_t = U_{t}^\top U_{t} + J_{t}^\top J_{t}-V_{t}^\top V_t-K_{t}^\top K_t $.  In Section \ref{sec:initial iterations}, we show that with high probability, $\Delta_0$ has the scale $\alpha,$ i.e., $C \alpha^2I \ge \Delta_0 \ge c\alpha^2 I $, where $C>c$ are two constants. Then, we apply the converge results in \cite{soltanolkotabi2023implicit} to argue that the algorithm first converges to a local point. By \cite{soltanolkotabi2023implicit}, this converge phase takes at most $T_0 = \cO((1/\eta\sigma_r\upsilon) \log(\sqrt{\sigma_1}/n\alpha))$ rounds.

$\bullet$ Then, in Section \ref{sec: phase 1 gd linear} (Phase 1), we mainly show that $M_t = \max\{\|U_tV_t^\top - \Sigma\|, \|U_tK_t^\top \|, \|J_tV_t^\top \|\}$ converges linearly until it is smaller than \begin{align}M_t \le \cO(\sigma_1\delta + \alpha^2) \|J_tK_t^\top \|.\label{stopping rule outline}\end{align} This implies that the difference between estimated matrix $U_tV_t^\top$ and true matrix $\Sigma$, $\|U_tV_t^\top - \Sigma\|$, will be dominated by $\|J_tK_t^\top \|.$ Moreover, during Phase 1 we can also show that $\Delta_t$ has the scale $\alpha.$ Phase 1 begins at $T_0$ rounds and terminates at $T_1$ rounds, and $T_1$ may tend to infinity, which implies that Phase 1 may not terminate.  In this case, since $M_t$ converges linearly and $M_t > \Omega(\sigma_1\delta + \alpha^2) \|J_tK_t^\top \|$, the loss also converges linearly. Note that, in the exact-parameterized case, i.e., $k=r$, we can prove that Phase 1 will not terminate since the stopping rule \eqref{stopping rule outline} is never satisfied as shown in Section \ref{sec: appendix_proofThmAsym}. 

$\bullet$ The Section \ref{sec: adjustment}  (Phase 2) mainly shows that, after Phase 1, the $\|U_t-V_t\|$ converges linearly until it achieves $$\|U_t-V_t\| \le \cO(\alpha^2/\sqrt{\sigma_1}) + \cO(\delta_{2k+1}\|J_tK_t^\top \|/\sqrt{\sigma_1}).$$ 
Assume Phase 2 starts at round $T_1$ and terminates at round $T_2.$ Then
 since we can prove that $\|U_t-V_t\|$ decreases from 
 \footnote{The upper bound $\cO(\sigma_1)$ of $\|U_t-V_t\|$ is proved in the first two phases.}
 $\cO(\sigma_1)$ to $\Omega(\alpha^2)$, Phase 2 only takes a relatively small number of rounds, i.e. at most $T_2-T_1=\cO(\log(\sqrt{\sigma_r}/\alpha)/\eta\sigma_r)$ rounds. We also show that $M_t$ remains small in this phase.

 $\bullet$ The Section \ref{sec: local cvg gd linear} (Phase 3) finally shows that the norm of $K_t$ converges linearly, with a rate dependent on the initialization scale. 
 As in Section \ref{sec: result and asymmetric}, the error matrix in matrix sensing brings additional challenges for the proof. We overcome this proof by further analyzing the convergence of (a) part of $K_t$ that aligns with $U_t$, and (b) part of $K_t$ that lies in the complement space of $U_t.$ We also utilize that $M_t$ and $\|U_t-V_t\|$ are small from the start of the phase and remain small.
 See Section \ref{sec: local cvg gd linear} for a detailed proof.

\subsection{Initial iterations}\label{sec:initial iterations}
We start our proof by first applying results in \cite{soltanolkotabi2023implicit} and provide some additional proofs for our future use. From \cite{soltanolkotabi2023implicit}, the converge takes at most $T_0 = \cO((1/\eta \sigma_r\upsilon)\log (\sqrt{\sigma_1}/n\alpha))$ rounds. 

Let us state a few properties of the initial iterations using Lemma \ref{lemma: mahdi complete}.
\paragraph{Initialization}
By our imbalance initialization $F_0 = \alpha \cdot \widetilde{F}_0, G_0 = (\alpha/3) \cdot \widetilde{G}_0$, and by random matrix theory about the singular value  \cite[Corollary 7.3.3 and 7.3.4]{vershynin2018high}, with probability at least $1-2\exp(-cn)$ for some constant $c$, if $n>8k$, 
we can show that $[\sigma_{\min}(F_0), \sigma_{\max}(F_0))]\subseteq [\frac{\sqrt{3}\alpha}{2},\frac{\sqrt{3}\alpha}{\sqrt{2}}]$, $[\sigma_{\min}(G_0),\sigma_{\max}(G_0)]\subseteq [\frac{\sqrt{3}\alpha}{6},\frac{\alpha}{\sqrt{6}}]$ and 
\begin{align}
    \frac{3\alpha^2}{2}I \ge F_0^\top F_0 - G_0^\top G_0 = U_0^\top U_0 + J_0^\top J_0 - V_0^\top V_0 - K_0^\top K_0 \ge \frac{\alpha^2}{2} I \label{eq:imbalance initialization}
\end{align}
As we will show later, we will prove the \eqref{eq:keyproperty} during all phases by \eqref{eq: Delta_t change little} and \eqref{eq:imbalance initialization}.

First, we show the following lemma, which is a subsequent corollary of the Lemma \ref{lemma: mahdi complete}.
\begin{lemma}\label{lem: initial iterations}
    There exist parameters $\zeta_0$, $\delta_0, \alpha_0, \eta_0$ such that, if we choose $\alpha\le \alpha_0$, $F_0 = \alpha\cdot \tilde{F}_0, G_0 = (\alpha/2) \cdot \tilde{G}_0$, where the elements of $\tilde{F}_0, \tilde{G}_0$ is $\cN(0,1)$,\footnote{Note that in \cite{soltanolkotabi2023implicit}, the initialization is $F_0 = \alpha \cdot \tilde{F_0}$ and $G_0 = \alpha \cdot \tilde{G_0}$, while Lemma \ref{lemma: mahdi complete} uses an imbalance initialization. It is easy to show that their results continue to hold with this imbalance initialization.} and suppose that the operator $\cA$ defined in Eq.\eqref{eq: observation_model} satisfies the restricted isometry property of order $2r+1$ with constant $\delta \le \delta_{0}$, then the gradient descent with step size $\eta \le \eta_0$ will achieve 
    \begin{align}
        \|F_tG_t^\top -\Sigma\| \le \min\{\sigma_r/2, \alpha^{1/2}\cdot \sigma_1^{3/4}\} \label{initial: condition 0}
    \end{align}
%
%
    within $T_0 = c_2(1/\eta \sigma_r)\log(\sqrt{\sigma_1}/n\alpha)$ rounds with probability at least $1-\zeta_0$ and constant $c_2\ge 1$, where $\zeta_0 = c_1\exp (-c_2k) + \exp(-(k-r+1))$ is a small constant. Moreover, during $t\le T_0$ rounds, we always have 
    \begin{gather}
        \max\{\|F_t\|,\|G_t\|\} \le 2\sqrt{\sigma_1}\label{initial: condition 1}\\
        \|U_t-V_t\| \le 4\alpha+\frac{40\delta_{2k+1}\sigma_1^{3/2}}{\sigma_r}\label{initial: condition 2}\\
        \|J_{t}\| \le   \cO\Big(2\alpha+\frac{\delta_{2k+1}\sigma_1^{3/2}\log(\sqrt{\sigma_1}/n\alpha)}{\sigma_r}\Big)\label{initial: condition 3}\\
        \frac{13\alpha^2}{8}I\ge \Delta_t\ge \frac{3\alpha^2}{8}I\label{initial: condition 4}
    \end{gather}
    
\end{lemma}

\begin{proof}
Since the initialization scale $\alpha\le \cO(\sqrt{\sigma_1})$, Eq.\eqref{initial: condition 1}, Eq.\eqref{initial: condition 2}, Eq.\eqref{initial: condition 3} and Eq.\eqref{initial: condition 4} hold for $t'=0$.
Assume that Eq.\eqref{initial: condition 0}, Eq.\eqref{initial: condition 1}, Eq.\eqref{initial: condition 2}, Eq.\eqref{initial: condition 3} and Eq.\eqref{initial: condition 4} hold for $t'=t-1$.
~\newline
\textbf{Proof of Eq.\eqref{initial: condition 0} and Eq.\eqref{initial: condition 1}}

First, by using the previous global convergence result Lemma \ref{lemma: mahdi complete}, the Eq.\eqref{initial: condition 0} holds by $\alpha^{3/5}\sigma_1^{7/10} < \sigma_r/2$ because $\alpha \le \cO(\sigma_r^{5/3}/\sigma_1^{7/6})=\cO(\kappa^{7/6}\sqrt{\sigma_r})$. Also, by Lemma \ref{lemma: mahdi complete}, Eq.\eqref{initial: condition 1} holds for all $t \in [T_0]$.

~\\
\textbf{Proof of Eq.\eqref{initial: condition 4}}

Recall $\Delta_t = U_t^\top U_t + J_t^\top J_t - V_t^\top V_t - K_t^\top K_t$, then for all $t \le T_0$, we have  $$\|\Delta_t-\Delta_{0}\|\le 2\eta^2 \cdot 25\sigma_1^2\cdot T_0\cdot 4\sigma_1\le 2c_2\log(\sqrt{\sigma_1}/n\alpha)(20\sigma_1^{3}\eta /\sigma_r) = 200c_2\eta\kappa\sigma_1^{2}\log(\sqrt{\sigma_1}/n\alpha) \le \alpha^2/8.$$ 
The first inequality holds by Eq.\eqref{eq: Delta_t change little} 
and $\|F_tG_t-\Sigma\| \le \|F_t\|\|G_t\|+\|\Sigma\| \le 5\sigma_1$.
The last inequality uses the fact that $\eta = \cO(\alpha^2/\kappa\sigma_1^2\log(\sqrt{\sigma_1}/n\alpha))$. Thus, at $t= T_0$, 
we have $\lambda_{\min}(\Delta_{T_0})\ge \lambda_{\min}(\Delta_0)-\alpha^2/8 \ge \alpha^2/2-\alpha^2/8 = 3\alpha^2/8$ and $\|\Delta_{T_0}\|\le \|\Delta_0\| + 3\alpha^2/2 + \alpha^2/8 = 13\alpha^2/8$.
~\\
\textbf{Proof of Eq.\eqref{initial: condition 2}}


Now we can prove that $\|U-V\|$ keeps small during the initialization part.
In fact, by Eq.\eqref{update rule: U} and Eq.\eqref{update rule: V}, we have 
\begin{align*}
    &\quad\Vert (U_{t+1}-V_{t+1})\Vert  \\&\le  \Vert U_t-V_t\Vert \Vert I-\eta \Sigma - \eta (V_t^\top V_t+K_t^\top K_t))\Vert  + \eta \Vert V_t \Vert \Vert U_t^\top U_t+J_t^\top J_t-V_t^\top V_t-K_t^\top K_t\Vert \\& \qquad \qquad+ 4\eta \delta_{2k+1} \Vert F_tG_t^\top -\Sigma\Vert \max\{\Vert U_t \Vert, \|V_t\|, \|J_t\|, \|K_t\|\}\\
    &\le (1-\eta \sigma_r)\Vert U_t-V_t\Vert  + 2\eta \alpha^2 \cdot 2\sqrt{\sigma_1} + 4\eta \delta_{2k+1} \cdot (\|F_t\|\|G_t\|+\|\Sigma\|)\cdot  2\sqrt{\sigma_1}\\
    &\le (1-\eta \sigma_r)\Vert U_t-V_t\Vert  + 2\eta \alpha^2 \cdot 2\sqrt{\sigma_1} + 40\eta \delta_{2k+1} \cdot \sigma_1^{3/2}.
\end{align*}
The second  inequality uses the inequality \eqref{eq:keyproperty}, while the third  inequality holds by $\max\{\|F_t\|, \|G_t\|\}\le 2\sqrt{\sigma_1}$.
Thus, since $\alpha = \cO(\delta_{2k+1}\sigma_1^{3/2}/\sigma_r)$, we can get $\Vert U_0-V_0 \Vert \le 4\alpha\le 4\alpha+\frac{40}{\sigma_r}\delta_{2k+1}\sigma_1^{3/2}$. If $\|U_t-V_t\| \le 4\alpha+\frac{40}{\sigma_r}\delta_{2k+1}\sigma_1^{3/2}$, we know that 
\begin{align*}
    \|U_{t+1}-V_{t+1}\| &\le (1-\eta \sigma_r)\left(4\alpha+\frac{40}{\sigma_r}\delta_{2k+1}\sigma_1^{3/2}\right) + 4\eta \alpha^2 \sqrt{\sigma_1} + 40\eta \delta_{2k+1}\cdot \sigma_1^{3/2} \\&\le (1-\eta \sigma_r)\left(4\alpha+\frac{40}{\sigma_r}\delta_{2k+1}\sigma_1^{3/2}\right)  + 4\eta \sigma_r\alpha+\frac{40}{\sigma_r}\delta_{2k+1}\sigma_1^{3/2}\\
    &\le 4\alpha+\frac{40}{\sigma_r}\delta_{2k+1}\sigma_1^{3/2}.
\end{align*}
Hence, $\Vert U_t-V_t\Vert \le 4\alpha+\frac{40}{\sigma_r}\delta_{2k+1}\sigma_1^{3/2}$ for $t\le T_0$ by induction. The second inequality holds by $\alpha = \cO(\sigma_r/\sqrt{\sigma_1})$
\\
\textbf{Proof of Eq.\eqref{initial: condition 3}}

Now we prove that $J_{t}$ and $K_{t}$ are bounded for all $t \le T_0$. By Eq.\eqref{update rule: J} and $\max\{\|F_t\|, \|G_t\|\}\le 2\sqrt{\sigma_1}$, denote $C_2= \max\{21c_2,32\}\ge 32$, we have 
\begin{align*}
    \|J_{T_0}\| &\le \|J_{0}\| + \eta\sum_{t=0}^{T_0-1} \max\{\|F_t\|,\|G_t\|\} \cdot 2\delta_{2k+1} \cdot (\|F_t\|\|G_t\| + \|\Sigma\|)\\
    &\le \|J_{0}\| + \eta T_0\cdot 20\sigma_1^{3/2} \cdot \delta_{2k+1}\\
    &\le \|J_0\| +   20c_2 \log(\sqrt{\sigma_1}/n\alpha) (\delta_{2k+1}\cdot \sigma_{1}^{3/2}/\sigma_r)\\
    &\le 2\alpha + 20c_2 \log(\sqrt{\sigma_1}/n\alpha) (\delta_{2k+1}\cdot \sigma_{1}^{3/2}/\sigma_r)\\
    &= 2\alpha+C_2\log(\sqrt{\sigma_1}/n\alpha) (\delta_{2k+1}\cdot \sigma_1^{3/2}/\sigma_r).
\end{align*}
Similarly, we can prove that $\|K_{T_0}\| \le 2\alpha+C_2 \log(\sqrt{\sigma_1}/n\alpha) (\delta_{2k+1}\cdot \sigma_{1}^{3/2}/\sigma_r).$ We complete the proof of Eq.\eqref{initial: condition 3}.
\end{proof}
\subsection{Phase 1: linear convergence phase.} \label{sec: phase 1 gd linear}
In this subsection, we analyze the first phase: the linear convergence phase. This phase starts at round $T_0$, and we assume that this phase terminates at round $T_1$. In this phase, the loss will converge linearly, with the rate independent of the initialization scale. Note that $T_1$ may tend to infinity, since this phase may not terminate. For example, when $k=r$, we can prove that this phase will not terminate (\S \ref{sec: appendix_proofThmAsym}), and thus leading a linear convergence rate that independent on the initialization scale.  
In this phase, we provide the following lemma, which shows some induction hypotheses during this phase. 
\begin{lemma}
Denote $M_t = \max\{\Vert U_tV_t^\top -\Sigma\Vert, \Vert U_tK_t^\top \Vert, \Vert J_tV_t^\top \Vert\}.$
Suppose Phase 1 starts at $T_0$ and ends at the first time $T_1$ such that 
\begin{align}
    \eta\sigma_r^2M_{t-1}/64\sigma_1 <(17\eta \sigma_1\delta_{2k+1}+\eta\alpha^2)\|J_{t-1}K_{t-1}^\top \|
\end{align}
During Phase 1 that $T_0\le t\le T_1$, we have the following three induction hypotheses:
\begin{gather}
    \max\{\Vert U_t \Vert, \Vert V_t \Vert\}\le 2\sqrt{\sigma_1}\label{assy:phase 1 condition 0}\\
    \Vert U_tV_t^\top -\Sigma \Vert \le \sigma_r/2.\label{assy:phase 1 condition 1}\\
    \max\{\Vert J_t \Vert, \Vert K_t \Vert\}\le 2\sqrt{\alpha}\sigma_1^{1/4}+2C_2\log(\sqrt{\sigma_1}/n\alpha) (\delta_{2k+1} \cdot \kappa^2\sqrt{\sigma_1})\le \sqrt{\sigma_1}\label{assy:phase 1 condition 2}\\
    \frac{7\alpha^2}{4}I \ge \Delta_t \ge \frac{\alpha^2}{4} I\label{assy:phase 1 condition 3}
\end{gather}
\end{lemma}

The induction hypotheses hold for $t=T_0$ due to Lemma \ref{lem: initial iterations}. Let us assume they hold for $t'<t,$ and consider the round $t.$ 
Let us first prove that the $r$-th singular value of $U$ and $V$ are lower bounded by $\text{poly}(\sigma_r,1/\sigma_1)$ at round $t$, if Eq.\eqref{assy:phase 1 condition 1} holds at round $t$. In fact, 
\begin{align}
    2\sqrt{\sigma_1}\cdot \sigma_r(U)&\ge\sigma_r(U)\sigma_1(V)\ge\sigma_r(UV^\top )\ge \sigma_r/2.\nonumber
\end{align}
which means 
\begin{align}\label{eq:lower bound of sigma_r(U)}
    \sigma_r(U)&\ge \sigma_r/4\sqrt{\sigma_1}.
\end{align}
Similarly, $\sigma_r(V)\ge \sigma_r/4\sqrt{\sigma_1}.$

\paragraph{Proof of Eq.\eqref{assy:phase 1 condition 1}}

First, since $\|U_{t-1}V_{t-1}^\top -\Sigma\| \le \sigma_r/2$, by Eq.\eqref{eq:lower bound of sigma_r(U)}, we can get 
\begin{align}
    \min\{\sigma_r(U_{t-1}), \sigma_r(V_{t-1})\} \ge \frac{\sigma_r}{4\sqrt{\sigma_1}}
\end{align}

Define $M_t=\max\{\Vert U_tV_t^\top -\Sigma\Vert, \Vert U_tK_t^\top \Vert, \Vert J_tV_t^\top \Vert\}. $ By the induction hypothesis, \begin{gather*}\max\{\|U_{t-1}\|,\|V_{t-1}\|\} \le 2\sqrt{\sigma_1},\\\max\{\|J_{t-1}\|,\|K_{t-1}\|\}\le 2\sqrt{\alpha}\sigma_1^{1/4}+2C_2\log(\sqrt{\sigma_1}/n\alpha) (\delta_{2k+1} \sigma_1^{3/2}/\sigma_r).\end{gather*} Then, by the updating rule and $C_2\ge 1$, we can get 
\begin{align}
    U_{t}K_t &= (1-\eta U_{t-1}U_{t-1}^\top )U_{t-1}K_{t-1}(1-\eta K_{t-1}K_{t-1}^\top ) + \eta (\Sigma-U_{t-1}V_{t-1}^\top )VK^\top  \nonumber\\&\quad + \eta U_{t-1}J_{t-1}^\top  J_{t-1}K_{t-1}^\top  + A_t,
\end{align}
where $A_t$ is the perturbation term that contains all $\cO(E_i(FG^\top -\Sigma))$ terms and $\cO(\eta^2)$ terms such that 
\begin{align*}
    \|A_t\| &\le 4\eta \delta_{2k+1}\|F_tG_t^\top -\Sigma\| \max\{\|F_t\|^2, \|G_t\|^2\} + 8\eta^2 \|F_tG_t^\top -\Sigma\|^2 \max\{\|F_t\|^2,\|G_t\|^2\}\\&\qquad +\eta^2 \max\{\|F_t\|^2,\|G_t\|^2\}^2 \cdot \|F_tG_t-\Sigma\|\\
    &\le  4\eta \delta_{2k+1}\|F_tG_t^\top -\Sigma\| \max\{\|F_t\|^2, \|G_t\|^2\} + 8\eta^2 \|F_tG_t^\top -\Sigma\|\cdot 5\sigma_1 \cdot 4\sigma_1\\&\qquad +\eta^2 \cdot 16\sigma_1^2 \cdot \|F_tG_t-\Sigma\|\\
    &\le 4\eta \delta_{2k+1}(3M_{t-1}+\|J_{t-1}K_{t-1}^\top \|) 4\sigma_1+\eta\alpha^2(3M_{t-1}+\|J_{t-1}K_{t-1}^\top \|)
\end{align*}
Using the similar technique for $J_tV_t^\top $ and $U_tV_t^\top -\Sigma$, we can finally get
\begin{align}
    M_t&\le \left(1-\frac{\eta\sigma_r^2}{16\sigma_1}\right) M_{t-1} + 2\eta M_{t-1} \cdot 2\sqrt{\sigma_1}\cdot \max\{\|J_{t-1}\|,\|K_{t-1}\|\} \nonumber\\&\qquad+  4\eta \delta_{2k+1}(3M_{t-1}+\|J_{t-1}K_{t-1}^\top \|) \cdot 4\sigma_1+\eta\alpha^2(3M_{t-1}+\|J_{t-1}K_{t-1}^\top \|)\nonumber\\&\le \left(1-\frac{\eta\sigma_r^2}{16\sigma_1}\right) M_{t-1} + 2\eta M_{t-1} \cdot 2\sqrt{\sigma_1}\cdot \left(\alpha+C_2\log (\sqrt{\sigma_1}/n\alpha)\delta_{2k+1}\sigma_1^{3/2}/\sigma_r\right) \nonumber\\&\qquad+  4\eta \delta_{2k+1}(3M_{t-1}+\|J_{t-1}K_{t-1}^\top \|) \cdot 4\sigma_1+\eta\alpha^2(3M_{t-1}+\|J_{t-1}K_{t-1}^\top \|)\nonumber\\&\le \left(1-\frac{\eta\sigma_r^2}{16\sigma_1}\right) M_{t-1} + \cO\left(\eta\sqrt{\sigma_1}  \cdot \left(\alpha+C_2\log(\sqrt{\sigma_1}/n\alpha) \delta_{2k+1}\sigma_1^{3/2}/\sigma_r\right)\right)\cdot M_{t-1} \nonumber\\&\qquad +(17\eta \sigma_1\delta_{2k+1}+\eta\alpha^2)\|J_{t-1}K_{t-1}^\top \|\nonumber\\&
    \le\left(1-\frac{\eta\sigma_r^2}{32\sigma_1}\right)M_{t-1} + (17\eta \sigma_1\delta_{2k+1}+\eta\alpha^2)\|J_{t-1}K_{t-1}^\top \|.\label{ineq: M recursion}
\end{align}
The last inequality holds by $\delta_{2k+1} = \cO(\sigma_r^3/\sigma_1^3\log(\sqrt{\sigma_1}/n\alpha))$ and $\alpha = \cO(\sigma_r^2/\sigma_1^{3/2})=\cO(\sqrt{\sigma_r}\kappa^{-3/2})$.

During Phase 1, we have $$\eta\sigma_r^2M_{t-1}/64\sigma_1 \ge (17\eta \sigma_1\delta_{2k+1}+\eta\alpha^2)\|J_{t-1}K_{t-1}^\top \|,$$ then
\begin{align}
    M_t\le \left(1-\frac{\eta \sigma_r^2}{64\sigma_1}\right)M_{t-1}.\label{ineq: exp decrease of M}
\end{align}
Hence, $\|U_tV_t^\top -\Sigma\| \le M_t\le M_{T_0} \le \|F_{T_0}G_{T_0}^\top -\Sigma\| \le \delta_{2k+1}.$

\paragraph{Proof of Eq.\eqref{assy:phase 1 condition 0}}
Now we bound the norm of $U_t$ and $V_t$. First, note that 
\begin{align*}
    \Vert (U_t-V_t)\Vert  
    \le (1-\eta \sigma_r)\Vert U_{t-1}-V_{t-1}\Vert  + \eta \cdot 2\alpha^2 \cdot 2\sqrt{\sigma_1} + 40\eta\cdot \delta_{2k+1} \cdot \sigma_1^{3/2}
\end{align*}
Hence, $\Vert U_t-V_t\Vert \le 4\alpha+40\delta_{2k+1}\sigma_1^{3/2}/\sigma_r$ still holds using the same technique in the initialization part. 

Thus, by the induction hypothesis Eq.\eqref{assy:phase 1 condition 1} and $\sigma_1\ge \delta_{2k+1}$, we have
\begin{align*}
   2\sigma_1\ge \sigma_1+\delta_{2k+1}\ge \|\Sigma\| + \|U_tV_t^\top -\Sigma\|&\ge \Vert U_tV_t^\top \Vert = \Vert V_tV_t^\top  + (U_t-V_t)V_t^\top \Vert \\&\ge \Vert V_tV_t^\top  \Vert - \Vert U_{t}-V_t\Vert \Vert V_t\Vert\\&\ge \Vert V_t\Vert^2 - \Vert V_t\Vert \cdot \left(4\alpha+\frac{40\delta_{2k+1}\sigma_1^{3/2}}{\sigma_r}\right) \\&\ge \|V_t\|^2-\|V_t\|.
\end{align*}
Then, we can get $\Vert V_t\Vert \le 2\sqrt{\sigma_1}$. Similarly, $\Vert U_t\Vert \le 2\sqrt{\sigma_1}.$

\paragraph{Proof of Eq.\eqref{assy:phase 1 condition 2}}
Since during Phase 1, \begin{align*}\Vert J_tK_t^\top \Vert\le  M_{t}\cdot \frac{\sigma_r^2}{64\sigma_1(17\sigma_1\delta_{2k+1}+ \alpha^2)}\le M_t\cdot \frac{1}{1088\kappa^2\delta_{2k+1}+64 \alpha^2\kappa/\sigma_r},\end{align*} by $\delta_{2k+1}<1/128$ and Eq.\eqref{ineq: exp decrease of M}, \begin{align}\Vert F_tG_t^\top -\Sigma\Vert &\le 4\max\{\Vert J_tK_t^\top \Vert, M_t\}\le 4M_t\cdot \max\Big\{1,\frac{1}{1088\kappa^2\delta_{2k+1}+64 \alpha^2\kappa/\sigma_r}\Big\}\nonumber\\&\le \Vert F_{T_0}G_{T_0}-\Sigma\Vert \left(1-\eta\sigma_r^2/64\sigma_1\right)^{t-T_0}/(1088\kappa^2\delta_{2k+1}+64 \alpha^2\kappa/\sigma_r).\label{phase 1: linear convergence}\end{align}
Thus, the maximum norm of $J_t, K_t$ can be bounded by 
\begin{align*}
    \Vert J_t\Vert &\le \Vert J_{T_0}\Vert + 2\eta \cdot 2\sqrt{\sigma_1} \delta_{2k+1}\cdot \sum_{t'=T_0}^{t-1}\Vert F_tG_t-\Sigma\Vert\\
    &\le 2\alpha+C_2\log(\sqrt{\sigma_1}/n\alpha) (\delta_{2k+1} \cdot \sigma_1^{3/2}/\sigma_r) + \frac{4\eta \sqrt{\sigma_1}\delta_{2k+1}}{1088\kappa^2\delta_{2k+1}+ 64\alpha^2\kappa/\sigma_r}\cdot  \Vert F_{T_0}G_{T_0}-\Sigma\Vert \cdot \frac{64\sigma_1}{\eta \sigma_r^2}\\&=
    2\alpha+C_2\log(\sqrt{\sigma_1}/n\alpha) (\delta_{2k+1} \cdot \sigma_1^{3/2}/\sigma_r) + \frac{ \sigma_1^{3/2}}{4\kappa^2\sigma_r^2}\cdot  \Vert F_{T_0}G_{T_0}-\Sigma\Vert
    \\&\le 2\alpha+C_2\log(\sqrt{\sigma_1}/n\alpha) (\delta_{2k+1} \cdot \sigma_1^{3/2}/\sigma_r) +\frac{\alpha^{1/2}\sigma_1^{9/4}}{4\kappa^2\sigma_r^2}\\&\le 2\sqrt{\alpha}\sigma_1^{1/4}+C_2\log(\sqrt{\sigma_1}/n\alpha) (\delta_{2k+1} \cdot \kappa^2\sqrt{\sigma_1})\\
    &\le 2\sqrt{\alpha}\sigma_1^{1/4}+2C_2\log(\sqrt{\sigma_1}/n\alpha) (\delta_{2k+1} \cdot \kappa^2\sqrt{\sigma_1}).
\end{align*}
The last inequality uses the fact that $2\alpha + \frac{\sqrt{\alpha}\sigma_1^{1/4}}{4} \le 2\sqrt{\alpha}\sigma_1^{1/4}$ by $\alpha = \cO(\sqrt{\sigma_r})$.
Similarly, $\Vert K_t\Vert \le2\sqrt{\alpha}\sigma_1^{1/4}+2C_2\log(\sqrt{\sigma_1}/n\alpha) (\delta_{2k+1} \cdot \kappa^2\cdot \sqrt{\sigma_1})$. 
We complete the proof of Eq.\eqref{assy:phase 1 condition 2}.
\paragraph{Proof of Eq.\eqref{assy:phase 1 condition 3}}
Last, for $t \in [T_0,T_1)$, we have 
\begin{align*}
    \|\Delta_t-\Delta_{T_0}\| &\le \sum_{t=T_0}^{T_1-1}2(\eta^2 \cdot \|F_tG_t^\top -\Sigma\|^2 \cdot \max\{\|F_t\|,\|G_t\|\}^2)\\
    &\le 2\eta^2 \|F_{T_0}G_{T_0}-\Sigma\|^2\sum_{t=T_0}^\infty\left(1-\frac{\eta\sigma_r^2}{16\sigma_1}\right)^{2(t-T_0)}\cdot 4\sigma_1\\
    &\le 2\eta^2 \cdot 25\sigma_1^2 \cdot \frac{16\sigma_1}{\eta\sigma_r^2}\cdot 4 \sigma_1\\
    &\le 3200\eta \kappa^2\sigma_1^2\\
    &\le \alpha^2/8,
\end{align*}
where the last inequality arises from the fact that $\eta = \cO(\alpha^2/\kappa^2\sigma_1^2)$. By 
$\frac{3\alpha^2}{8}I\le \Delta_{T_0} \le \frac{13\alpha^2}{8}I$,
we can have  $\|\Delta_t\| \le 13\alpha^2/8 + \alpha^2/8 \le 7\alpha^2/4$ and $\lambda_{\min}(\Delta_t) \ge 3\alpha^2/8-\alpha^2/8 = \alpha^2/4.$ Hence, the inequality Eq.\eqref{assy:phase 1 condition 3}
 still holds during Phase 1.
 Moreover, by Eq.\eqref{phase 1: linear convergence}, during the Phase 1, 
 for a round $t \ge 0$, we will have
\begin{align}
    \|F_{t+T_0}G_{t+T_0}^\top -\Sigma\|&\le    \Vert F_{T_0}G_{T_0}-\Sigma\Vert \left(1-\eta\sigma_r^2/64\sigma_1\right)^{t}/(1088\kappa^2\delta_{2k+1}+64 \alpha^2\kappa/\sigma_r)\nonumber\\
    &\le \Vert F_{T_0}G_{T_0}-\Sigma\Vert\left(1-\eta\sigma_r^2/64\sigma_1\right)^{t}\cdot \frac{\sigma_r}{64\alpha^2\kappa}\nonumber\\
    &\le \frac{\sigma_r}{2}\cdot \left(1-\eta\sigma_r^2/64\sigma_1\right)^{t}\cdot \frac{\sigma_r}{64 \alpha^2\kappa} \nonumber\\
    &=\frac{\sigma_r^2}{128\alpha^2\kappa}\left(1-\eta\sigma_r^2/64\sigma_1\right)^{t}.\label{phase 1 loss}
 \end{align}
 The conclusion \eqref{phase 1 loss} always holds in Phase 1. Note that Phase 1 may not terminate, and then the loss is linear convergence. We assume that at round $T_1$, Phase 1 terminates, which implies that 
 \begin{align}
     \sigma_r^2 M_{T_1-1}/64\sigma_1 < (17\sigma_1\delta_{2k+1}+\alpha^2)\|J_{T_1-1}K_{T_1-1}^\top \|,
 \end{align}
 and the algorithm goes to Phase 2.  
\subsection{Phase 2:  Adjustment Phase.} \label{sec: adjustment}

In this phase, we prove $U-V$ will decrease exponentially. This phase terminates at the first time $T_2$ such that  \begin{align}\Vert U_{T_2-1}-V_{T_2-1}\Vert \le \frac{8\alpha^2\sqrt{\sigma_1}+64\delta_{2k+1} \sqrt{\sigma_1}\Vert J_{T_2-1}K_{T_2-1}^\top \Vert }{\sigma_r}.\label{phase 2 stop rule}\end{align}
By stopping rule \eqref{phase 2 stop rule}, since $\|U_{T_1}-V_{T_1}\| \le \cO(\sigma_1)$, this phase will take at most $\cO(\log(\sqrt{\sigma_r}/\alpha)/\eta \sigma_r)$ rounds, i.e. \begin{align}T_2-T_1 = \cO(\log(\sqrt{\sigma_r}/\alpha)/\eta \sigma_r).\label{T2-T1 upper bound}\end{align}
We use the induction to show that all the following hypotheses hold during Phase 2.
\begin{gather}
\max\{\|F_{t-1}\|,\|G_{t-1}\} \le 2\sqrt{\sigma_1}\label{phase 2 FG trivial}\\
    M_t\le (1088\kappa^2\delta_{2k+1}+64\alpha^2\kappa/\sigma_r)\Vert J_tK_t^\top \Vert \le \Vert J_tK_t^\top \Vert \label{phase 2 M_t bound}\\
    \small{
    \max\{\Vert J_{t-1}\Vert, \|K_{t-1}\|\}\le2\sqrt{\alpha}\sigma_1^{1/4}+ (2C_2+16C_3)\log(\sqrt{\sigma_1}/n\alpha) (\delta_{2k+1} \cdot \kappa^2\sqrt{\sigma_1})\le \sigma_r/4\sqrt{\sigma_1}\label{ineq: phase 2: JK bound}}
    \\
    \|J_tK_t^\top \| \le \left(1+\frac{\eta \sigma_r^2}{128\sigma_1}\right)\|J_{t-1}K_{t-1}^\top \|\label{ineq: JtKt update phase 2}    \\\|U_t-V_t\| \le (1-\eta\sigma_r/2)\|U_{t-1}-V_{t-1}\|\label{ineq: phase 2 U-V decrease}\\
    \frac{3\alpha^2}{16}\cdot I \le \Delta_t \le \frac{29\alpha^2}{16}\cdot I . \label{delta t adjustment}
\end{gather}
\paragraph{Proof of \eqref{ineq: phase 2: JK bound}}
To prove this, we first assume that this adjustment phase will only take at most $C_3(\log (\alpha)/\eta\sigma_r)$ rounds. By the induction hypothesis for the previous rounds, \begin{small}\begin{align*}\Vert J_{t}\Vert &\le J_{T_1}+\sum_{i=T_1}^{t-1} \eta \delta_{2k+1}\cdot \|F_tG_t^\top -\Sigma\|\\&\le 2\sqrt{\alpha}\sigma_1^{1/4}+2C_2\log(\sqrt{\sigma_1}/n\alpha) (\delta_{2k+1} \cdot \sigma_1^{3/2}/\sigma_r)  + \sum_{i=T_1}^{t-1} \eta \delta_{2k+1}\cdot \|F_iG_i^\top -\Sigma\|\\&\le2\sqrt{\alpha}\sigma_1^{1/4}+2C_2\log(\sqrt{\sigma_1}/n\alpha) (\delta_{2k+1} \cdot \sigma_1^{3/2}/\sigma_r)  +  C_3(\log(\sqrt{\sigma_1}/n\alpha) / \eta\sigma_r) \cdot \eta \delta_{2k+1}\cdot 4 \Vert J_{i-1}K_{i-1}^\top \Vert\\&\le 2\sqrt{\alpha}\sigma_1^{1/4}+2C_2\log(\sqrt{\sigma_1}/n\alpha) (\delta_{2k+1} \cdot \sigma_1^{3/2}/\sigma_r)  +  C_3(\log(\sqrt{\sigma_1}/n\alpha) / \eta\sigma_r) \cdot \eta \delta_{2k+1} 16\sigma_1\\&\le 2\sqrt{\alpha}\sigma_1^{1/4}+(2C_2+16C_3)\log(\sqrt{\sigma_1}/n\alpha) (\delta_{2k+1} \cdot \sigma_1^{3/2}/\sigma_r).\end{align*}\end{small}
Similarly, due to the symmetry property, we can bound the $\|K_t\|$ using the same technique. Thus, \begin{align*}
    \max\{\|J_t\|,\|K_t\|\} \le 2\sqrt{\alpha}\sigma_1^{1/4}+(2C_2+16C_3)\log(\sqrt{\sigma_1}/n\alpha) (\delta_{2k+1} \cdot \sigma_1^{3/2}/\sigma_r).
\end{align*} 
\paragraph{Proof of \eqref{phase 2 M_t bound}}
First, we prove that during $t \in [T_1,T_2)$, \begin{align}M_t\le (1088\kappa^2\delta_{2k+1}+64\alpha^2\kappa/\sigma_r)\Vert J_tK_t^\top \Vert \le \Vert J_tK_t^\top \Vert \le 4\alpha \kappa^4\sigma_1^{1/2}+\delta_{2k+1}\sigma_1. \end{align} in this phase.  

Then, by $\delta_{2k+1} \le \cO(1/\log(\sqrt{\sigma_1}/n\alpha)\kappa^{2})$ and $\alpha \le \cO(\sigma_r/\sqrt{\sigma_1})$, choosing sufficiently small coefficient, we can have 
\begin{align}
    J_tK_t^\top &= (I-\eta J_{t-1}J_{t-1}^\top )J_{t-1}K_{t-1}^\top (I-\eta K_{t-1}K_{t-1}^\top) + \eta^2 J_{t-1}J_{t-1}^\top J_{t-1}K_{t-1}^\top K_{t-1}K_{t-1}^\top \nonumber\\&\quad - \eta J_{t-1}V_{t-1}^\top V_{t-1}K_{t-1}^\top - \eta J_tU_t^\top U_tK_t^\top +C_{t-1},\label{eq: Jk important bund}
\end{align}
where $C_t$ represents the relatively small perturbation term, which contains terms of  $\cO(\delta)$ and $\cO(\eta^2)$. By \eqref{phase 2 FG trivial}, we can easily get 
\begin{align}
    C_{t-1} \ge -\left(4\eta \delta_{2k+1}\cdot \|F_{t-1}G_{t-1}^\top - \Sigma\|\cdot 4\sigma_1 \right)\label{per ct bound}
\end{align}
Thus, combining \eqref{eq: Jk important bund} and \eqref{per ct bound}, we have
\begin{small}
\begin{align*}
    &\Vert J_tK_t^\top \Vert \\&\ge \Vert I-\eta J_{t-1}J_{t-1}^\top \Vert \Vert I-\eta K_{t-1}K_{t-1}^\top \Vert \Vert J_{t-1}K_{t-1}^\top \Vert -4\eta M_{t-1}\cdot 4\sigma_1\label{JK update essential equation}\\& \quad -4\eta\delta_{2k+1} \Vert J_{t-1}K_{t-1}\Vert\cdot 2\sigma_1  - \eta^2 64\sigma_1^3
    \\&\ge \left(1-2\eta \max\{\|J_{t-1}\|,\|K_{t-1}\|\}^2-16\cdot 1088\eta \kappa^2\delta_{2k+1}\sigma_1-1024\eta \alpha^2\kappa^2-8\eta \delta_{2k+1}\cdot \sigma_1\right)\|J_{t-1}K_{t-1}^\top \|\\&\ge \left(1-\frac{\eta\sigma_r^2}{128\sigma_1}\right) \Vert J_{t-1}K_{t-1}^\top \Vert. 
\end{align*}
\end{small}
The second inequality is because $M_{t-1} \le (1088\kappa^2\delta_{2k+1}+64\alpha^2\kappa/\sigma_r)\|J_{t-1}K_{t-1}^\top \|$, and  
the last inequality holds by Eq.\eqref{ineq: phase 2: JK bound} and \begin{align}
    \delta_{2k+1}  = \cO(\kappa^{-4}), \alpha = \cO(\kappa^{-3/2}\sqrt{\sigma_r})
\end{align}
Then,  
note that by Eq.\eqref{ineq: M recursion}, we have
\begin{align*}
    M_t\le \left(1-\frac{\eta\sigma_r^2}{32\sigma_1}\right)M_{t-1} +(17\eta \sigma_1\delta_{2k+1}+\eta\alpha^2)\|J_{t-1}K_{t-1}^\top \|.
\end{align*}
Then, by $M_{t-1} \le (1088\kappa^2\delta_{2k+1}+64\alpha^2\kappa/\sigma_r)\cdot \|J_{t-1}K_{t-1}^\top \|$ and denote $L=17\sigma_1\delta_{2k+1}+\alpha^2$, we have
\begin{align*}
    M_t&\le \left(1-\frac{\eta\sigma_r^2}{32\sigma_1}\right)M_{t-1} + (17\eta \sigma_1\delta_{2k+1}+\eta\alpha^2)\|J_{t-1}K_{t-1}^\top \|\\&\le 
    \left(1-\frac{\eta\sigma_r^2}{32\sigma_1}\right)\cdot (1088\kappa^2\delta_{2k+1}+64\alpha^2\kappa/\sigma_r)\Vert J_{t-1}K_{t-1}^\top \Vert + \eta L\|J_{t-1}K_{t-1}^\top \|\\
    &= \left(1-\frac{\eta\sigma_r^2}{32\sigma_1}\right)\cdot \frac{64L\kappa}{\sigma_r}\Vert J_{t-1}K_{t-1}^\top \Vert + \eta L\|J_{t-1}K_{t-1}^\top \|\\
    &\le \left(\frac{64L\kappa}{\sigma_r}-2\eta L\right)\|J_{t-1}K_{t-1}^\top \|\\
    &\le \left(\frac{64L\kappa}{\sigma_r}- 2\eta L\right)\Big/\left(1-\frac{\eta\sigma_r^2}{128\sigma_1}\right) \Vert J_tK_t^\top \Vert \\&\le \frac{64L\kappa}{\sigma_r}\Vert J_tK_t^\top \Vert.
\end{align*}
Hence, $$M_t \le \frac{64L\kappa}{\sigma_r}\|J_tK_t^\top \| \le \|J_tK_t^\top \|$$ for all $t$ in Phase 2. The last inequality is because $\delta_{2k+1} = \cO(1/\kappa^2\log(\sqrt{\sigma_1}/n\alpha))$. 
Moreover, by $\delta_{2k+1} \le \cO(1/\kappa^2\log(\sqrt{\sigma_1}/n\alpha)^2)$ and $(a+b)^2 \le 2a^2 + 2b^2$ we have 
\begin{align}
    \|J_tK_t^\top \| \le \|J_t\|\|K_t\| &\le \Big(2\sqrt{\alpha}\sigma_1^{1/4}+(2C_2+16C_3)\log(\sqrt{\sigma_1}/n\alpha) (\delta_{2k+1} \cdot \kappa^2\sqrt{\sigma_1})\Big)^2\\&\le 
    4\alpha \kappa^4\sigma_1^{1/2}+\delta_{2k+1}\sigma_1.
\end{align}
We complete the proof of Eq.\eqref{phase 2 M_t bound}.
\paragraph{Proof of Eq.\eqref{ineq: JtKt update phase 2}}
Moreover, by the updating rule of $J_t$ and $K_t$, \eqref{eq: Jk important bund} and \eqref{per ct bound} we have 
 \begin{align}
     &\|J_tK_t^\top\| \nonumber\\&\quad \le \|(I-\eta J_{t-1}J_{t-1}^\top )J_{t-1}K_{t-1}^T (I - \eta K_{t-1}K_{t-1}^\top)\| + \|\eta^2 (J_{t-1}J_{t-1}^\top )J_{t-1}K_{t-1}^T(K_{t-1}K_{t-1}^\top)\| \label{JK update important}\\&\quad \qquad + 4\eta M_{t-1} \cdot 4\sigma_1 + 4\eta \delta_{2k+1}\|J_{t-1}K_{t-1}^\top \| \cdot 2\sigma_1\nonumber\\
     &\quad \le \|J_{t-1}K_{t-1}^\top \| + \eta^2 (\sqrt{\sigma_1}/2)^4 \|J_{t-1}K_{t-1}^\top \| + 4\eta \frac{64L\kappa}{\sigma_r}\|J_{t-1}K_{t-1}^\top \|\cdot 4\sigma_1 + 8\eta\sigma_1 \delta_{2k+1}\|J_{t-1}K_{t-1}^\top \|\nonumber\\
     &\quad = \|J_{t-1}K_{t-1}^\top \|\cdot \left(1 + \eta^2 \sigma_1^2/16 + 1024L\kappa^2 + 8\sigma_1\delta_{2k+1}\right).\nonumber
 \end{align}
 The last inequality uses the fact that $\|J_{t-1}\| \le \sqrt{\sigma_1}/2, \|K_{t-1}\| \le \sqrt{\sigma_1}/2$ and $M_{t-1} \le \frac{64L\kappa}{\sigma_r}\|J_{t-1}K_{t-1}^\top \|.$
 Now by the fact that $L = 17\sigma_1\delta_{2k+1}+\alpha^2 =\cO(\frac{\sigma_r^2}{\sigma_1\kappa^2})$, we can choose small constant so that 
 \begin{align}
     \eta^2 \sigma_1^2/16 \le \frac{\sigma_r^2}{384\sigma_1}, \quad 1024 L \kappa^2 \le \frac{\sigma_r^2}{384\sigma_1}, \quad 8\sigma_1\delta_{2k+1} \le \frac{\sigma_r^2}{384\sigma_1}.\nonumber
 \end{align}
 Thus, we can have 
 \begin{align}
     \|J_tK_t^\top \| \le \|J_{t-1}K_{t-1}^\top \| \cdot \left(1 + \frac{\eta\sigma_r^2}{128\sigma_1}\right).\nonumber
 \end{align}
 We complete the proof of \eqref{ineq: JtKt update phase 2}
\paragraph{Proof of \eqref{ineq: phase 2 U-V decrease}}
Hence, similar to Phase 1, by $\|U_tV_t^\top -\Sigma\| \le M_t \le 4\alpha \kappa^4\sigma_1^{1/2}+\delta_{2k+1}\sigma_1$ and $\|U_t-V_t\| \le \|U_{T_1}-V_{T_1}\| \le 4\alpha+\frac{40\delta\sigma_1^{3/2}}{\sigma_r}$, we can show that 
\begin{align}
    \max\{\|U_t\|,\|V_t\|\} \le 2\sqrt{\sigma_1}\nonumber
\end{align}

Also, consider  
\begin{align*}
    &U_t-V_t \\&= (I-\eta \Sigma-V_t^\top V_t - K_t^\top K_t )(U_{t-1}-V_{t-1})  - \eta V_t \Delta_t\\
    &\quad  + \eta \cdot \left(E_1(F_{t-1}G_{t-1}^\top - \Sigma) V_{t-1} + E_2(F_{t-1}G_{t-1}^\top - \Sigma)K_{t-1} \right) \\&\quad \quad -\eta\cdot \left( E_1^\top (F_{t-1}G_{t-1}^\top - \Sigma) U_{t-1} + E_3^\top (F_{t-1}G_{t-1}^\top - \Sigma)J_{t-1}\right).
\end{align*}
Hence, by the RIP property and $\Delta_{t-1} \le 2\alpha^2 I$ (\eqref{delta t adjustment}), we can get
\begin{align*}
    \Vert (U_t-V_t)\Vert  
    &\le (1-\eta \sigma_r)\Vert U_{t-1}-V_{t-1}\Vert  + 2\eta\alpha^2 \cdot 2\sqrt{\sigma_1} + 4\eta\delta_{2k+1} \cdot  2\sqrt{\sigma_1}\cdot \Vert F_{t-1}G_{t-1}^\top -\Sigma \Vert\\
    &\le (1-\eta \sigma_r)\Vert U_{t-1}-V_{t-1}\Vert  + 2\eta\alpha^2 \cdot 2\sqrt{\sigma_1}+ 8\eta \delta_{2k+1} \cdot  \sqrt{\sigma_1}\cdot 4\Vert J_{t-1}K_{t-1}^\top  \Vert\\
    &\le (1-\eta \sigma_r)\Vert U_{t-1}-V_{t-1}\Vert  + 2\eta\alpha^2 \cdot 2\sqrt{\sigma_1}+ 32\eta \delta_{2k+1} \cdot  \sqrt{\sigma_1}\cdot \Vert J_{t-1}K_{t-1}^\top  \Vert
\end{align*}
Since 
$$\Vert U_{t-1}-V_{t-1}\Vert \ge \frac{8\alpha^2\sqrt{\sigma_1}+64\delta_{2k+1} \sqrt{\sigma_1}\Vert J_{t-1}K_{t-1}^\top \Vert }{\sigma_r}.$$
 for all $t$ in Phase 2, 
we can have   
\begin{align*}\|U_t-V_t\| \le (1-\eta\sigma_r/2)\|U_{t-1}-V_{t-1}\|\end{align*} during Phase 2.

Moreover, since Phase 2 terminates at round $T_2$, such that 
$$\Vert U_{T_2-1}-V_{T_2-1}\Vert \le \frac{8\alpha^2\sqrt{\sigma_1}+64\delta_{2k+1} \sqrt{\sigma_1}\Vert J_{T_2-1}K_{T_2-1}^\top \Vert }{\sigma_r},$$ it takes at most \begin{align}C_3\log(\sqrt{\sigma_r}/\alpha)/\eta\sigma_r=t_2^*\label{phase 2: time bound}\end{align} rounds for some constant $C_3$ because (a)  \eqref{ineq: phase 2 U-V decrease}, (b) and $U_t-V_t$ decreases from $\Vert U_{T_1}-V_{T_1}\Vert \le 4\sqrt{\sigma_1}$ to at most  $\|U_{T_2}-V_{T_2}\|  = \Omega(\alpha^2\sqrt{\sigma_1}/\sigma_r)$. Also, the changement of $\Delta_t$ can be bounded by 
\begin{align*}
    \|\Delta_t-\Delta_{T_1}\| &\le \sum_{t=T_1}^{T_2-1} 2(\eta^2 \cdot \|F_tG_t^\top -\Sigma\|^2 \cdot 4\sigma_1)\\&\le 2(\eta^2) \cdot 100\sigma_1^3 \cdot (T_2-T_1)\\
    &\le 2(\eta^2) \cdot 100\sigma_1^3\cdot C_3\log(\sqrt{\sigma_1}/n\alpha)(1/\eta \sigma_r)\\
    &\le 10C_3\log(\sqrt{\sigma_1}/n\alpha) (\eta \kappa\sigma_1^2)\\
    &\le \alpha^2/16.
\end{align*}
The last inequality holds by choosing $\eta \le \alpha^2/160C_3\kappa\sigma_1^2$. Then, $\lambda_{\min}(\Delta_{t}) \ge \lambda_{\min}\Delta_{T_1}-\alpha^2/16 \ge \alpha^2/4 - \alpha^2/16 = 3\alpha^2/16$ and $\|\Delta_t\| \le \|\Delta_{T_1}\|  + \alpha^2/16 \le 7\alpha^2/4 + \alpha^2/16 \le 29\alpha^2/16$. Hence, inequality \eqref{eq:keyproperty}
 still holds during Phase 2.

\subsection{Phase 3: local convergence}\label{sec: local cvg gd linear} In this phase, we show that the norm of $K_t$ will decrease at a linear rate. Denote the SVD of $U_t$ as $U_t = A_t\Sigma_tW_t$, where $\Sigma_t \in \RR^{r\times r}$, $W_t \in \RR^{r\times k}$, and define $W_{t,\perp} \in \RR^{(k-r)\times k}$ is the complement of $W_t$.

We use the induction to show that all the following hypotheses hold during Phase 3.
\begin{gather}
    \max\{\|J_t\|,\|K_t\|\} \le \cO(2\sqrt{\alpha}\sigma_1^{1/4}+\delta_{2k+1} \log(\sqrt{\sigma_1}/n\alpha) \cdot \kappa^2\sqrt{\sigma_1})\le \sqrt{\sigma_1}/2\label{phase 3: J_tK_t bound}\\
    M_t\le \frac{64L\kappa}{\sigma_r}\|J_tK_t^\top \| \le \|J_tK_t^\top \|\label{ineq:M_t bound}\\
    \|J_tK_t^\top \| \le \left(1+\frac{\eta \sigma_r^2}{128\sigma_1}\right)\|J_{t-1}K_{t-1}^\top \|\label{ineq: JtKt update phase 3}\\
    \Vert U_t-V_t\Vert \le \frac{8\alpha^2\sqrt{\sigma_1}+64\delta_{2k+1} \sqrt{\sigma_1}\Vert J_{t}K_{t}^\top \Vert }{\sigma_r}\label{ineq:U-V bound}\\
    \frac{\alpha^2}{8}\cdot I\le \Delta_t \le 2\alpha^2 I \label{Phase 3:keyproperty not change}\\
    \|K_t\|\le 2\|K_tW_{t,\perp}^\top \|\label{ineq:K_t part bound.}\\ \|K_{t+1}W_{t+1,\perp}^\top \|\le \|K_tW_{t,\perp}^\top \| \cdot \left(1-\frac{\eta\alpha^2}{8}\right)\label{ineq:K_t transform}.
\end{gather}
Assume the hypotheses above hold before round $t$, then at round $t$, by the same argument in Phase 1 and 2, the inequalities \eqref{ineq:M_t bound} and \eqref{ineq:U-V bound} still holds, then $\max\{\|U_t\|,\|V_t\|\} \le 2\sqrt{\sigma_1}$ and $\min\{\sigma_r(U),\sigma_r(V)\} \ge \sigma_r/4\sqrt{\sigma_1}$.

Last, we should prove the induction hypotheses \eqref{phase 3: J_tK_t bound}
 , \eqref{Phase 3:keyproperty not change}, \eqref{ineq:K_t part bound.} and \eqref{ineq:K_t transform}. 
 \paragraph{Proof of Eq.\eqref{ineq: JtKt update phase 3}} 
 Similar to the proof of \eqref{ineq: JtKt update phase 2} in Phase 2, we can derive \eqref{ineq: JtKt update phase 3} again.
 \paragraph{Proof of Eq.\eqref{ineq:K_t part bound.}} 
 First, to prove \eqref{ineq:K_t part bound.}, note that we can get 
\begin{align*}
    M_t\ge \Vert U_tK_t\Vert  = \Vert A_t\Sigma_tW_tK_t^\top  \Vert &= \Vert \Sigma_tW_tK_t^\top \Vert \\&\ge \sigma_r(U)\cdot \Vert K_tW_t^\top \Vert \ge \frac{\|K_tW_t^\top \|\sigma_r}{4\sqrt{\sigma_1}}\ge \frac{\Vert K_tW_t^\top \Vert\sqrt{\sigma_r}}{4\sqrt{\kappa}}.
\end{align*}
Hence, \begin{align}\Vert K_tW_t^\top \Vert \le 4\sqrt{\kappa} M/\sqrt{\sigma_r}\le \frac{64\sigma_1L \sqrt{\kappa} }{\sigma_r^{5/2}}\Vert J_tK_t^\top \Vert \le \frac{32L \kappa^{3/2} }{\sigma_r^{3/2}}\Vert K_t\Vert\cdot \sqrt{\sigma_1} \le \frac{32L \kappa^{2} }{\sigma_r}\Vert K_t\Vert.\label{ineq:KtWt and JtKt}\end{align} Thus, 
\begin{align*}
    \Vert K_t\Vert &\le \Vert K_tW_{t,\perp}^\top \Vert+ \Vert K_tW_t^\top \Vert\\&\le \Vert K_tW_{t,\perp}^\top \Vert + \frac{64L\kappa}{\sigma_r}\Vert K_t\Vert\\&\le \Vert K_tW_{t,\perp}^\top \Vert + \frac{1}{2}\Vert K_t\Vert.
\end{align*}
The last inequality uses the fact that $\delta_{2k+1} = \cO(\sigma_r^3/\sigma_1^3)$
Hence, $\Vert K_tW_{t,\perp}^\top \Vert \ge \Vert K_t\Vert/2$, and \eqref{ineq:K_t part bound.}
 holds during Phase 3.
 \paragraph{Proof of Eq.\eqref{Phase 3:keyproperty not change}}
To prove the \eqref{Phase 3:keyproperty not change}, by the induction hypothesis of Eq.\eqref{ineq:K_t transform}, note that \begin{align}\|\Delta_t-\Delta_{T_2}\|& \le 2\eta^2 \cdot \sum_{t'=T_2}^{t-1} \|F_{t'}G_{t'}^\top -\Sigma\|^2 4\sigma_1\nonumber\\&\le 2\eta^2  \sum_{t'=T_2}^{t-1} 16\sigma_1\|J_{t'}K_{t'}^\top \|^2\nonumber\\
&\le 64\sigma_1 \eta^2  \cdot \sum_{t'=T_2}^{\infty} \|J_{t'}\|^2\|K_{t'}W_{t',\perp}^\top \|^2\nonumber\\
&\le 64\sigma_1\cdot \eta^2 \left( \sigma_1 \cdot \|K_{T_2}W_{T_2,\perp}^\top \|^2\cdot \frac{8}{\eta\alpha^2}\right) \label{ineq: geometry series}\\&\le  \frac{512\eta\sigma_1^2}{\alpha^2}\cdot \|K_{T_2}\|^2\nonumber\\& \le \frac{128\eta\sigma_1^2}{\alpha^2}\cdot \sigma_1\nonumber\\&\le \alpha^2/16.\nonumber\end{align}
The Eq.\eqref{ineq: geometry series} holds by the sum of geometric series. The last inequality holds by $\eta \le \cO(\alpha^4/\sigma_1^3)$
Then, we have 
\begin{align*}
    \|\Delta_t\| &\le \|\Delta_{T_2}\| + \|\Delta_t-\Delta_{T_2}\| \le \frac{29\alpha^2}{16} + \frac{\alpha^2}{16} \le 2\alpha^2.\nonumber\\
    \lambda_{\min}(\Delta_t) &\ge \lambda_{\min}(\Delta_{T_2}) - \|\Delta_t-\Delta_{T_2}\| \ge \frac{3\alpha^2}{16}-\frac{\alpha^2}{16} = \frac{\alpha^2}{8}.\nonumber
\end{align*}
Hence, \eqref{Phase 3:keyproperty not change} holds during Phase 3.
\paragraph{Proof of Eq.\eqref{phase 3: J_tK_t bound}}
  To prove the \eqref{phase 3: J_tK_t bound}, note that 

 \begin{align}
     \|K_t\| \le 2\|K_tW_{t,\perp}^\top \| \le 2\|K_{T_2}W_{T_2,\perp}^\top \| \le 2\|K_{T_2}\| \le \cO(\delta_{2k+1} \log(\sqrt{\sigma_1}/n\alpha)\cdot \sigma_1^{3/2}/\sigma_r).\label{K_t:upper bound}
 \end{align}
 On the other hand, by $\Delta_t \le 2\alpha^2 I $, we have 
 \begin{align*}
    W_{t,\perp}J_t^\top J_tW_{t,\perp}^\top  - W_{t,\perp}K_t^\top K_tW_{t,\perp}^\top  - W_{t,\perp}V_t^\top V_tW_{t,\perp}^\top  \le 2\alpha^2\cdot  I.
\end{align*} Hence, denote $L_t = \|J_tK_t^\top \|\le \sigma_1/4,$
 \begin{align}
    &W_{t,\perp}J_t^\top J_tW_{t,\perp}^\top \le 2\alpha^2 I + W_{t,\perp}K_t^\top K_tW_{t,\perp}^\top  + W_{t,\perp}V_t^\top V_tW_{t,\perp}^\top  \nonumber\\
    &=2\alpha^2 I + W_{t,\perp}K_t^\top K_tW_{t,\perp}^\top  + W_{t,\perp}(V_t-U_t)^\top (V_t-U_t)W_{t,\perp}^\top \nonumber\\
    &\le 2\alpha^2 I + W_{t,\perp}K_t^\top K_tW_{t,\perp}^\top  + \left(\frac{8\alpha^2\sqrt{\sigma_1}+64\delta_{2k+1} \sqrt{\sigma_1}L_t}{\sigma_r}\right)^2\cdot I\nonumber\\&= W_{t,\perp}K_t^\top K_tW_{t,\perp}^\top  + \left(2\alpha + \frac{8\alpha^2\sqrt{\sigma_1}+64\delta_{2k+1} \sqrt{\sigma_1}L_t}{\sigma_r}\right)^2I \label{ineq:J_tW_(t,perp) upper bound}.
\end{align}
 Also, by inequality \eqref{ineq:J_tW_(t,perp) upper bound}, we have 
 \begin{align}
     \|J_tW_{t,\perp}^\top  \|-\| K_tW_{t,\perp}^\top \| &\le \frac{\|J_tW_{t,\perp}^\top \|^2 - \|K_tW_{t,\perp}^\top \|^2}{\|J_tW_{t,\perp}^\top  \|+\| K_tW_{t,\perp}^\top \|}\nonumber\\
     &\le \frac{\left(2\alpha + \frac{8\alpha^2\sqrt{\sigma_1}+64\delta_{2k+1} \sqrt{\sigma_1}L_t}{\sigma_r}\right)^2}{2\|K_tW_{t,\perp}^\top \| + \|J_tW_{t,\perp}^\top  \|-\| K_tW_{t,\perp}^\top \|}\nonumber\\
     &\le \frac{\left(2\alpha + \frac{8\alpha^2\sqrt{\sigma_1}+64\delta_{2k+1} \sqrt{\sigma_1}L_t}{\sigma_r}\right)^2}{ \|J_tW_{t,\perp}^\top  \|-\| K_tW_{t,\perp}^\top \|}\nonumber
 \end{align}
 Thus, by $L_t\le \sigma_1/4$, we can get 
 \begin{align}
     \|J_tW_{t,\perp}^\top  \| &\le \| K_tW_{t,\perp}^\top \|+2\alpha + \frac{8\alpha^2\sqrt{\sigma_1}+64\delta_{2k+1} \sqrt{\sigma_1}L_t}{\sigma_r}\nonumber\\
     &\le \|K_{T_2}\| + 2\alpha + \frac{8\alpha^2\sqrt{\sigma_1}+64\delta_{2k+1} \sqrt{\sigma_1}L_t}{\sigma_r}\nonumber\\
     &\le\cO(2\sqrt{\alpha}\sigma_1^{1/4}+\delta_{2k+1}\log(\sqrt{\sigma_1}/n\alpha)\kappa^2\sqrt{\sigma_1}).\nonumber
 \end{align}
 The second inequality holds by $\|K_tW_{t,\perp}^\top \| \le \|K_{T_2}W_{T_2,\perp}^\top \| \le \|K_{T_2}\|.$
 On the other hand, note that 
 \begin{align}
     \|J_t\| &\le \|J_tW_{t}^\top \| + \|J_tW_{t,\perp}^\top \|\nonumber\\&\le \|J_tU_t^\top \|/\sigma_r(U) + \|J_tW_{t,\perp}^\top \|\nonumber\\&\le \|J_tV_t\|/\sigma_r(U) + \|J_t(U_t-V_t)\|/\sigma_r(U) + \|J_tW_{t,\perp}^\top \|\nonumber\\
     &\le M_t/\sigma_r(U) + \|J_t\|\|(U_t-V_t)\|/\sigma_r(U) + \|J_tW_{t,\perp}^\top \|\nonumber\\
     &\le \frac{64L\kappa}{\sigma_r}\|J_t\|\|K_t\|\cdot \frac{4\sqrt{\sigma_1}}{\sigma_r} + \|J_t\|\frac{8\alpha^2\sqrt{\sigma_1} + 64\delta_{2k+1}\sqrt{\sigma_1}\|J_tK_t^\top \|}{\sigma_r}\cdot \frac{4\sqrt{\sigma_1}}{\sigma_r} + \|J_tW_{t,\perp}^\top \|\nonumber\\&\le \left(\frac{64\sigma_1^{3/2}L}{\sigma_r^3}\cdot \sqrt{\sigma_1}+\frac{32\alpha^2\sigma_1 + 256\delta_{2k+1}\sigma_1\cdot \sigma_1}{\sigma_r^2}\right) \|J_t\| + \|J_tW_{t,\perp}^\top \|\nonumber\\
     &\le \frac{1}{2}\|J_t\| + \|J_tW_{t,\perp}^\top \|.\label{ineq:J_t bound}
 \end{align}
 The last inequality holds because 
\begin{align*}
    \delta_{2k+1} = \cO(\kappa^{-4}\log^{-1}(\sqrt{\sigma_1}/n\alpha)),\ \ \  \alpha \le \cO(\sigma_r/\sqrt{\sigma_1}) 
\end{align*}

 Hence, by the inequality \eqref{ineq:J_t bound}, we can get 
 \begin{align}
     \|J_t\|\le 2\|J_tW_{t,\perp}^\top \| = \cO(2\sqrt{\alpha}\sigma_1^{1/4}+\delta_{2k+1}\log(\sqrt{\sigma_1}/n\alpha)\cdot \kappa^2\sqrt{\sigma_1}).
 \end{align}
 Thus, \eqref{phase 3: J_tK_t bound} holds during Phase 3.

\paragraph{Proof of Eq.\eqref{ineq:K_t transform}}
Now we prove the inequality \eqref{ineq:K_t transform}. We consider 
the changement of $K_t$. We have 
\begin{align*}
    K_{t+1} =  K_t(I-U_t^\top U_t-J_t^\top J_t) + E_3(F_tG_t^\top -\Sigma)U_t + E_4(F_tG_t^\top -\Sigma)J_t  
\end{align*}

Now consider $K_{t+1}W_{t,\perp}^\top $, we can get 
\begin{align*}
    K_{t+1}W_{t,\perp}^\top  &= K_t(I-\eta W_t^\top \Sigma^2W_t-J_t^\top J_t)W_{t,\perp}^\top  + \eta E_3(F_tG_t^\top -\Sigma)U_tW_{t,\perp}^\top  + \eta E_4(F_tG_t^\top -\Sigma)J_tW_{t,\perp}^\top \\
    &= K_tW_{t,\perp}^\top  - \eta K_tJ_t^\top J_tW_{t,\perp}^\top  + \eta E_4(F_tG_t^\top -\Sigma)J_tW_{t,\perp}^\top \\
    &=K_tW_{t,\perp}^\top  - \eta K_tW_{t,\perp}^\top W_{t,\perp}J_t^\top J_tW_{t,\perp}^\top  - \eta K_tW_{t}^\top W_{t}J_t^\top J_tW_{t,\perp}^\top  + \eta E_4(F_tG_t^\top -\Sigma)J_tW_{t,\perp}^\top 
\end{align*}
Hence, by the Eq.\eqref{ineq:KtWt and JtKt},
\begin{align*}
    \Vert K_{t+1}W_{t,\perp}^\top  \Vert&\le  \Vert K_tW_{t,\perp}^\top  (I - \eta W_{t,\perp}J_t^\top J_tW_{t,\perp}^\top )\Vert +\frac{64\eta L\kappa^{3/2}}{\sigma_r^{3/2}}\Vert J_tK_t^\top \Vert \cdot \Vert J_tW_{t,\perp}^\top \Vert\Vert J_t\Vert + 4\eta\delta_{2k+1} M_t \Vert J_tW_{t,\perp}^\top \Vert\\
    &\le \Vert K_tW_{t,\perp}^\top  (I - \eta W_{t,\perp}J_t^\top J_tW_{t,\perp}^\top )\Vert +\frac{64\eta L\kappa^{3/2}}{\sigma_r^{3/2}}\Vert J_tK_t^\top \Vert \cdot \Vert J_tW_{t,\perp}^\top \Vert\Vert J_t\Vert \\&\qquad + \frac{16\sigma_1\eta L}{\sigma_r^2}\|J_tK_t^\top \| \Vert J_tW_{t,\perp}^\top \Vert\\
    &\le \Vert K_tW_{t,\perp}^\top  (I -\eta  W_{t,\perp}J_t^\top J_tW_{t,\perp}^\top )\Vert +\frac{80\eta L\kappa^{2}}{\sigma_r}\Vert J_tK_t^\top \Vert \cdot \Vert J_tW_{t,\perp}^\top \Vert
\end{align*}
The second inequality uses the fact that $\delta_{2k+1} \le 1/16$ and \eqref{ineq:KtWt and JtKt}. 
The last inequality uses the fact that $\|J_t\| \le \sqrt{\sigma_1}.$
Note that $\lambda_{\min}(\Delta_t)\ge \alpha^2/8 \cdot I$, then multiply the $W_{t,\perp}^\top $, we can get 
\begin{align*}
    W_{t,\perp}J_t^\top J_tW_{t,\perp}^\top  - W_{t,\perp}V_t^\top V_tW_{t,\perp}^\top  - W_{t,\perp}K_t^\top K_tW_{t,\perp}^\top  \ge \frac{\alpha^2}{8}\cdot  I.
\end{align*}
Hence, 
\begin{align*}
    W_{t,\perp}J_t^\top J_tW_{t,\perp}^\top - W_{t,\perp}K_t^\top K_tW_{t,\perp}^\top  \ge \frac{\alpha^2}{8}\cdot  I.
\end{align*}
Thus, define $\phi_t = W_{t,\perp}J_t^\top J_tW_{t,\perp}^\top - W_{t,\perp}K_t^\top K_tW_{t,\perp}^\top  $, then we can get 
\begin{align*}
    \Vert K_{t+1}W_{t,\perp}^\top  \Vert&\le\Vert K_tW_{t,\perp}^\top  (I - W_{t,\perp}J_t^\top J_tW_{t,\perp}^\top )\Vert +\frac{80L\kappa^2}{\sigma_r}\Vert J_tK_t^\top \Vert \cdot \Vert J_tW_{t,\perp}^\top \Vert\\
    &\le \Vert K_tW_{t,\perp}^\top  (I - W_{t,\perp}K_t^\top K_tW_{t,\perp}^\top -\eta\phi_t)\Vert +\frac{80L\kappa^2}{\sigma_r}\Vert J_tK_t^\top \Vert \cdot \Vert J_tW_{t,\perp}^\top \Vert
\end{align*}

Define loss $L_t = \Vert J_tK_t^\top \Vert$. Note that 
\begin{align}
    L_t &= \Vert J_tK_t^\top \Vert \nonumber\\&= \Vert J_tW_{t,\perp}^\top  W_{t,\perp}K_t^\top  +  J_tW_{t}^\top  W_{t}K_t^\top \Vert \nonumber\\
    &\le \Vert J_tW_{t,\perp}^\top  W_{t,\perp}K_t^\top \Vert + \Vert J_tW_{t}^\top  W_{t}K_t^\top \Vert\nonumber \\
    &\le \Vert J_tW_{t,\perp}^\top  W_{t,\perp}K_t^\top \Vert + \sqrt{\sigma_1}\cdot \frac{64L\kappa^{3/2}}{\sigma_r^{3/2}}\Vert J_tK_t^\top \Vert\label{ineq:need explain 2} \\
    &\le \Vert J_tW_{t,\perp}^\top  W_{t,\perp}K_t^\top \Vert + \frac{L_t}{2}.\nonumber
\end{align}
The Eq.\eqref{ineq:need explain 2} holds by Eq.\eqref{ineq:KtWt and JtKt} and $\|W_t^\top \| = 1$, and the last inequality holds by $\delta_{2k+1} = \cO(\kappa^4)$.

Hence, \begin{align}\Vert J_tW_{t,\perp}^\top  W_{t,\perp}K_t^\top \Vert \ge L_t/2.\label{eq: JW ge Lt}\end{align}
Similarly, \begin{align}\Vert J_tW_{t,\perp}^\top  W_{t,\perp}K_t^\top \Vert \le 2L_t\label{eq: JW le 2Lt}\end{align}
Then, 
\begin{align*}
    \Vert K_{t+1}W_{t,\perp}^\top  \Vert&\le \Vert K_tW_{t,\perp}^\top  (I - \eta W_{t,\perp}K_t^\top K_tW_{t,\perp}^\top -\eta\phi_t)\Vert +\frac{160\eta L\kappa^2}{\sigma_r}\Vert J_tW_{t,\perp}^\top W_{t,\perp}K_t^\top \Vert \cdot \Vert J_tW_{t,\perp}^\top \Vert.
\end{align*}
If $\Vert J_tW_{t,\perp}^\top \Vert \le 10\kappa\alpha,$ we can get 

\begin{align}
    \Vert K_{t+1}W_{t,\perp}^\top  \Vert&\le \Vert K_tW_{t,\perp}^\top  (I - \eta W_{t,\perp}K_t^\top K_tW_{t,\perp}^\top -\eta\phi_t)\Vert +\frac{160\eta L\kappa^2}{\sigma_r}\Vert J_tW_{t,\perp}^\top W_{t,\perp}K_t^\top \Vert \cdot \Vert J_tW_{t,\perp}^\top \Vert\nonumber\\
    &\le \Vert K_{t}W_{t,\perp}^\top  \Vert \Vert (I -\eta  W_{t,\perp}K_t^\top K_tW_{t,\perp}^\top -\eta\phi_t)\Vert +\frac{160\eta L\kappa^2}{\sigma_r}\Vert J_tW_{t,\perp}^\top W_{t,\perp}K_t^\top \Vert \cdot \Vert J_tW_{t,\perp}^\top \Vert\nonumber\\&\le \Vert K_{t}W_{t,\perp}^\top  \Vert \left(1-\frac{\eta\alpha^2}{8}\right) +\frac{160\eta L\kappa^2}{\sigma_r}\Vert J_tW_{t,\perp}^\top \Vert \Vert W_{t,\perp}K_t^\top \Vert \cdot \Vert J_tW_{t,\perp}^\top \Vert\nonumber\\
    &\le \Vert K_tW_{t,\perp}^\top \Vert \cdot \left(1-\frac{\eta\alpha^2}{8}\right) + \frac{160\eta L\kappa^2}{\sigma_r} 100\kappa^2\alpha^2 \|K_tW_{t,\perp}^\top \|\nonumber\\ 
    &\le \Vert K_tW_{t,\perp}^\top \Vert \cdot \left(1-\frac{\eta\alpha^2}{16}\right)\label{ineq: alpha}\\
    &\le \Vert K_tW_{t,\perp}^\top \Vert \cdot \left(1-\frac{\eta\|J_tW_{t,\perp}^\top \|}{1600\kappa^2}\right)\label{ineq: jtwt}
\end{align}
 by choosing $\delta_{2k+1}\le  \cO(\kappa^{-5})$.
 Now if $\Vert J_tW_{t,\perp}^\top \Vert \ge 10\kappa\alpha$,
 
\begin{gather*}
    W_{t,\perp}J_t^\top J_tW_{t,\perp}^\top  - W_{t,\perp}K_t^\top K_tW_{t,\perp}^\top  - W_{t,\perp}V_t^\top V_tW_{t,\perp}^\top  \le 2\alpha^2\cdot  I\\
    W_{t,\perp}J_t^\top J_tW_{t,\perp}^\top  - W_{t,\perp}K_t^\top K_tW_{t,\perp}^\top   \le 2\alpha^2\cdot  I + W_{t,\perp}(U_t-V_t)^\top (U_t-V_t)W_{t,\perp}^\top 
\end{gather*}
Hence, 

If $\Vert J_tW_{t,\perp}^\top \Vert \ge 10\kappa\alpha$, then 
\begin{align*}
    \Vert J_tW_{t,\perp}\Vert^2 &=\Vert W_{t,\perp}J_t^\top J_tW_{t,\perp}^\top  \Vert \\ &\le \Vert W_{t,\perp}K_t^\top K_tW_{t,\perp}^\top \Vert +\left(2\alpha + \frac{8\alpha^2\sqrt{\sigma_1}+64\delta_{2k+1} \sqrt{\sigma_1}L_t}{\sigma_r}\right)^2\\
    &\le \Vert W_{t,\perp}K_t^\top K_tW_{t,\perp}^\top \Vert + \left(2\alpha + \frac{8\alpha^2\sqrt{\sigma_1}+64\delta_{2k+1} \sqrt{\sigma_1}\|J_tW_{t,\perp}^\top \|\cdot \sqrt{\sigma_1}}{\sigma_r}\right)^2\\
    &\le \Vert W_{t,\perp}K_t^\top K_tW_{t,\perp}^\top \Vert +(10\alpha+ 64\delta_{2k+1}\kappa \|J_tW_{t,\perp}^\top \|)^2\\
    &\le \Vert W_{t,\perp}K_t^\top K_tW_{t,\perp}^\top \Vert +(1/10\kappa+64\delta_{2k+1}\kappa)\cdot \Vert J_tW_{t,\perp}^\top \Vert^2\\
    &\le \|W_{t,\perp}K_t^\top K_tW_{t,\perp}^\top \Vert +(1/2)\cdot \Vert J_tW_{t,\perp}^\top \Vert^2.
\end{align*}
Thus, $\Vert K_tW_{t,\perp}^\top  \Vert \ge \Vert J_tW_{t,\perp}^\top \Vert/\sqrt{2}\ge \|J_tW_{t,\perp}^\top \|/2.$
\begin{align*}
    \Vert K_{t+1}W_{t,\perp}^\top  \Vert&\le \Vert K_tW_{t,\perp}^\top  (I - \eta W_{t,\perp}K_t^\top K_tW_{t,\perp}^\top -\eta\phi_t)\Vert +\frac{160\eta L\kappa^2}{\sigma_r}\Vert J_tW_{t,\perp}^\top W_{t,\perp}K_t^\top \Vert \cdot \Vert J_tW_{t,\perp}^\top \Vert\\
    &\le \Vert K_{t}W_{t,\perp}^\top  \Vert \Vert (I - \eta W_{t,\perp}K_t^\top K_tW_{t,\perp}^\top -\eta\phi_t)\Vert + \frac{160\eta L\kappa^2}{\sigma_r}\Vert J_tW_{t,\perp}^\top W_{t,\perp}K_t^\top \Vert \cdot \Vert J_tW_{t,\perp}^\top \Vert
\end{align*}
Then, if we denote $K' = K_{t}W_{t,\perp}^\top $, then we know 
$\Vert K'(1-\eta (K')^\top K')\Vert \le (1-\eta \frac{\sigma_1^2(K')}{2})\Vert K'\Vert $. 
Let $K' = A'\Sigma'W'$
\begin{align*}
    \Vert K'(1-\eta (K')^\top K') \Vert &= \Vert A'\Sigma' W'(I-\eta(W')^\top (\Sigma')^2W')\Vert \\&=\Vert \Sigma' (I-\eta (\Sigma')^2)\Vert
\end{align*}
Let $\Sigma'_{ii} = \zeta_i$ for $i\le r,$ then $\Sigma'(I-\eta(\Sigma')^2)_{ii} = \zeta_i-\eta\zeta_i^3$, then by the fact that $\zeta_1= \sigma_1(K_tW_{t,\perp}^\top )\le  1$, we can have $\zeta_1-\eta\zeta_1^3 = \max_{1\le i\le r} \zeta_i-\eta\zeta_i^3$ and then $$\Vert \Sigma (I-\eta\Sigma^2)\Vert = (1-\eta \Vert K'\Vert^2 )\Vert K'\Vert.$$
Hence, 
\begin{align}
    \Vert K_{t+1}W_{t,\perp}^\top  \Vert&\le \Vert K_tW_{t,\perp}^\top  (I -\eta  W_{t,\perp}K_t^\top K_tW_{t,\perp}^\top -\eta\phi_t)\Vert +\frac{160\eta L\kappa^2}{\sigma_r}\Vert J_tW_{t,\perp}^\top W_{t,\perp}K_t^\top \Vert \cdot \Vert J_tW_{t,\perp}^\top \Vert \nonumber\\&\le \Vert K_tW_{t,\perp}^\top  (I -\eta  W_{t,\perp}K_t^\top K_tW_{t,\perp}^\top )\Vert +\frac{160\eta L\kappa^2}{\sigma_r}\Vert J_tW_{t,\perp}^\top W_{t,\perp}K_t^\top \Vert \cdot \Vert J_tW_{t,\perp}^\top \Vert \nonumber\\&\le 
    \Vert K_tW_{t,\perp}^\top  \Vert \left(1-\eta \frac{\Vert K_tW_{t,\perp}^\top \Vert^2}{2}\right) +\frac{160\eta L\kappa^2}{\sigma_r}\Vert J_tW_{t,\perp}^\top \Vert \Vert W_{t,\perp}K_t^\top \Vert \cdot \Vert J_tW_{t,\perp}^\top \Vert \nonumber\\
    &\le \Vert K_tW_{t,\perp}^\top  \Vert \left(1-\eta \frac{\Vert J_tW_{t,\perp}^\top \Vert^2}{8}\right) +\frac{160\eta L\kappa^2}{\sigma_r}\Vert J_tW_{t,\perp}^\top \Vert \Vert W_{t,\perp}K_t^\top \Vert \cdot \Vert J_tW_{t,\perp}^\top \Vert \nonumber\\&\le \Vert K_tW_{t,\perp}^\top  \Vert \left(1-\eta \frac{\Vert J_tW_{t,\perp}^\top \Vert^2}{16}\right)\label{ineq: jtwt 2}\\
    &\le \Vert K_tW_{t,\perp}^\top  \Vert \left(1-4\eta \kappa^2\alpha^2\right).\label{ineq: alpha 2}
\end{align}
The fifth inequality is because $\delta_{2k+1} = O(\kappa^{-4})$.
Thus, for all cases, by Eq.\eqref{ineq: alpha}, \eqref{ineq: jtwt}, \eqref{ineq: alpha 2} and \eqref{ineq: jtwt 2}, we have 
\begin{align}
    \Vert K_{t+1}W_{t,\perp}^\top \Vert &\le \Vert K_tW_{t,\perp}^\top \Vert \cdot \min\left\{\left(1-\frac{\eta\alpha^2}{4}\right), \left(1-\frac{\eta\|J_tW_{t,\perp}^\top \|^2}{1600\kappa^2}\right)\right\}\nonumber\\&\le \Vert K_tW_{t,\perp}^\top \Vert \cdot \left(1-\frac{\eta\alpha^2}{8}\right) \cdot \left(1-\frac{\eta\|J_tW_{t,\perp}^\top \|^2}{3200\kappa^2}\right),\label{ineq:K_(t+1)W_(t,perp)^T to K_tW_(t,perp)^T }
\end{align}
where we use the inequality $\max\{a,b\} \le \sqrt{ab}$.
Now we prove the following claim: 
\begin{align}
    \Vert K_{t+1}W_{t+1,\perp}^\top \Vert \le \Vert K_{t+1}W_{t,\perp}^\top \Vert \cdot \left(1+\cO(\eta\delta_{2k+1} \|J_tW_{t,\perp}^\top \|^2/\sigma_r^{3/2})\right).
\end{align}
First consider the situation that $\Vert J_tW_{t,\perp}^\top  \Vert \le 10\kappa \alpha$. We start at these two equalities:
\begin{align*}
    K_{t+1} &= K_{t+1}W_{t,\perp}^\top W_{t,\perp} + K_{t+1}W_t^\top W_t\\
    K_{t+1} &= K_{t+1}W_{t+1,\perp}^\top W_{t+1,\perp} + K_{t+1}W_{t+1}^\top W_{t+1}.
\end{align*}
Thus, we have 
\begin{align*}
    K_{t+1}W_{t,\perp}^\top W_{t,\perp}W_{t+1,\perp}^\top  + K_{t+1}W_t^\top W_tW_{t+1,\perp}^\top  = K_{t+1}W_{t+1,\perp}^\top 
\end{align*}
Consider 
\begin{align*}
    \Vert W_tW_{t+1,\perp}^\top \Vert &= \Vert W_{t+1,\perp}W_t^\top \Vert \\
    &=\Vert W_{t+1,\perp}U_t^\top  (U_tU_t^\top )^{-1/2}\Vert \\
    &= \Vert W_{t+1,\perp} U_t^\top \Vert \Vert (U_tU_t^\top )^{-1/2}\Vert \\
    & \le  \Vert W_{t+1,\perp}\Vert \Vert U_{t+1}-U_t\Vert \cdot \sigma_r(U)^{-1}\\
    &\le \frac{4\sqrt{\sigma_1}}{\sigma_r}\cdot\eta \cdot  (2\sqrt{\sigma_1}\cdot M_t + 2\delta_{2k+1} \cdot (L_t+3M_t)\cdot 2\sqrt{\sigma_1})\\
    &\le \frac{4\sqrt{\sigma_1}}{\sigma_r}\cdot\eta \cdot (3\sqrt{\sigma_1}\cdot M_t + 2\delta_{2k+1}\cdot L_t)\\
    &\le \frac{4\sqrt{\sigma_1}}{\sigma_r}\cdot\eta (\frac{48 L\kappa\sqrt{\sigma_1}}{\sigma_r}\|J_tK_t^\top \| +2\sqrt{\sigma_1}\delta_{2k+1}\cdot L_t)\\
    &\le C\eta (\delta_{2k+1}\kappa^4+\alpha^2\kappa^2/\sigma_r)\|J_tK_t^\top \|.
\end{align*}
for some constant $C.$
Also, note that $\|F_tG_t^\top -\Sigma\| \le L_t + 3M_t \le 4L_t,$
\begin{align*}
    \Vert K_{t+1}W_{t}^\top \Vert &= \Vert (K_{t+1}-K_t)W_{t}^\top  \Vert  + \|K_tW_{t}^\top \|\\&\le \Vert \eta K_t(U_t^\top U_t+J_t^\top J_t)W_{t}^\top \Vert +\eta\delta_{2k+1}\cdot (4L_t)\cdot 2\sqrt{\sigma_1} + \|K_tW_{t}^\top \|\\
    &\le \Vert \eta K_tJ_t^\top J_tW_{t}^\top \Vert +8\sqrt{\sigma_1}\eta\delta_{2k+1}\cdot L_t + \frac{64L\kappa^{3/2}}{\sigma_r^{3/2}}L_t \\
    &\le \eta L_t \Vert J_tW_{t}^\top \Vert + 8\sqrt{\sigma_1}\eta\delta_{2k+1}\cdot L_t  + \frac{64L\kappa^{3/2}}{\sigma_r^{3/2}}L_t\\
    &\le L_t\cdot (\eta \cdot \sqrt{\sigma_1} + 8\sqrt{\sigma_1}\eta \delta_{2k+1} + \frac{64L\kappa^{3/2}}{\sigma_r^{3/2}})\\
    &\le \frac{1}{4\sqrt{\sigma_1}}L_t\\
    &\le \frac{1}{4}\|K_t\|
\end{align*}
and 
\begin{align*}
    \|K_{t+1}W_{t,\perp}^\top \| &\ge\|K_tW_{t,\perp}^\top \|-\|(K_{t+1}-K_t)W_{t,\perp}^\top \|\\&\ge  \frac{1}{2}\|K_t\| - \eta \|K_t(U_t^\top U_t+J_t^\top J_t)W_{t}^\top \Vert -8\sqrt{\sigma_1}\eta\delta_{2k+1}\cdot L_t \\&\ge\frac{1}{2}\|K_t\| -\eta L_t\|J_tW_t^\top \| - 8\sqrt{\sigma_1}\eta\delta_{2k+1}\cdot L_t  \\
    &\ge \|K_t\|(\frac{1}{2}-\eta \|J_t\|\cdot \|J_tW_{t}^\top \|-8\sqrt{\sigma_1}\eta \delta_{2k+1}\cdot \|J_t\|) \\
    &\ge \|K_t\|(\frac{1}{2}-\eta \sigma_1-8\eta \delta_{2k+1}\sigma_1)\\
    &\ge \frac{1}{4}\|K_t\|\\&\ge \|K_{t+1}W_{t}^\top \|
\end{align*}
Here, we use the fact that $\eta \le 1/\sigma_1$, $\delta_{2k+1}\le 1/32$ and $\|J_t\| \le \sqrt{\sigma_1}.$
Hence, we have 
\begin{align*}
    \Vert K_{t+1}W_{t+1,\perp}^\top \Vert &\le \Vert K_{t+1}W_{t,\perp}^\top \Vert \Vert W_{t,\perp}W_{t+1,\perp}^\top \Vert + \Vert K_{t+1}W_{t}^\top \Vert\Vert W_tW_{t+1,\perp}^\top \Vert \\&\le \Vert K_{t+1}W_{t,\perp}^\top \Vert  + \Vert K_{t+1}W_{t,\perp}^\top \Vert \cdot C\eta(\delta_{2k+1}\kappa^4+\alpha^2\kappa^2/\sigma_r)L_t\\
    &\le \left(1+C\eta(\delta_{2k+1}\kappa^4+\alpha^2\kappa^2/\sigma_r)L_t\right)\Vert K_{t+1}W_{t,\perp}^\top \Vert \\
    &\le \left(1+2C\eta(\delta_{2k+1}\kappa^4+\alpha^2\kappa^2/\sigma_r)\Vert J_tW_{t,\perp}^\top W_{t,\perp}K_t^\top \Vert \right)\Vert K_{t+1}W_{t,\perp}^\top \Vert\\&\le \left(1+2C\eta(\delta_{2k+1}\kappa^4+\alpha^2\kappa^2/\sigma_r)\Vert J_tW_{t,\perp}^\top \|\|W_{t,\perp}K_t^\top \Vert \right)\Vert K_{t+1}W_{t,\perp}^\top \Vert
\end{align*}
The inequality on the fourth line is because Eq.\eqref{eq: JW ge Lt}.

Note that 
\begin{align*}
    W_{t,\perp}J_t^\top J_tW_{t,\perp}^\top - W_{t,\perp}K_t^\top K_tW_{t,\perp}^\top  \ge \frac{\alpha^2}{8}\cdot  I.
\end{align*}
Thus, $\Vert K_tW_{t,\perp}^\top \Vert \le \Vert J_tW_{t,\perp}^\top \Vert $ and 
\begin{align}
    \Vert K_{t+1}W_{t+1,\perp}^\top \Vert &\le
     \left(1+2C\eta(\delta_{2k+1}\kappa^4+\alpha^2\kappa^2/\sigma_r)\Vert J_tW_{t,\perp}^\top \|\|W_{t,\perp}K_t^\top \Vert\right)\Vert K_{t+1}W_{t,\perp}^\top \Vert\nonumber\\
    &\le  \left(1+2C\eta(\delta_{2k+1}\kappa^4+\alpha^2\kappa^2/\sigma_r)\Vert J_tW_{t,\perp}^\top \|^2\right)\Vert K_{t+1}W_{t,\perp}^\top \Vert\label{ineq: K_(t+1)W_(t+1,perp) to K_(t+1)W_(t,perp)}
\end{align}
By inequalities \eqref{ineq:K_(t+1)W_(t,perp)^T to K_tW_(t,perp)^T } and \eqref{ineq: K_(t+1)W_(t+1,perp) to K_(t+1)W_(t,perp)}, we can get 
\begin{align*}
    & ~~~\|K_{t+1}W_{t+1,\perp}^\top \| \\
    &\le \left(1+2C\eta(\delta_{2k+1}\kappa^4+\alpha^2\kappa^2/\sigma_r)\Vert J_tW_{t,\perp}^\top \|^2\right)\Vert K_{t+1}W_{t,\perp}^\top \Vert \\&\le \left(1+2C\eta(\delta_{2k+1}\kappa^4+\alpha^2\kappa^2/\sigma_r)\Vert J_tW_{t,\perp}^\top \|^2\right) \cdot \left(1-\frac{\eta\alpha^2}{8}\right) \cdot \left(1-\frac{\eta\|J_tW_{t,\perp}^\top \|^2}{3200\kappa^2}\right) \|K_tW_{t,\perp}^\top \|\\
    &\le \left(1-\frac{\eta\alpha^2}{8}\right) \|K_tW_{t,\perp}^\top \|.
\end{align*}
The last inequality is because 
\begin{align*}
    2C\eta(\delta_{2k+1}\kappa^4+\alpha^2\kappa^2/\sigma_r)\Vert J_tW_{t,\perp}^\top \|^2\le \frac{\eta\|J_tW_{t,\perp}^\top \|^2}{3200\kappa^2}
\end{align*}
by choosing 
\begin{align}\delta_{2k+1} = \cO(\kappa^{-6})\label{delta 1}\end{align} and \begin{align}\alpha = \cO(\kappa^{-2}\cdot \sqrt{\sigma_r}).\label{alpha 1}\end{align} 
Thus, we can prove $\Vert K_tW_{t,\perp}^\top \Vert$ decreases at a linear rate. 

Now we have completed all the proofs of the induction hypotheses. Hence, 
\begin{align}
    \|F_tG_t^\top -\Sigma\| &\le 2\|J_tK_t^\top \| \nonumber\\&\le 4\|K_t^\top \|\cdot \sqrt{\sigma_1} \nonumber\\&\le 4\|K_tW_{t,\perp}^\top \|\sqrt{\sigma_1} \nonumber\\&\le 4\|K_tW_{T_2,\perp}^\top \|\cdot \sqrt{\sigma_1}\left(1-\frac{\eta\alpha^2}{8}\right)^{t-T_2}\nonumber\\
    &\le 4\|K_{T_2}\|\cdot \sqrt{\sigma_1}\left(1-\frac{\eta\alpha^2}{8}\right)^{t-T_2}\nonumber\\
    &\le 2\sigma_1\left(1-\frac{\eta\alpha^2}{8}\right)^{t-T_2}\label{phase 3: loss}
\end{align}

Now combining three phases \eqref{phase 1 loss}, \eqref{phase 2: time bound} and \eqref{phase 3: loss}, if we denote $t_2^* + T_0 = T' = \widetilde{\cO}(1/\eta \sigma_r)$, then for any round $T\ge 4T'$,  Phase 1 and Phase 3 will take totally at least $T-T'$ rounds. 
Now we consider two situations.
\paragraph{Situation 1:} Phase 1 takes at least $\frac{3(T-T')}{4}$ rounds. Then, by \eqref{phase 1 loss}, suppose Phase 1 starts at $T_0$ rounds and terminates at $T_1$ rounds, we will have 
\begin{align}
   \|F_{T_1}G_{T_1}^\top - \Sigma\| &\le \frac{\sigma_r^2}{128\alpha^2\kappa}\left(1-\frac{\eta \sigma_r^2}{64\sigma_1}\right)^{T_1-T_0}\nonumber\\& \le 
    \frac{\sigma_r^2}{128\alpha^2\kappa}\left(1-\frac{\eta \sigma_r^2}{64\sigma_1}\right)^{T/2}.\label{situation 1 loss}
\end{align}
The last inequality uses the fact that $T\ge 4T'$ and
$$T_1-T_0 \ge \frac{3(T-T')}{4}\ge T/2$$
Then, by \eqref{ineq: JtKt update phase 2}, \eqref{phase 2 M_t bound}, \eqref{ineq:M_t bound} and \eqref{ineq: JtKt update phase 3}, we know that 
\begin{align}
    \|F_{T}G_T^\top - \Sigma\| &\le 4\|J_TK_T^\top \|\nonumber\\
    &\le 4\|J_{T_1}K_{T_1}^\top - \Sigma \|\cdot \left(1+\frac{\eta\sigma_r^2}{128\sigma_1}\right) ^{T-T_1}\nonumber \\
    &\le 4\|F_{T_1}G_{T_1}^\top - \Sigma\|\cdot \left(1+\frac{\eta\sigma_r^2}{128\sigma_1}\right) ^{T-T_1}\nonumber\\
    &\le 4\|F_{T_1}G_{T_1}^\top - \Sigma\|\cdot \left(1+\frac{\eta\sigma_r^2}{128\sigma_1}\right) ^{T/2}
\end{align}
The last inequality uses the fact that $T_1-T_0\ge \frac{3(T-T')}{4}\ge \frac{T}{2}$, which implies that $\frac{T}{2}\ge T-T_1$
Then, combining with \eqref{situation 1 loss}, we can get 
\begin{align}
    \|F_{T}G_T^\top - \Sigma\| &\le \frac{\sigma_r^2}{128\alpha^2\kappa}\left(1-\frac{\eta \sigma_r^2}{64\sigma_1}\right)^{T/2} \cdot \left(1+\frac{\eta\sigma_r^2}{128\sigma_1}\right) ^{T/2}\nonumber\\
    &\le \frac{\sigma_r^2}{128\alpha^2\kappa}\left(1-\frac{\eta \sigma_r^2}{128\sigma_1}\right)^{T/2}\label{eq: explain 1 final phase}\\
    &\le \frac{\sigma_r^2}{128\alpha^2\kappa}\left(1-\frac{\eta\alpha^2}{8}\right)^{T/2}.\label{eq:explain 2 final phase}
\end{align}
\eqref{eq: explain 1 final phase} uses the basic inequality $(1-2x) (1+x) \le (1-x)$, 
and \eqref{eq:explain 2 final phase} uses the fact that $\alpha = \cO(\kappa^{-2}\sqrt{\sigma_r}) =\cO(\sqrt{\kappa \sigma_r})$.
\paragraph{Situation 2:} Phase 3 takes at least $\frac{T-T'}{4}$ rounds. Then, by \eqref{phase 3: loss}, suppose Phase 3 starts at round $T_2$, we have 
\begin{align}
    \|F_{T}G_{T}^\top - \Sigma\|
    &\le 2\sigma_1\left(1-\frac{\eta\alpha^2}{8}\right)^{t-T_2} \nonumber\\
    &\le 2\sigma_1\left(1-\frac{\eta\alpha^2}{8}\right)^{(T-T')/4}\nonumber\\
    &\le \frac{\sigma_r^2}{128\alpha^2\kappa}\left(1-\frac{\eta\alpha^2}{8}\right)^{T/8}.
\end{align}
The last inequality uses the fact that $\alpha = \cO(\kappa^{-2}\sqrt{\sigma_r}) = \cO(\kappa^{-1}\sqrt{\sigma_r})$ and $\frac{T-T'}{4}\ge \frac{T-T/4}{4} \ge T/8.$
Thus, by $\|F_TG_T^\top - \Sigma\|^2 \le  n \cdot \|F_TG_T^\top - \Sigma\|^2$, 
we complete the proof by choosing $4T' = T^{(1)}$ and $c_7 = 1/128^2$.

\section{Proof of Theorem \ref{theorem: exact asymmetric}}\label{sec: appendix_proofThmAsym}
 By the convergence result in \citep{soltanolkotabi2023implicit}, the following three conditions hold for $t = T_0.$ 
\begin{gather}
    \max\{\|J_t\|,\|K_t\|\}\le \cO\left(2\alpha + \frac{\delta_{2k+1}\sigma_1^{3/2}\log(\sqrt{\sigma_1}/n\alpha)}{\sigma_r}\right)\label{exact: condition 1}\\
    \max\{\|U_t\|,\|V_t\|\} \le 2\sqrt{\sigma_1}\label{exact: condition 2}
\end{gather}
and 
\begin{align}
    \|F_tG_t^T-\Sigma\| \le \alpha^{1/2}\sigma_1^{3/4}\le \sigma_r/2.\label{exact: condition 3}
\end{align}

Then, we define $M_t = \max\{\|U_tV_t^\top - \Sigma\|, \|U_tK_t^\top\|, \|J_tV_t^\top\|\},$  by the same techniques in Section \ref{sec: phase 1 gd linear}, if we have 
\begin{align}
    \sigma_r^2M_{t-1}/64\sigma_1 \ge (17 \sigma_1\delta_{2k+1}+\alpha^2)\|J_{t-1}K_{t-1}^\top \|,\label{exact: Mt cond}
\end{align}
we can prove that 
\begin{align}\label{ineq: Mt exact}
    M_t\le \left(1-\frac{\eta\sigma_r^2}{64\sigma_1}\right)M_{t-1}.
\end{align}
and 
\begin{gather*}
\max\{\Vert J_t \Vert, \Vert K_t \Vert\}\le 2\sqrt{\alpha}\sigma_1^{1/4}+2C_2\log(\sqrt{\sigma_1}/n\alpha) (\delta_{2k+1} \cdot \kappa^2\sqrt{\sigma_1})\le \sqrt{\sigma_1}\\
\|F_tG_t^\top - \Sigma\| \le \sigma_r/2\\
\max\{\|U_t\|, \|V_t\|\} \le 2\sqrt{\sigma_1}.
\end{gather*}
Now note that 
\begin{align}
    \|U_{t-1}K_{t-1}^\top \| \ge \lambda_{\min}(U_{t-1}) \cdot \|K_{t-1}^\top \| = \sigma_r(U_{t-1}) \cdot \|K_{t-1}^\top \| \ge \frac{\sigma_r}{4\sqrt{\sigma_1}}\cdot \|K_{t-1} \|, \label{lower bound of UK: exact}
\end{align}
Now since $\delta_{2k+1} =\cO(\kappa^{-3})$ and $\alpha =\cO(\kappa^{-1}\sqrt{\sigma_r})$ are small parameters, we can derive the $M_t$'s lower bound by 
\begin{align}
    M_{t-1} &\ge \|U_{t-1}K_{t-1}^\top \| \nonumber\\&\ge \frac{\sigma_r}{4\sqrt{\sigma_1}}\cdot \|K_{t-1} \|\nonumber\\& \ge \frac{\sigma_r}{4\sqrt{\sigma_1}}\|K_{t-1}\|\cdot \frac{\|J_{t-1}\|}{\sqrt{\sigma_1}}\\&\ge 64\sigma_1 \cdot \frac{17\sigma_1\delta_{2k+1} + \alpha^2}{\sigma_r^2}\|J_{t-1}K_{t-1}^\top \|.
\end{align}
Hence, \eqref{exact: Mt cond} always holds for $t \ge T_0$, and then by \eqref{ineq: Mt exact}, we will have 
\begin{align*}
    M_t &\le \left(1-\frac{\eta\sigma_r^2}{16\sigma_1}\right)^{t-T_0}M_{T_0}\\&\le \left(1-\frac{\eta\sigma_r^2}{16\sigma_1}\right)^{t-T_0}\|F_{T_0}G_{T_0}^T\|\\
    &\le \frac{\sigma_r}{2}\cdot \left(1-\frac{\eta\sigma_r^2}{16\sigma_1}\right)^{t-T_0} .
\end{align*}

Thus, we can bound the loss by 
\begin{align}
    \|F_tG_t^\top-\Sigma\| &\le \|U_tV_t^\top - \Sigma\| + \|J_tV_t^\top \| + \| U_t K_t^\top \| + \|J_tK_t^\top \|\nonumber\\&\le 3M_t + \|J_tK_t^\top \|\nonumber\\&\le 3M_t + \cO(2\alpha + \delta_{2k+1}\kappa \sqrt{\sigma_1}\log(\sqrt{\sigma_1}/n\alpha) \cdot \frac{4\sqrt{\sigma_1}}{\sigma_r}M_t\nonumber\\
    &\le 4M_t\label{ineq: explain exact} \\&\le 2\sigma_r\cdot \left(1-\frac{\eta\sigma_r^2}{64\sigma_1}\right)^{t-T_0}.\nonumber
\end{align}
where Eq.\eqref{ineq: explain exact} uses the fact that $\delta_{2k+1} \le \cO(\kappa^{-2}\log^{-1}(\sqrt{\sigma_1}/n\alpha))$ and $\alpha \le \cO(\sigma_r/\sqrt{\sigma_1}).$ Now we can choose $T^{(2)} = 2T_0$, and then by $t-T_0 \ge t/2$ for all $t\ge T^{(2)}$, we have 
\begin{align}
    \|F_tG_t^T-\Sigma \|_F^2 \le n \|F_tG_t^T-\Sigma \|^2 \le 2n\sigma_r\cdot \left(1-\frac{\eta\sigma_r^2}{64\sigma_1}\right)^{t-T_0} \le 2n\sigma_r\cdot \left(1-\frac{\eta\sigma_r^2}{64\sigma_1}\right)^{t/2}.
\end{align}
We complete the proof.

\section{Proof of Theorem \ref{thm:fast method}}\label{appendix: proof of fast}
During the proof of Theorem \ref{thm:fast method}, we assume $\beta$ satisfy that 
\begin{align}
    \label{eq: beta condition}
    \max\{c_7\gamma^{1/6}\sigma_1^{1/3}, c\delta_{2k+1}^{1/6} \kappa^{1/6}\sigma_1^{5/12}\}\le \beta \le c_8\sqrt{\sigma_r}
\end{align}
for some large constants $c_7, c$ and small constant $c_8$. In particular, this requirement means that $\gamma \le \sigma_r/4$. Then, since $\|\cA^*\cA(\tilde{F}_{T^{(3)}}\tilde{G}_{T^{(3)}}^\top-\Sigma)\| \ge \frac{1}{2}\|\tilde{F}_{T^{(3)}}\tilde{G}_{T^{(3)}}^\top-\Sigma\|$ by RIP property and $\delta_{2k+1} \le 1/2$, we can further derive $\|F_{T^{(3)}}G_{T^{(3)}}^\top-\Sigma\|=\|\tilde{F}_{T^{(3)}}\tilde{G}_{T^{(3)}}^\top-\Sigma\| \le \sigma_r/2.$ 

To guarantee \eqref{eq: beta condition}, we can use choose $\gamma$ to be small enough, i.e., $\gamma \ll \sigma_1\kappa^{-2}$, so that \eqref{eq: beta condition} holds easily. In the following, we denote $\delta_{2k+1} = \sqrt{2k+1}\delta$.

\subsection{Proof Sketch of Theorem \ref{thm:fast method}}
First, suppose we modify the matrix $\widetilde{F}_{T^{(3)}}, \widetilde{G}_{T^{(3)}}$ to $F_{T^{(3)}}$ and $G_{T^{(3)}}$ at $t =T^{(3)},$ then $\|F_{T^{(3)}}\|^2 = \lambda_{\max}((F_{T^{(3)}})^\top F_{T^{(3)}})=\beta^2$ and $\|U_{T^{(3)}}\|^2 \le \beta^2.$ 
Also, by $\|\widetilde{F}_{T^{(3)}}\|\le 2\sqrt{\sigma_1}$, we can get that $\|G_{T^{(3)}}\| \le \|\widetilde{G}_{T^{(3)}}\|\cdot \frac{\|\widetilde{F}_{T^{(3)}}\|}{\beta}\le \|\widetilde{G}_{T^{(3)}}\|\cdot \frac{2\sqrt{\sigma_1}}{\beta}$ is still bounded. Similarly, $\|V_{T^{(3)}}\|\le \|\widetilde{V}_{T^{(3)}}\|\cdot \frac{2\sqrt{\sigma_1}}{\beta}$ and $\|K_{T^{(3)}}\| \le \|\widetilde{K}_{T^{(3)}}\| \cdot \frac{2\sqrt{\sigma_1}}{\beta}$ is still bounded.  With these conditions, define $S_t= \max\{\|U_tK_t^\top \|, \|J_tK_t^\top \|\}$ and $P_t = \max\{\|J_tV_t^\top \|, \|U_tV_t^\top -\Sigma\|\}$. 
For $\|K_{t+1}\|$, since we can prove $\lambda_{\min}(F_t^\top F_t)\ge \beta^2/2$ for all $t \ge T^{(3)}$ using induction, with the updating rule, we can bound $\|K_{t+1}|$ as the following
\begin{align}
    \|K_{t+1}\|&\le \|K_t\|\|1-\eta F_t^\top F_t\| + 2\eta \delta_{2k+1}\cdot \|F_tG_t^\top -\Sigma\|\max\{\|U_t\|,\|J_t\|\}\\
    &\le \|K_t\|\cdot \left(1-\frac{\eta \beta^2}{2}\right) + \left(4\eta \delta_{2k+1}\beta\cdot P_t + 4\beta^2\eta \delta_{2k+1}\|K_t\|\right).\label{update rule of Kt main text fast}
\end{align}
The first term of \eqref{update rule of Kt main text fast} ensures the linear convergence, and the second term represents the perturbation term. To control the perturbation term, for $P_t$, with more calculation 
(see details in the rest of the section), we have
\begin{align}
    P_{t+1}\le \left(1-\eta\sigma_r^2/8\beta^2\right) P_t + \eta \|K_t\|\cdot \widetilde{\cO}\left(\left(\delta_{2k+1}\sigma_1+\sqrt{\alpha \sigma_1^{7/4}}\right)/\beta\right).\label{update rule of Pt main}
\end{align}
The last inequality uses the fact that $S_t \le \|K_t\|\cdot \max\{\|U_t\|,\|J_t\|\}\le \|K_t\|\cdot \|F_t\| \le \sqrt{2}\beta\cdot \|K_t\|.$

Combining \eqref{update rule of Pt main} and \eqref{update rule of Kt main text fast}, we can show that $P_t + \sqrt{\sigma_1}\|K_t\|$ converges at a linear rate $(1-\cO(\eta \beta^2)),$ since the second term of Eq. \eqref{update rule of Pt main} and Eq.\eqref{update rule of Kt main text fast} contain $\delta_{2k+1}$ or $\alpha$, which is relatively small and can be canceled by the first term. Hence, $\|F_tG_t^\top -\Sigma\| \le 2P_t + 2S_t \le 2P_t + \sqrt{2}\beta\|K_t\|$ converges at a linear rate.

\subsection{Proof of Theorem \ref{thm:fast method}}
\label{sec: proof.fast_method}
At time $t\ge T^{(3)}$, we have $\sigma_{\min}(U_{T^{(3)}}V_{T^{(3)}}) \ge \sigma_{\min}(\Sigma) - \| U_{T^{(3)}}V_{T^{(3)}}^\top  - \Sigma\| \ge \sigma_r-\alpha^{1/2}\cdot \sigma_1^{3/4} \ge \sigma_r/2. $ The last inequality holds because $\alpha = O(\kappa^{-3/2}\cdot \sqrt{\sigma_r})$. 
Then, given that $\|F_{T^{(3)}}\|^2 = \lambda_{\max}((F_{T^{(3)}})^\top F_{T^{(3)}}) = \beta^2 $, we have $\|U_{T^{(3)}}\|^2 \le \beta^2$. Hence, by $\sigma_1(U) \cdot \sigma_r(V) \ge \sigma_r(UV^\top )$, we have $$\sigma_r(V_{T^{(3)}}) \ge \frac{\sigma_r(U_{T^{(3)}}V_{T^{(3)}})}{\sigma_1(U_{T^{(3)}})}\ge \frac{\sigma_r}{2\beta}.$$

Also, by $\sigma_1'=\|\widetilde{F}_{T^{(3)}}\|\le 2\sqrt{\sigma_1},$ we can get $$\|G_{T^{(3)}}\| \le \|\widetilde{G}_{T^{(3)}}\|\|B\Sigma_{inv}^{-1}\| \le \|\widetilde{G}_{T^{(3)}}\| \cdot \frac{\sigma_1'}{\beta}\le\|\widetilde{G}_{T^{(3)}}\| \cdot \frac{2\sqrt{\sigma_1}}{\beta}.$$ Similarly, $\|V_{T^{(3)}}\| \le  \|\widetilde{V}_{T^{(3)}}\| \cdot \frac{2\sqrt{\sigma_1}}{\beta}$ and $\|K_{T^{(3)}}\| \le  \|\widetilde{K}_{T^{(3)}}\| \cdot \frac{2\sqrt{\sigma_1}}{\beta}$.

Denote $S_t = \max\{\|U_tK_t^\top \|, \|J_tK_t^\top \|\}, P_t = \max\{\|J_tV_t^\top \|, \|U_tV_t^\top -\Sigma\|\}$. Now we prove the following statements by induction:
\begin{gather}
    P_{t+1} \le \left(1-\frac{\eta\sigma_r^2}{8\beta^2}\right)P_t + \eta S_t \cdot \cO \left(\frac{\log(\sqrt{\sigma_1}/n\alpha)\delta_{2k+1}\kappa^2 \sigma_1^2+\sqrt{\alpha}\sigma_1^{7/4}}{\beta^2}\right)\label{fast: condition 2}\\
    \|F_{t+1}G_{t+1}^\top -\Sigma\| \le \frac{\beta^6}{\sigma_1^2}\left(1-\frac{\eta \beta^2}{2}\right)^{t+1-T^{(3)}}\le \sigma_r/2\label{fast: condition 3}\\
    \max\{\|F_{t+1}\|,\|G_{t+1}\|\} \le 4 \sigma_1/\beta \label{fast: condition 1}
    \\ \frac{\beta^2}{2} I \le F_{t+1}^\top F_{t+1} \le 2\beta^2 I \label{fast: condition 4}\\
    \|K_{t}\|\le \cO(2\sqrt{\alpha}\sigma_1^{1/4}+\delta_{2k+1}\log(\sqrt{\sigma_1}/n\alpha)\cdot \kappa^2\sqrt{\sigma_1}) \cdot \frac{2\sqrt{\sigma_1}}{\beta} \label{fast: condition 5}
\end{gather}

\paragraph{Proof of Eq.\eqref{fast: condition 2}}

First, since $\|F_t\|^2 = \lambda_{\max}((F_t)^\top F_t) \le 2\beta^2 $, we have $\|U_{t}\|^2 \le 2\beta^2$. Then, because  $\sigma_{\min}(U_tV_t)\ge \sigma_{\min}(\Sigma) - \|U_tV_t^\top -\Sigma\| \ge \sigma_r/2$, by $\sigma_1(U) \cdot \sigma_r(V) \ge \sigma_r(UV^\top )$, we have $$\sigma_r(V_{t}) \ge \frac{\sigma_r(U_{t}V_{t})}{\sigma_1(U_{t})}\ge \frac{\sigma_r}{2\beta}.$$

we write down the updating rule as 
\begin{align*}
    &\quad U_{t+1}V_{t+1}^\top -\Sigma 
    \\&= (1-\eta U_tU_t^\top )(U_tV_t^\top -\Sigma)(1-\eta V_tV_t^\top )  -\eta U_tK_t^\top K_tV_t^\top -\eta U_tJ_t^\top J_tV_t^\top  + B_t
\end{align*}
where $B_t$ contains the $\cO(\eta^2)$ terms and $\cO(E_i(F_tG_t^\top -\Sigma))$ terms 
\begin{align*}
    \|B_t\| &\le 4\eta \delta_{2k+1} (F_tG_t^\top -\Sigma) \max\{\|F_t\|^2,\|G_t\|^2\} + \cO(\eta^2\|F_tG_t^\top -\Sigma\|^2 \max\{\|F_t\|^2,\|G_t\|^2\}) 
\end{align*}
Hence, we have 
\begin{align}
    &\quad \|U_{t+1}V_{t+1}^\top -\Sigma\| \nonumber\\&\le (1-\frac{\eta\sigma_r^2}{4\beta^2})\|U_tV_t^\top -\Sigma\| + \eta \|U_tK_t^\top \|\|K_tV_t^\top \| + \eta \|J_tV_t^\top \|\|J_t^\top U_t^\top \| + \|B_t\|\nonumber\\&\le (1-\frac{\eta\sigma_r^2}{4\beta^2})P_t + \eta S_t \|K_t\| \|V_t\| + \eta P_t \| J_t\|\|U_t\| + \|B_t\|\nonumber\\&\le (1-\frac{\eta\sigma_r^2}{4\beta^2})P_t + \eta S_t \cdot \frac{4\sigma_1}{\beta^2}\cdot \cO\left(2\sqrt{\alpha}\sigma_1^{1/4}+\delta_{2k+1}\log(\sqrt{\sigma_1}/n\alpha)\cdot \kappa^2\sqrt{\sigma_1}\right)\cdot 2\sqrt{\sigma_1} + \eta P_t \beta\cdot \beta \nonumber\\&\qquad   + 4\eta \delta_{2k+1}\cdot 2(P_t+S_t)\cdot 4\sigma_1\cdot \frac{4\sigma_1}{\beta^2} + \cO(\eta^2(P_t+S_t)^2\cdot 4\sigma_1\cdot \frac{4\sigma_1}{\beta^2})\nonumber\\
    &\le  \left(1-\frac{\eta\sigma_r^2}{8\beta^2}\right)P_t + \eta S_t \cdot \cO \left(\frac{\log(\sqrt{\sigma_1}/n\alpha)\delta_{2k+1}\kappa^2 \sigma_1^2+\sqrt{\alpha}\sigma_1^{7/4}}{\beta^2}\right)\label{ineq: PtSt recursion}
\end{align}
The last inequality uses the fact that 
\begin{gather}
    \beta^2 = \cO(\sigma_r^{1/2})\nonumber\\
    \delta_{2k+1} = \cO(\kappa^{-2})\nonumber\\
    P_t + S_t \le 2\|F_tG_t^\top -\Sigma\| \le \cO(\sigma_1^2/\beta^2)\le 1/\eta.\nonumber
\end{gather}
Similarly, we have 
\begin{align*}
    &\quad \|J_{t+1}V_{t+1}^\top \|\\&\le \left(1-\eta J^\top J\right) JV^\top  (1-\eta V^\top V) -\eta JK^\top KV^\top  - \eta JU^\top (UV^\top -\Sigma) +C_t  
\end{align*}
where $C_t$ satisfies that 
\begin{align*}
    \|C_t\| &\le 4\eta \delta_{2k+1} (F_tG_t^\top -\Sigma)\max\{\|F_t\|^2, \|G_t\|^2\}  + \cO(\eta^2\|F_tG_t^\top -\Sigma\|\max\{\|F_t\|^2, \|G_t\|^2\}) \\
    &\le 4\eta \delta_{2k+1}\cdot 2(P_t+S_t)\cdot \frac{16\sigma_1^2}{\beta^2} + \cO(\eta^2(P_t+S_t)\cdot \sigma_1\cdot \frac{\sigma_1}{\beta^2}).
\end{align*}

Thus, similar to Eq.\eqref{ineq: PtSt recursion}, we have 
\begin{align*}
    \|J_{t+1}V_{t+1}^\top  \| 
    \le \left(1-\frac{\eta\sigma_r^2}{8\beta^2}\right)P_t + \eta S_t \cdot \cO \left(\frac{\log(\sqrt{\sigma_1}/n\alpha)\delta_{2k+1}\kappa^2 \sigma_1^2+\sqrt{\alpha}\sigma_1^{7/4}}{\beta^2}\right).
\end{align*}
Hence, we have 
\begin{align*}
    P_{t+1} \le \left(1-\frac{\eta\sigma_r^2}{8\beta^2}\right)P_t + \eta S_t \cdot \cO \left(\frac{\log(\sqrt{\sigma_1}/n\alpha)\delta_{2k+1}\kappa^2 \sigma_1^2+\sqrt{\alpha}\sigma_1^{7/4}}{\beta^2}\right).
\end{align*}
\paragraph{Proof of Eq.\eqref{fast: condition 3}}

We have
$S_t \le \|K_t\| \cdot \max\{\|U_t\|, \|J_t\|\}\le \|K_t\| \cdot \|F_t\|\le \sqrt{2}\beta\cdot \|K_t\|$. So the inequality above can be rewritten as 
\begin{align*}
    P_{t+1} &\le \left(1-\frac{\eta\sigma_r^2}{8\beta^2}\right)P_t + \eta  \sqrt{2}\beta \cdot\|K_t\|\cdot  \cO \left(\frac{\log(\sqrt{\sigma_1}/n\alpha)\delta_{2k+1}\kappa^2 \sigma_1^2+\sqrt{\alpha}\sigma_1^{7/4}}{\beta^2}\right) \\&= \left(1-\frac{\eta\sigma_r^2}{8\beta^2}\right)P_t + \eta \|K_t\| \cdot \cO \left(\frac{\log(\sqrt{\sigma_1}/n\alpha)\delta_{2k+1}\kappa^2 \sigma_1^2+\sqrt{\alpha}\sigma_1^{7/4}}{\beta}\right)
\end{align*}
Also, for $K_{t+1}$, we have 
\begin{align*}
    \|K_{t+1}\| &= \|K_t\|\|(1-\eta F_t^\top F_t)\| + 2\delta_{2k+1}\cdot \|F_tG_t^\top -\Sigma\|\max\{\|U_t\|, \|J_t\|\}\\
    &\le \|K_t\|(1-\frac{\eta \beta^2}{2}) + 2\eta \delta_{2k+1}\cdot (P_t+S_t)\cdot \sqrt{2}\beta \\
    &\le \|K_t\|(1-\frac{\eta \beta^2}{2}) +2\eta \delta_{2k+1}\cdot P_t\cdot \sqrt{2}\beta  + 2\eta \delta_{2k+1}\cdot \sqrt{2}\beta  \|K_t\| \cdot\sqrt{2}\beta \\
    &=\|K_t\|(1-\frac{\eta \beta^2}{2}) +4\eta \delta_{2k+1}\cdot \beta P_t  + 4\beta^2\eta \delta_{2k+1}\cdot  \|K_t\|
\end{align*}
Thus, we can get 
\begin{align*}
   &\quad  P_{t+1} + \sqrt{\sigma_1}\|K_{t+1}\| \\&\le \max\{1-\frac{\eta\sigma_r^2}{8\beta^2}, 1-\frac{\eta \beta^2}{2}\}(P_t + \|K_t\|)  \\&\qquad +\eta\max\left\{\cO \left(\frac{\log(\sqrt{\sigma_1}/n\alpha)\delta_{2k+1}\kappa^2 \sigma_1^{3/2}+\sqrt{\alpha}\sigma_1^{5/4}}{c}\right)+ 4\beta^2 \delta_{2k+1}, 4\beta \sqrt{\sigma_1}\delta_{2k+1}\right\}\\& \qquad \qquad \cdot (P_t + \sqrt{\sigma_1}\|K_t\|)\\
    &\le (1-\frac{\eta \beta^2}{4})(P_t + \sqrt{\sigma_1}\|K_t\|).
\end{align*}
The last inequality uses the fact that $\beta\le \cO(\sigma_r^{1/2})$ and 
\begin{align}
\delta_{2k+1} \le \cO(\beta/\sqrt{\sigma_1}\log(\sqrt{\sigma_1}/n\alpha)).
\end{align}
Hence, \begin{align*}
    \|K_{t}\| &\le (P_{T^{(3)}}/\sqrt{\sigma_1}+\|K_{T^{(3)}}\|)\cdot \left(1-\frac{\eta \beta^2}{2}\right)^{t-T^{(3)}}\\
    &\le \|K_{T^{(3)}}\|+\|F_tG_t^\top -\Sigma\|/\sqrt{\sigma_1}\\
    &\le \cO(\sqrt{\alpha}\sigma_1^{1/4}+\delta_{2k+1}\log(\sqrt{\sigma_1}/n\alpha)\cdot \kappa^2\sqrt{\sigma_1})+\alpha^{1/2}\cdot \sigma_1^{1/4}\\
    &= \cO(\sqrt{\alpha}\sigma_1^{1/4}+\delta_{2k+1}\log(\sqrt{\sigma_1}/n\alpha)\cdot \kappa^2\sqrt{\sigma_1})
\end{align*}
Hence, $P_t + \sqrt{\sigma_1}\|K_t\|$ is linear convergence. Hence, by $\beta\le \sqrt{\sigma_1}$, \begin{align*}\quad \|F_{t+1}G_{t+1}^\top -\Sigma\|& \le 2P_{t+1} + 2S_{t+1} \\&\le 2P_{t+1} + \sqrt{2}\beta\|K_{t+1}\|\\&\le (2+\sqrt{2}\beta/\sqrt{\sigma_1})(P_{t+1} + \sqrt{\sigma_1}\|K_{t+1}\|)\\&\le 4(P_{T^{(3)}} + \sqrt{\sigma_1}\|K_{T^{(3)}}\|)\cdot  \left(1-\frac{\eta \beta^2}{2}\right)^{t+1-T^{(3)}}\end{align*}
Last, 
note that by $\beta\ge c_7(\gamma^{1/6}\sigma_1^{1/3})$ and 
$\beta \ge c\delta_{2k+1}^{1/6}\kappa^{1/6}\sigma_1^{5/12}\log(\sqrt{\sigma_1}/n\alpha)^{1/6},$ by choosing 
 for some constants $c_7$ and $c$, by choosing large $c'$ and $c_7 = 2^6$, we can get 
$$\gamma \le \frac{\beta^6}{2\sigma_1^2}, \quad \sqrt{\sigma_1}\cdot\cO(\log(\sqrt{\sigma_1}/n\sqrt{\alpha})\delta_{2k+1}\cdot \sigma_1^{3/2}/\sigma_r) \cdot (2\sqrt{\sigma_1}/\beta) \le \frac{\beta^6}{2\sigma_1^2}$$ and 
$$P_{T^{(3)}} + \sqrt{\sigma_1}\|K_{T^{(3)}}\| \le \gamma + \sqrt{\sigma_1}\cdot\cO(\log(\sqrt{\sigma_1}/n\sqrt{\alpha})\delta_{2k+1}\cdot \sigma_1^{3/2}/\sigma_r) \cdot (2\sqrt{\sigma_1}/\beta)\le \beta^6/\sigma_1^2 $$
we have 
\begin{align}
    \|F_{t+1}G_{t+1}^\top -\Sigma\| \le \left(\frac{\beta^6}{\sigma_1^2}\right) \left(1-\frac{\eta \beta^2}{2}\right)^{t+1-T^{(3)}}
\end{align}
\paragraph{Proof of Eq.\eqref{fast: condition 1}}

Note that we have  $\max\{\|F_{T^{(3)}}\|,\|G_{T^{(3)}}\|\} \le 4\sqrt{\sigma_1}\cdot \sqrt{\sigma_1}/\beta = 4\sigma_1/\beta$. Now suppose  $\max\{\|F_{t'}\|,\|G_{t'}\|\} \le 4\sqrt{\sigma_1}\cdot \sqrt{\sigma_1}/\beta = 4\sigma_1/\beta$ for all $t' \in [T^{(3)}, t]$, 
then the changement of $F_{t+1}$ and $G_{t+1}$ can be bounded by 
\begin{align*}
    \|F_{t+1}-F_{T^{(3)}}\| &\le\eta  \sum_{t'=T^{(3)}}^{t}2\| F_{t'}G_{t'}-\Sigma\|\|G_{t'}\|\le \eta\cdot 2\cdot \left(\frac{\beta^6}{\sigma_1^2}+\frac{\sigma_r}{2}\right)\cdot \frac{2}{\eta \beta^2}\frac{4\sigma_1}{\beta} \le \frac{16\beta^3}{\sigma_1} + \frac{8\sigma_1^2}{\beta^3}\\
    \|G_t-G_{T^{(3)}}\| &\le \eta\sum_{t'=T^{(3)}}^{t-1}2\| F_{t'}G_{t'}-\Sigma\|\|F_{t'}\|\le \frac{16\beta^3}{\sigma_1}+\frac{8\sigma_1^2}{\beta^3}
\end{align*}

Then, by the fact that $\beta\le \cO(\sigma_1^{-1/2})$, we can show that 
\begin{gather*}
    \|F_{t+1}\| \le \|F_{T^{(3)}}\| + \|F_{t+1}-F_{T^{(3)}}\| \le \frac{2\sigma_1}{\beta} + \frac{16\beta^3}{\sigma_1}+\frac{8\sigma_1^2}{\beta^3}\le \frac{4\sigma_1}{\beta},\\
    \|G_{t+1}\| \le \|G_{T^{(3)}}\| + \|G_{t+1}-G_{T^{(3)}}\| \le \frac{2\sigma_1}{\beta} + \frac{16c^3}{\sigma_1}+\frac{8\sigma_1^2}{\beta^3}\le \frac{4\sigma_1}{\beta}.
\end{gather*}
\paragraph{Proof of Eq.\eqref{fast: condition 4}}

Moreover, we have 
\begin{align*}
    \sigma_{k}(F_{t+1}) &\ge \sigma_{k}(F_{T^{(3)}})-\sigma_{\max}(F_{t+1}-F_{T^{(3)}})\\&=\sigma_{k}(F_{T^{(3)}})-\|F_{t+1}-F_{T^{(3)}}\| \\&\ge \beta - \frac{16\beta^3}{\sigma_1}\\
    &\ge \beta/\sqrt{2},
\end{align*}
and 
\begin{align*}
    \|F_t\| \le \|F_{T^{(3)}}\| + \|F_t-F_{T^{(3)}}\|\le \beta +\frac{16\beta^3}{\sigma_1} \le \sqrt{2}\beta.
\end{align*}
The last inequality is because $\beta\le \cO(\sigma_1^{-1/2})$.
Hence, since $F_{t+1} \in \RR^{n\times k}$, we have 
\begin{align}
    \frac{\beta^2}{2}I\le F_{t+1}^\top F_{t+1}\le 2\beta^2I
\end{align}
Thus, we complete the proof.

\section{Technical Lemma}\label{appendix: technical lemma}
\subsection{Proof of Lemma \ref{lemma:initialangle}}\label{sec: proof of lemma initial angle}
\begin{proof}
    We only need to prove with high probability, 
    \begin{align}
        \max_{i,j \in [n]}\cos^2 \theta_{x_j,x_k} \le \frac{c}{\log^2(r\sqrt{\sigma_1}/\alpha)(r\kappa)^2}. 
    \end{align}
    In fact, since $\cos^2\theta_{x_j,x_k} =\sin^2(\frac{\pi}{2}-\theta_{x_j,x_k})\le (\pi/2-\theta_{x_j,x_k})^2$, we have
    \begin{align}
        \PP\left[|\pi/2-\theta_{x_j,x_k}|> \cO\left(\frac{\sqrt{c}}{\log(r\sqrt{\sigma_1}/\alpha)r\kappa }\right)\right]\ge \PP\left[\cos^2 \theta_{x_j,x_k}> \cO\left(\frac{c}{\log^2(r\sqrt{\sigma_1}/\alpha)(r\kappa)^2}\right)\right].
    \end{align}
    Moreover, for any $m>0$, by Lemma \ref{lemma:angle},  
    \begin{align}
        \PP\left[|\pi/2-\theta_{x_j,x_k}|> m\right] \le \cO\left(\frac{(\sin(\frac{\pi}{2}-m))^{k-2}}{1/\sqrt{k-2}}\right) &= \cO\left(\sqrt{k-2}(\cos m)^{k-2}\right)\\&\le \cO\left(\sqrt{k}(1-m^2/4)^{k-2}\right)\\&\le \cO\left(\sqrt{k}\exp\left(-\frac{4k}{m^2}\right)\right).
    \end{align}
    The second inequality uses the fact that $\cos x \le 1-x^2/4.$ Then, if we choose $$m = \frac{\sqrt{c}}{\log(r\sqrt{\sigma_1}/\alpha)r\kappa}$$ and let $k \ge 16/m^4= \frac{16\log^4(r\sqrt{\sigma_1}/\alpha)(r\kappa)^4}{c^2}$, we can have 
    \begin{align}
        \PP\left[\cos^2 \theta_{x_j,x_k}> m^2\right]&\le 
 \PP\left[|\pi/2-\theta_{x_j,x_k}|> m\right]  \\&\le  \cO\left(k\exp\left(-\frac{m^2k}{4}\right)\right)\\& \le \cO\left(k\exp\left(-\sqrt{k}\right)\right)
    \end{align}
    Thus, by taking the union bound over $j,k \in [n]$, there is a constant $c_2$ such that, with probability at least $1-c_4n^2k\exp(-\sqrt{k}), $ we have 
    \begin{align}
        \theta_0 \le \frac{c}{\log^2(r\sqrt{\sigma_1}/\alpha)(r\kappa)^2}.
    \end{align}
\end{proof}
\subsection{Proof of Lemma \ref{lemma:initial length}}\label{proof of lemma initial legnth}
\begin{proof}
    Since $x_i = \alpha/\sqrt{k}\cdot \tilde{x}_i$, where each element in $\tilde{x}_i$ is sampled from $\cN(0,1)$. By Theorem 3.1 in \cite{vershynin2018high}, there is a constant $c$ such that
    \begin{align}
        \PP\left[|\|\tilde{x}_i^0\|_2^2-k|\ge t\right]\le 2\exp(-ct)
    \end{align}
    Hence, choosing $t= (1-\frac{1}{\sqrt{2}})k$, we have
    \begin{align*}
        \PP[\|\tilde{x}_i^0\|_2^2 \in [k/\sqrt{2},\sqrt{2}k]]\le\PP[|\|\tilde{x}_i^0\|_2^2-k|\ge t]\le  2\exp(-ct) \le 2\exp(-ck/4)
    \end{align*}
    Hence, 
    \begin{align}
        \PP\Big[\|x_i^0\|^2 \in [\alpha^2/2,2\alpha^2]\Big] = \PP\Big[\|\tilde{x}_i^0\|^2 \in [k/\sqrt{2},\sqrt{2}k]\Big] \le 2\exp(-ck/4).
    \end{align}
    By taking the union bound over $i \in [n]$, we complete the proof.
\end{proof}

\begin{lemma}\label{lemma:angle}
    Assume $x,y \in \RR^n$ are two random vectors such that each element is independent and sampled from $\cN(0,1)$, then define $\theta$ as the angle between $x,y$, we have
    \begin{align}
        \PP\Big(\Big|\theta-\frac{\pi}{2}\Big| \le m\Big) \le \frac{3\pi\sqrt{n-2}(\sin (\pi/2-m))^{n-2}}{4\sqrt{2}}.
    \end{align}
\end{lemma}
\begin{proof}

First, it is known that $\frac{x}{\|x\|}$ and $\frac{y}{\|y\|}$ are independent and uniformly distributed over the sphere $\mathbb{S}^{n-1}.$ Thus, without loss of generality, we can assume $x$ and $y$ are independent and uniformly distributed over the sphere.

    Note that $\theta \in [0,\pi]$, and the CDF of $\theta$ is 
    \begin{align}
        f(\theta) = \frac{\Gamma(n/2)\sin^{n-2}(\theta)}{\sqrt{\pi}\Gamma(\frac{n-1}{2})}
    \end{align}
    Then, we have 
    \begin{align}
        \PP\Big(\Big|\theta-\frac{\pi}{2}\Big| > m\Big) &= 1-\frac{\int_{\pi/2-m}^{\pi/2+m}\sin^{n-2}\theta d\theta }{\int_{0}^\pi \sin^{n-2}\theta d\theta}=\frac{\int_0^{\pi/2-m}\sin^{n-2}\theta d\theta}{\int_0^{\pi/2}\sin^{n-2}\theta d\theta}\\&\le \frac{(\pi/2)\cdot \sin^{n-2}(\pi/2-m)}{\int_0^{\pi/2} \cos^{n-2}\theta d\theta}\\&\le \frac{(\pi/2 \cdot (\pi/2-m)^{n-2})}{\int_0^{\sqrt{2}}(1-t^2/2)^{n-2}dt}\\
        &\le \frac{(\pi/2)\cdot (\pi/2-m)^{n-2}}{\frac{2\sqrt{2}}{3\sqrt{n-2}}}\\
        &=\frac{3\pi\sqrt{n-2}(\sin (\pi/2-m))^{n-2}}{4\sqrt{2}}.
    \end{align}
\end{proof}

\begin{lemma}[Lemma 7.3 (1) in \cite{stoger2021small}]\label{lemma: rip matrix}
Let $\cA$ be a linear measurement operator that satisfies the RIP property of order $2k+1$ with constant $\delta,$ then we have for all matrices with rank no more than $2k$ 
\begin{align}
    \|(I-\cA^*\cA)(X)\|\le \sqrt{2k}\cdot \delta \|X\|.
\end{align}
\end{lemma}

\begin{lemma}[\cite{soltanolkotabi2023implicit}]\label{lemma: mahdi complete}
    There exist parameters $\zeta_0$, $\delta_0, \alpha_0, \eta_0$ such that, if we choose $\alpha\le \alpha_0$, $F_0 = \alpha\cdot \tilde{F}_0, G_0 = (\alpha/3) \cdot \tilde{G}_0$, where the elements of $\tilde{F}_0, \tilde{G}_0$ is $\cN(0,1/n)$,\footnote{Note that in \cite{soltanolkotabi2023implicit}, the initialization is $F_0 = \alpha \cdot \tilde{F_0}$ and $G_0 = \alpha \cdot \tilde{G_0}$, while Lemma \ref{lemma: mahdi complete} uses a slightly imbalance initialization. It is easy to show that their techniques also hold with this imbalance initialization.} and suppose that the operator $\cA$ defined in Eq.\eqref{eq: observation_model} satisfies the restricted isometry property of order $2r+1$ with constant $\delta \le \delta_{0}$, then the gradient descent with step size $\eta \le \eta_0$ will achieve 
    \begin{align}
        \|F_tG_t^\top -\Sigma\| \le \alpha^{3/5}\cdot \sigma_1^{7/10}
    \end{align}
    within $T = \widetilde{\cO}(1/\eta \sigma_r)$ rounds with probability at least $1-\zeta_0$, where $\zeta_0 = c_1\exp (-c_2k) + (c_3\upsilon)^{k-r+1}$ is a small constant. Moreover, during $T$ rounds, we always have 
    \begin{align}
        \max\{\|F_t\|,\|G_t\|\} \le 2\sqrt{\sigma_1}.
    \end{align}

    The parameters $\alpha_0, \delta_0$ and $\eta_0$ are selected by 
    \begin{gather}
        \alpha_0 = \cO\left(\frac{\sqrt{\sigma_1}}{k^5\max\{2n,k\}^2}\right)\cdot \left(\frac{\sqrt{k}-\sqrt{r-1}}{\kappa^2\sqrt{\max\{2n,k\}}}\right)^{C\kappa}\\
        \delta_0 \le \cO\left(\frac{1}{\kappa^3\sqrt{r}}\right)\\
        \eta \le \cO\left(\frac{1}{k^5\sigma_1}\cdot \frac{1}{\log\left(\frac{2\sqrt{2\sigma_1}}{\upsilon\alpha (\sqrt{k}-\sqrt{r-1}}\right)}\right)
    \end{gather}
\end{lemma}

\section{Experiment Details}\label{sec:exp}
In this section, we provide experimental results to corroborate our theoretical observations.  

\textbf{Symmetric Lower Bound}
In the first experiment, we choose $n=50, r=2$,  three different $k=5,3,2$ and learning rate $\eta = 0.01$ for the symmetric matrix factorization problem. The results are shown in Figure \ref{figure:symmetric}, which matches our $\Omega(1/T^2)$ lower bound result in Theorem \ref{thm:symmetric lower bound} for the over-parameterized setting, and previous linear convergence results for exact-parameterized setting. 

\textbf{Asymmetric Matrix Sensing}
In the second experiment, we choose configuration $n=50, k=4, r=2$, sample number $m = 700 \approx nk^2$   and learning rate $\eta = 0.2$ for the asymmetric matrix sensing problem. 
To demonstrate the direct relationship between convergence speed and initialization scale, we conducted multiple trials employing distinct initialization scales $\alpha = 0.5, 0.2, 0.05$. The experimental results in Figure \ref{figure:asymmetric} offer compelling evidence supporting three key findings:

$\bullet$ The loss exhibits a linear convergence pattern.

$\bullet$ A larger value of $\alpha$ results in faster convergence under the over-parameterization setting

$\bullet$ The convergence rate is not dependent on the initialization scale under the exact-parameterization setting.

These observations highlight the influence of the initialization scale on the algorithm's performance. 

In the last experiment, we run our new method with the same $n$ and $r$ but two different $k=3,4$. Unlike the vanilla gradient descent, at the midway point of the episode, we applied a transformation to the matrices $F_t$ and $G_t$ as specified by Eq. \eqref{eq: main_matrix change}. As illustrated in Figure \ref{figure:algfast}, it is evident that the rate of loss reduction accelerates after the halfway mark. This compelling observation serves as empirical evidence attesting to the efficacy of our algorithm.

\section{Additional Experiments}

In this section, we provide some additional experiments to further corroborate our theoretical findings.

\subsection{Comparisons between Asymmetric and Symmetric Matrix Sensing}

We run both asymmetric and symmetric matrix sensing with $n=50, n=4, r=2$ with sample $m = 1200$ and learning rate $ \eta = 0.2$. We run the experiment for three different initialization scales $\alpha = 0.5, 0.2, 0.05$. The experiment results in Figure \ref{fig:comp_asym_sym} show that asymmetric matrix sensing converges faster than symmetric matrix sensing under different initialization scales.
\begin{figure}[H]
    \centering
    \includegraphics[scale=0.6]{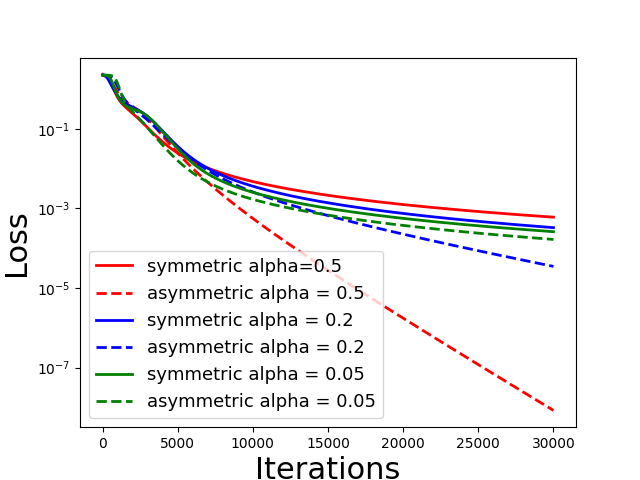}
    \caption{Comparisons between asymmetric and symmetric matrix sensing with different initialization scales. The dashed line represents the asymmetric matrix sensing, and the solid line represents the symmetric matrix sensing. Different color represents the different initialization scales.} 
    \label{fig:comp_asym_sym}
\end{figure}

\subsection{Well-Conditioned Case and Ill-Conditioned Case}

We run experiments with different conditional numbers of the ground-truth matrix. The conditional number $\kappa$ is selected as $\kappa = 1.5, 3$ and $10$. The minimum eigenvalue is selected by $0.66, 0.33$ and $0.1$ respectively. The experiment results are shown in Figure \ref{fig:different_kappa}

\begin{figure}[H]
    \centering
    \includegraphics[scale=0.6]{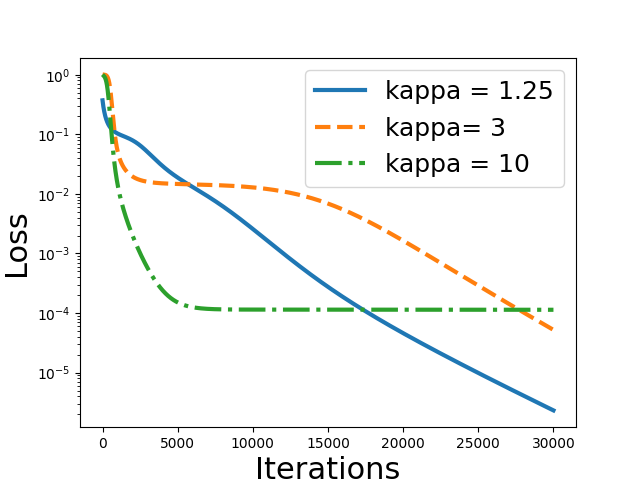}
    \caption{Comparisons between different conditional numbers}
    \label{fig:different_kappa}
\end{figure}

From the experiment results, we can see two phenomena:

$\bullet$ When the minimum eigenvalue is smaller, the gradient descent will converge to a smaller error at a linear rate. We call this phase the local convergence phase.

$\bullet$ After the local convergence phase, the curve first remains flat and then starts to converge at a linear rate again. We can see that the curve remains flat for a longer time when the matrix is ill-conditioned, i.e. $\kappa$ is larger. 

This phenomenon has been theoretically identified by the previous work for the incremental learning \citep{jiang2022algorithmic, jin2023understanding}, in which GD is shown to sequentially recover singular components of the ground truth from the largest singular value to the smallest singular value.

\subsection{Larger Initialization Scale}

We also run experiments with a larger initialization scale $\alpha.$
The experiment results are shown in Figure \ref{fig:dif_large_alpha}. 
We find that if $\alpha$ is overly large, i.e. $\alpha = 3$ and $5$, the algorithm actually converges slower and even fails to converge. This is reasonable since there is an upper bound requirement Eq. \eqref{eq: alpha_eta_requirement} for $\alpha$ in Theorem \ref{theorem:main theorem assymmetric}.

\begin{figure}[H]
    \centering
    \includegraphics[scale=0.6]{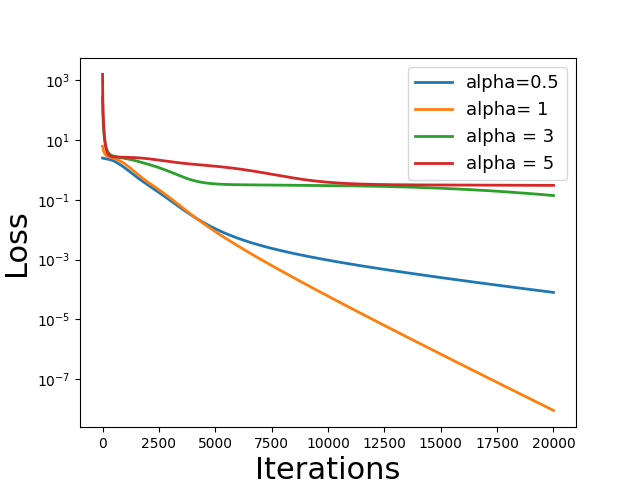}
    \caption{Comparisons between different large initialization scales}
    \label{fig:dif_large_alpha}
\end{figure}

\subsection{Larger True Rank and Over-Parameterized Rank}

We run experiments with larger configurations $n=50, k=10$ and $r=5$. We use $m=2000$ samples. The experiment results are shown in Figure \ref{figure:large}. We show that similar phenomena of symmetric and asymmetric cases also hold for a larger rank of the true matrix and a larger over-parameterized rank. Moreover, our new method also performs well in this setting.
\begin{figure}[t]\label{figure:large}
    \centering
    \subfigure[Symmetric case]{\label{figure:loss exact scale different_large}\includegraphics[scale = 0.28]{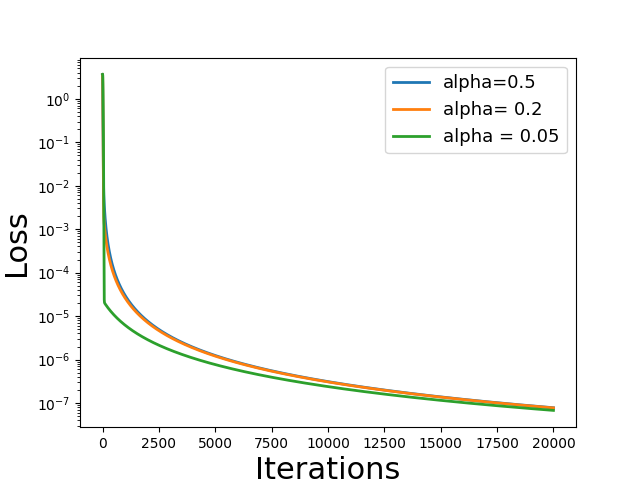}}
    \subfigure[Asymmetric Case]{\label{figure:loss asymmetric_large}\includegraphics[scale = 0.28]{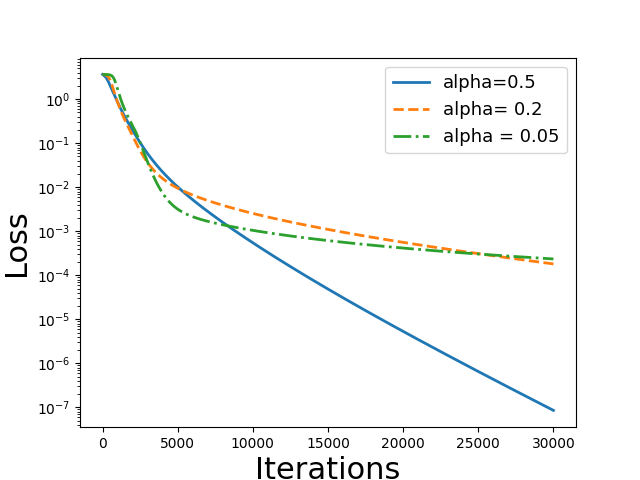}
    }
    \subfigure[Our new method]{\label{figure:algfast_large}\includegraphics[scale=0.28]{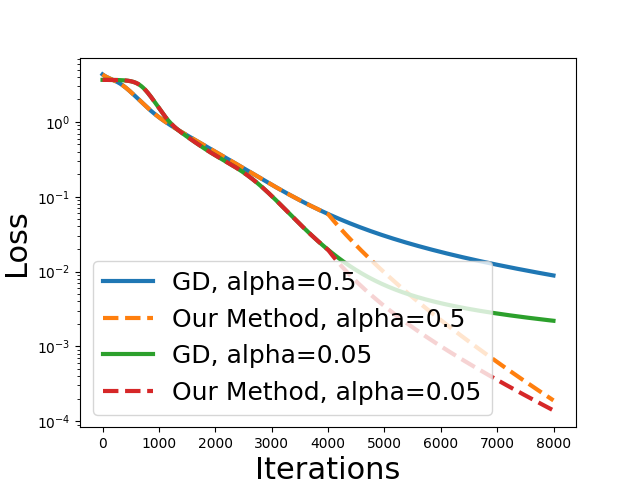}}
    \vspace{-0.4cm}
    \caption{Experiment Results of larger true rank $r=5$ and over-parameterized rank $k=10$.} 
    \vspace{-0.6cm}
\end{figure}

\subsection{Initialization Phase}

If we use GD with small initialization, GD always goes through an initialization phase where the loss is relatively flat, and then converges rapidly to a small error. In this subsection, we plot the first 5000 episodes of Figure \ref{figure:loss asymmetric}. After zooming into the first 5000 iterations, we find the existence of the initialization phase. That is, the loss is rather flat during this phase. We can also see that the initialization phase is longer when $\alpha$ is smaller. The experiment results are shown in Figure \ref{fig:short_asymmetric}.
\begin{figure}[H]
    \centering
    \includegraphics[scale=0.6]{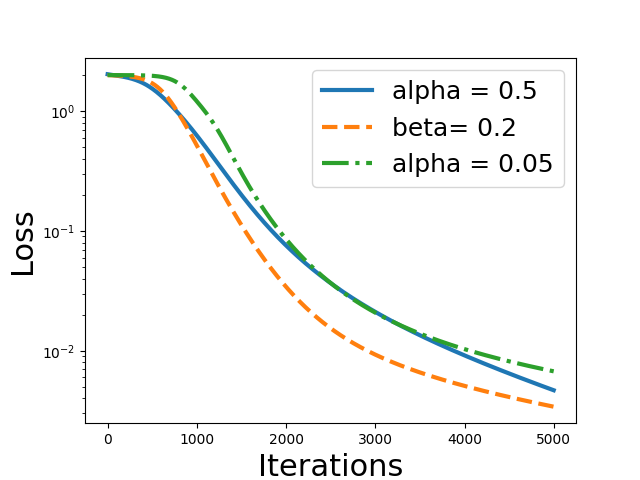}
    \caption{First 5000 episodes of Figure \ref{figure:loss asymmetric}}
    \label{fig:short_asymmetric}
\end{figure}

\end{document}